\documentclass[11pt]{article}

\usepackage[margin=1in]{geometry}
\usepackage{amsmath, amssymb, amsthm, mathtools}
\usepackage{mathrsfs}
\usepackage{bm}
\usepackage{bbm}
\usepackage{enumitem}
\usepackage{hyperref}
\usepackage{xcolor}
\usepackage{microtype}
\usepackage{graphicx}
\usepackage{float}
\usepackage{subcaption}
\usepackage{booktabs,tabularx,array,ragged2e}
\graphicspath{{figures/}{figures/exp1/}{figures/align/}}

\hypersetup{
  colorlinks=true,
  linkcolor=blue,
  citecolor=blue,
  urlcolor=blue
}

\theoremstyle{plain}
\newtheorem{theorem}{Theorem}[section]
\newtheorem{lemma}[theorem]{Lemma}
\newtheorem{proposition}[theorem]{Proposition}
\newtheorem{corollary}[theorem]{Corollary}

\theoremstyle{definition}
\newtheorem{definition}[theorem]{Definition}
\newtheorem{assumption}[theorem]{Assumption}

\theoremstyle{remark}
\newtheorem{remark}[theorem]{Remark}

\newcommand{\E}{\mathbb{E}}
\newcommand{\R}{\mathbb{R}}
\newcommand{\N}{\mathbb{N}}
\newcommand{\Id}{\mathrm{I}}
\newcommand{\diag}{\mathrm{diag}}
\newcommand{\tr}{\mathrm{tr}}
\newcommand{\rank}{\mathrm{rank}}
\newcommand{\eps}{\varepsilon}
\newcommand{\GL}{\mathrm{GL}}
\newcommand{\Orth}{\mathrm{O}}
\newcommand{\SPD}{\mathrm{SPD}}

\newcommand{\Harm}{\mathrm{H}}

\newcommand{\Cart}{\operatorname{Cart}}
\newcommand{\supp}{\operatorname{supp}}

\newcommand{\Core}{\mathrm{Core}}
\newcommand{\Overlap}{\mathrm{Overlap}}
\newcommand{\Pirow}{\Pi_{\mathrm{row}}}
\newcommand{\Picol}{\Pi_{\mathrm{col}}}
\DeclareMathOperator*{\argmax}{arg\,max}
\DeclareMathOperator*{\argmin}{arg\,min}

\newcolumntype{Y}{>{\RaggedRight\arraybackslash\hspace{0pt}}X}
\newcolumntype{L}[1]{>{\RaggedRight\arraybackslash\hspace{0pt}}p{#1}}

\title{\vspace{-1.0em}
Geometric and Spectral Alignment for Deep Neural Network II\\
\vspace{-0.5em}}

\author{
Ziran Liu$^{1,5*}$,
Wei Wang$^{2*}$,
Jinhao Wang$^{3}$,
Pengcheng Wang$^{4}$,
Xinyi Sui$^{3}$\\
Cihan Ruan$^{3}$,
Nam Ling$^{3}$,
Wei Jiang$^{2}$\\[0.5em]
$^1$Shanghai Institute for Mathematics and Interdisciplinary Sciences (SIMIS), \\Shanghai 200433, China\\
$^2$Futurewei Technologies, Inc., San Jose, CA 95131\\
$^3$Dept. of Computer Science and Engineering, Santa Clara University,\\
Santa Clara, CA 95050\\
$^4$Dept. of Computer Science, Purdue University, West Lafayette, IN 47906\\
$^5$Research Institute of Intelligent Complex Systems, Fudan University,\\ Shanghai 200433, China\\
[0.4em]
\texttt{zliu@simis.cn, rickweiwang@futurewei.com, jwang11@scu.edu}\\
\texttt{wang4495@purdue.edu, xsui@scu.edu, luciacihanruan@gmail.com}\\
\texttt{nling@scu.edu, wjiang@futurewei.com}
}

\date{}

\usepackage[backend=biber,style=numeric,sorting=nyt]{biblatex}
\begin{document}
\maketitle

\begin{abstract}
This paper develops the angular and static-channel component of Geometric and Spectral Alignment for residual Jacobian chains. Starting from Cartan-coordinate rigidity and fitted effective-rank windows, we study how dominant singular subspaces are transported across adjacent layers and how the resulting finite matrices can be displayed in physical channel coordinates.

The main results are deterministic, margin-verified results. We bound the error between full interface transport and its dominant-window truncation, add fitted-tail errors so that empirical spectra can be certified against the Gibbs--Cartan tail model, and distinguish source-mode incidence from fully physical input-output channel incidence. Given row groups and active supports, the Physical Alignment Matrix decomposes orthogonally as core plus overlap plus noise. Active-column gaps, pairwise overlap margins, and noise bounds combine into a static certificate radius under which the full transport and the truncated transport induce the same active supports, pairwise incidence graph, SRS sets, hub columns, and core/overlap/noise masks. The finer SC/SA/ST labels of the Invariant Channel Mapping require additional row-energy and profile-correlation margins, stated as explicit perturbation tests.

The empirical section reports the matrices and block-energy heatmaps that measure these certificate quantities across CNNs, language models, and vision/diffusion backbones. The figures are interpreted as finite-dimensional measurements; complete membership in the Physical GSA certificate domain requires checking the numerical margin protocol stated in Section~\ref{sec:empirical}.
\end{abstract}

\newpage
\tableofcontents
\newpage

\section{Introduction}
The spectral article proves that a budgeted residual Jacobian cocycle has a short Cartan-coordinate path and, under a positive rank-separation margin, a stable dominant energy window.  This paper studies the angular geometry inside that window.  The singular values of a layer specify how much energy is carried by each mode; the singular vectors specify how those modes are routed through the next layer.  The resulting angular transport is the finite-dimensional object seen in physical channel coordinates, and it is closely related to residual-Jacobian alignment phenomena observed in residual networks \cite{li2023residualalignment} and to classical subspace-angle geometry \cite{bjorck1973angles}.

The basic matrix data are as follows.  For a layer matrix $W=U\Sigma V^\top$, the columns of $U$ are output singular directions, the columns of $V$ are input singular directions, and the diagonal entries of $\Sigma$ are singular amplitudes.  After selecting an effective-rank window, adjacent layers provide two frames: the left singular frame of the current layer and the right singular frame of the next layer.  Their overlap matrix, together with singular-value weights and physical output realization, produces the transport matrices used in the paper.  A row/column permutation then displays the same transport in channel coordinates.

The objects developed here are finite-dimensional and directly computable from static matrices.  Given two adjacent layer maps with singular-vector frames, we form angular transport matrices and energy-weighted variants.  Mode-profile row groups and active support margins then determine a Physical Alignment Matrix, pairwise relational triples, and a global decomposition
\[
\widehat M_{\mathrm{phy}}=M_{\mathrm{core}}+M_{\mathrm{overlap}}+M_{\mathrm{noise}}.
\]
This decomposition records which output channel groups connect to which declared column supports, which supports are shared, and which residual mass is not structured by the selected dominant window.  When the selected transport has source-mode columns these are mode supports; when it has physical-input columns they are physical channel supports.

The main certificate implication proved in this paper is
\[
\begin{aligned}
\text{stable spectral window}
&\Longrightarrow
\text{controlled truncated physical transport}\\
&\Longrightarrow
\text{margin-stable static incidence and support-level ICM}.
\end{aligned}
\]
The same finite-dimensional coordinates also give a mathematical form of low-disruption fine-tuning: scale perturbations with small relative-ratio cost are close to uniform layerwise scaling, and small SVD-frame perturbations require the left and right singular-vector rotations to remain coherent up to explicitly controlled relative-rotation error.  This provides a coordinate system in which one can test whether a particular adapter or low-rank adaptation update is low-disruption; no claim is made that those methods automatically satisfy the stated margins.

The experiments are organized by the quantities in the theorems; visual block structure is treated as a measurement of certificate variables rather than as proof of domain membership by itself.  Physical alignment matrices and block-energy heatmaps measure the finite matrices, row/column partitions, noise residuals, and pairwise margins that appear in the statements.  The displayed block-dominant patterns are consistent with the structural behavior predicted by the theory; a stronger finite-dimensional margin check additionally reports the active gaps, noise bounds, and pairwise margin values stated in the hypotheses.

\begin{table}[H]
\centering
\footnotesize
\renewcommand{\arraystretch}{1.22}
\setlength{\tabcolsep}{3.0pt}
\begin{tabularx}{\textwidth}{L{0.22\textwidth}L{0.38\textwidth}Y}
\toprule
\textbf{Object} & \textbf{Formal location} & \textbf{Role in a network layer}\\
\midrule
Dominant energy window & Effective-rank window $R_\eps$ in Definition~\ref{def:effective-rank}. & Selects the singular directions carrying a prescribed fraction of layer energy.\\
Physical transport matrices & Angular and energy-weighted variants in Definition~\ref{def:transport-variants}. & Describe how dominant output directions of one layer interact with input directions of the next.\\
Physical Alignment Matrix & Permuted finite transport matrix $\widehat M_{\mathrm{phy}}$ in Definition~\ref{def:physical-alignment-matrix}. & Displays physical channel-group incidence after a chosen row/column ordering.\\
Pairwise relational triple & $\widehat{\mathcal M}^{(i,j)}_{\mathrm{pair}}$ in Definition~\ref{def:pairwise-triple}. & Separates two channel groups into private support, shared support, and pairwise overlap.\\
Core/overlap/noise split & $M_{\mathrm{core}}+M_{\mathrm{overlap}}+M_{\mathrm{noise}}$ in Definition~\ref{def:global-core-overlap-noise}. & Decomposes static channel incidence into dedicated structure, controlled sharing, and residual unstructured mass.\\
ICM/SRS/Hub variables & Invariant Channel Mapping and support sets in Definition~\ref{def:icm}. & Provide a static channel-level anatomy: salient core rows, auxiliary rows, structural supports, receptive support sets, shared hubs, and noise channels.\\
Low-disruption fine-tuning coordinates & Scale ratios and SVD-frame rotations in Section~\ref{sec:fine-tuning-paths}. & Describe local adaptation as a small change in layer scales and singular-vector frames, rather than an uncontrolled re-wiring of channels.\\
\bottomrule
\end{tabularx}
\caption{Finite structural objects in the angular and static-channel part of GSA.  Each object is defined before its first formal use; this table records the intended meaning for readers from deep learning.}
\label{tab:partII-object-dictionary}
\end{table}

\begin{table}[H]
\centering
\footnotesize
\renewcommand{\arraystretch}{1.18}
\setlength{\tabcolsep}{3.0pt}
\begin{tabularx}{\textwidth}{L{0.25\textwidth}L{0.29\textwidth}Y}
\toprule
\textbf{Margin or hypothesis} & \textbf{Finite quantity} & \textbf{Structural conclusion it certifies}\\
\midrule
Cartan/interface budget & $\log\lambda_k$ and chart errors $e_k^{\mathrm{chart}}$ & Short motion of fitted spectral coordinates.\\
Rank-tail separation & $\mathfrak m_\eps(\alpha_k)$ & The same dominant effective-rank window is selected on both sides of an interface.\\
Truncation tail & $\mathcal E_{\mathrm{tr},k}(R_s,R_t)$ & The dominant-window physical transport approximates the full output-total transport.\\
Active-column gap & $\Gamma_i(\widehat M)$ & The selected active support $\mathcal C_i$ is unchanged under perturbation.\\
Pairwise exclusive-overlap gap & $m_{i,j}-3o_{i,j}$ & The one-third coherent-overlap condition remains valid for a pair.\\
Noise bound & $\|M_{\mathrm{noise}}\|_F$ & The core/overlap model is close to the measured transport in Frobenius norm.\\
Row/profile ICM gaps & $\Gamma_i^{\mathrm{SC}},\Gamma_i^{\mathrm{ST}},\Gamma_i^{\mathrm{SA}}$ & The optional SC/SA/ST labels of the ICM remain unchanged.\\
\bottomrule
\end{tabularx}
\caption{Assumption-to-conclusion map for the certificate theory.  The paper proves deterministic implications from the listed finite margins to stable incidence conclusions; empirical figures measure the matrices from which these margins are computed.}
\label{tab:assumption-conclusion-map}
\end{table}

\section{Spectral input and finite-dimensional notation}
\label{sec:spectral-inputs}
This section records the spectral notation used by the angular theory and proves the finite-dimensional estimates used later in the paper.  The arguments are included here so that the angular and static-channel theory can be read independently of the companion spectral article.

\subsection{Minimal matrix-geometric notation for angular transport}
\label{subsec:minimal-angular-notation}
All objects in this article are finite-dimensional.  The real general linear group is
\[
\GL(d):=\{A\in\R^{d\times d}:\det A\ne0\},
\]
and the orthogonal group is
\[
\Orth(d):=\{Q\in\R^{d\times d}:Q^\top Q=\Id\}.
\]
For a full-rank layer matrix $A\in\GL(d)$, the positive form $A^\top A$ belongs to
\[
\SPD(d):=\{P=P^\top:x^\top Px>0\text{ for every }x\ne0\}.
\]
The map $A\mapsto A^\top A$ removes the left orthogonal gauge, since $(QA)^\top(QA)=A^\top A$ for every $Q\in\Orth(d)$.  Thus singular values are quotient-radial data, while singular vectors are angular data depending on the chosen input and output frames.

For an SVD $W=U\Sigma V^\top$, the output frame $U$ and input frame $V$ determine subspaces in the physical channel coordinates.  The overlap between two $R$-dimensional orthonormal frames $X,Y\in\R^{d\times R}$ is measured by $X^\top Y$; the singular values of this matrix are the cosines of the principal angles between the two subspaces \cite{bjorck1973angles}.  The angular transport matrices below are weighted and physically realized versions of such subspace-overlap matrices.  In a neural layer, this means that they record how dominant output modes of one transformation are received by the input modes or physical output channels of the next transformation.

\begin{lemma}[Square spectral embedding]
\label{lem:square-embedding}
Let $W\in\R^{m\times n}$ and set $d=\max\{m,n\}$.  Define $\widetilde W\in\R^{d\times d}$ by
\begin{equation}\label{eq:square-padding-definition}
(\widetilde W)_{ab}:=
\begin{cases}
W_{ab}, & 1\le a\le m,\ 1\le b\le n,\\
0, & \text{otherwise}.
\end{cases}
\end{equation}
Then the singular values of $\widetilde W$ are the singular values of $W$ together with $|m-n|$ additional zeros.  Consequently,
\[
\|\widetilde W\|_2=\|W\|_2,
\qquad
\|\widetilde W\|_F=\|W\|_F.
\]
\end{lemma}

\begin{proof}
If $m\le n$, then $d=n$ and $\widetilde W$ is obtained by adding $n-m$ zero rows below $W$.  Therefore
\[
\widetilde W^\top \widetilde W=W^\top W.
\]
The matrix $W^\top W$ has dimension $n\times n$; its eigenvalues are the $m$ squared singular values of $W$ together with $n-m$ zeros.  Hence the $n$ singular values of the square matrix $\widetilde W$ are the singular values of $W$ together with $n-m=|m-n|$ additional zeros.

If $m>n$, then $d=m$ and $\widetilde W$ is obtained by adding $m-n$ zero columns to $W$.  Then
\[
\widetilde W^\top\widetilde W=
\begin{bmatrix}
W^\top W&0\\
0&0
\end{bmatrix},
\]
where the lower-right zero block has size $(m-n)\times(m-n)$.  Thus the eigenvalues of $\widetilde W^\top\widetilde W$ are those of $W^\top W$ together with $m-n$ additional zeros, and taking nonnegative square roots proves the singular-value statement.

The operator norm is the largest singular value and the Frobenius norm is the square root of the sum of squared singular values.  Since the only added singular values are zero, both norms are unchanged.
\end{proof}

\begin{definition}[Spectral dimension versus physical embedding]
\label{def:spectral-dimension}
For a rectangular layer $W\in\R^{m\times n}$, the \emph{physical display dimensions} are $(m,n)$, while the \emph{spectral fitting length} is
\[
 d_{\mathrm{sp}}(W):=\min\{m,n\}
\]
or a declared numerical-rank cutoff if zero or numerically negligible singular values are removed before fitting.  All Cartan coordinates, harmonic sums, power-law tails, and effective-rank margins in Section~2 are computed on this spectral list and not on the zero-padded ambient dimension of Lemma~\ref{lem:square-embedding}.  Thus the symbol $d$ in the spectral orbit statements below denotes $d_{\mathrm{sp}}$ unless a physical ambient dimension is explicitly stated.  The square embedding is used only to place transport matrices in compatible row/column coordinates; its padding zeros are not part of the exact or relative power-law orbit.
\end{definition}

\begin{definition}[Harmonic sums and the normalized power-law orbit]
\label{def:harmonic}
For $d\in\N$ and $s\in\R$, set
\[
\Harm_{d,s}:=\sum_{i=1}^d i^{-s}.
\]
A spectral list of length $d$ lies on the exact normalized power-law orbit with exponent $\alpha>0$ if its ordered singular values satisfy
\begin{equation}\label{eq:power-law}
\sigma_i(W)=C i^{-\alpha},\qquad i=1,\ldots,d,
\end{equation}
for some $C>0$, and if $\|W\|_F^2=d$.  For a matrix whose declared spectral list has this form, the corresponding canonical Cartan representative is
\begin{equation}\label{eq:canonical-gram}
G_d(\alpha)=\diag\!\left(\frac{d}{\Harm_{d,2\alpha}}i^{-2\alpha}\right)_{i=1}^d,
\end{equation}
and the radial coordinate is
\[
g_d(\alpha)=\log\sqrt{d/\Harm_{d,2\alpha}}.
\]
For a compact interval $I=[\alpha_{\min},\alpha_{\max}]\subset(0,\infty)$ define
\[
m_d(I):=\min_{\alpha\in I}g_d'(\alpha),
\qquad
 g_d'(\alpha)=\frac{\sum_{i=1}^d(\log i)i^{-2\alpha}}{\Harm_{d,2\alpha}}.
\]
For $d\ge2$, $m_d(I)>0$.
\end{definition}

\begin{assumption}[Exact power-law and Frobenius normalization]
\label{ass:power-law}
A layer matrix satisfies the exact spectral orbit hypothesis if its spectral fitting list of length $d=d_{\mathrm{sp}}(W)$ has singular values of the form \eqref{eq:power-law} and satisfies $\|W\|_F^2=d$ after the chosen layerwise normalization.  For rectangular layers this assumption is imposed before square padding, so appended zero singular values are excluded from the power-law list.
\end{assumption}

\begin{assumption}[Layerwise Frobenius normalization]
\label{ass:frob-norm}
For a chain $(W_k)_{k=0}^{L-1}$ with common spectral fitting length $d$ on the interfaces under comparison, the layerwise normalization assumption is
\[
\|W_k\|_F^2=d,
\qquad k=0,\ldots,L-1.
\]
If adjacent rectangular layers have different fitting lengths, the statements below are applied after choosing a common spectral window or after rescaling each layer by its own declared spectral length; the physical padding dimension is not used in this normalization.
\end{assumption}

\begin{proposition}[Orbit membership and radial coordinate]
\label{prop:cartan-membership}
Under Assumption~\ref{ass:power-law}, the top singular value obeys
\[
\|W\|_2=C=\sqrt{\frac{d}{\Harm_{d,2\alpha}}},
\]
and hence
\[
\log\|W\|_2=g_d(\alpha).
\]
Moreover, the ordered eigenvalue representative of $W^\top W$ is exactly $G_d(\alpha)$.
\end{proposition}

\begin{proof}
Using \eqref{eq:power-law},
\[
\|W\|_F^2=\sum_{i=1}^d\sigma_i(W)^2
=\sum_{i=1}^d C^2 i^{-2\alpha}=C^2\Harm_{d,2\alpha}.
\]
The normalization $\|W\|_F^2=d$ therefore gives
\[
C^2=\frac{d}{\Harm_{d,2\alpha}},
\qquad
C=\sqrt{\frac{d}{\Harm_{d,2\alpha}}},
\]
since $C>0$.  The operator norm is the largest singular value, which is $\sigma_1(W)=C$.  The eigenvalues of $W^\top W$ are $\sigma_i(W)^2=C^2i^{-2\alpha}$, and substituting the expression for $C^2$ gives exactly \eqref{eq:canonical-gram}.
\end{proof}

\begin{lemma}[Approximate orbit projection error]
\label{lem:robust-lognorm}
Let $0\le\delta_{\mathrm{pl}}<1$.  Suppose $W$ satisfies $\|W\|_F^2=d$ and there exist $C>0$ and $\alpha>0$ such that
\[
(1-\delta_{\mathrm{pl}})Ci^{-\alpha}
\le \sigma_i(W)\le
(1+\delta_{\mathrm{pl}})Ci^{-\alpha},
\qquad i=1,\ldots,d.
\]
Then
\begin{equation}\label{eq:robust-lognorm-II}
\bigl|\log\|W\|_2-g_d(\alpha)\bigr|
\le
\eta(\delta_{\mathrm{pl}})
:=\log\frac{1+\delta_{\mathrm{pl}}}{1-\delta_{\mathrm{pl}}}.
\end{equation}
\end{lemma}

\begin{proof}
Squaring the two-sided singular-value inequalities and summing over $i$ gives
\[
(1-\delta_{\mathrm{pl}})^2C^2\Harm_{d,2\alpha}
\le
\sum_{i=1}^d\sigma_i(W)^2
\le
(1+\delta_{\mathrm{pl}})^2C^2\Harm_{d,2\alpha}.
\]
Since $\sum_i\sigma_i(W)^2=\|W\|_F^2=d$, this implies
\[
\frac{\sqrt d}{(1+\delta_{\mathrm{pl}})\sqrt{\Harm_{d,2\alpha}}}
\le
C
\le
\frac{\sqrt d}{(1-\delta_{\mathrm{pl}})\sqrt{\Harm_{d,2\alpha}}}.
\]
For the top singular value, the same relative fit gives
\[
(1-\delta_{\mathrm{pl}})C\le\|W\|_2\le(1+\delta_{\mathrm{pl}})C.
\]
Combining the lower bound for $\|W\|_2$ with the lower bound for $C$, and the upper bound for $\|W\|_2$ with the upper bound for $C$, yields
\[
\sqrt{\frac d{\Harm_{d,2\alpha}}}\frac{1-\delta_{\mathrm{pl}}}{1+\delta_{\mathrm{pl}}}
\le
\|W\|_2
\le
\sqrt{\frac d{\Harm_{d,2\alpha}}}\frac{1+\delta_{\mathrm{pl}}}{1-\delta_{\mathrm{pl}}}.
\]
Taking logarithms and recalling $g_d(\alpha)=\log\sqrt{d/\Harm_{d,2\alpha}}$ proves \eqref{eq:robust-lognorm-II}.
\end{proof}

\begin{definition}[Layerwise Cartan chart error]
\label{def:chart-error}
For a fitted power-law coordinate $\alpha_k$ at layer $k$, define the layerwise Cartan chart error by
\[
e_k^{\mathrm{chart}}
:=
\bigl|\log\|W_k\|_2-g_d(\alpha_k)\bigr|.
\]
In the exact normalized power-law case, $e_k^{\mathrm{chart}}=0$.  Under the relative $\delta_{\mathrm{pl},k}$ power-law fit of Lemma~\ref{lem:robust-lognorm}, one may take
\[
e_k^{\mathrm{chart}}\le \eta(\delta_{\mathrm{pl},k})
:=\log\frac{1+\delta_{\mathrm{pl},k}}{1-\delta_{\mathrm{pl},k}}.
\]
The notation separates the deterministic spectral-coordinate estimate from the statistical or numerical procedure used to fit $\alpha_k$.
\end{definition}

\begin{definition}[Interface amplification and non-backtracking]
\label{def:interface-amplification-II}
For nonzero matrices $A,B$, define
\[
\Lambda(A,B):=\frac{\|AB\|_2}{\sqrt{\|A\|_2\|B\|_2}}.
\]
For a chain write $\lambda_k:=\Lambda(W_{k+1},W_k)$.  The interface is non-backtracking if
\[
\|W_{k+1}W_k\|_2\ge \max\{\|W_{k+1}\|_2,\|W_k\|_2\}.
\]
\end{definition}

\begin{remark}[Geometric-mean normalization of the interface budget]
The quantity $\Lambda(A,B)$ is intentionally normalized by the geometric mean of the two adjacent operator norms.  Thus, with $a=\|A\|_2$, $b=\|B\|_2$, and $p=\|AB\|_2$, one has $p=\Lambda(A,B)\sqrt{ab}$.  Under the non-backtracking condition $p\ge\max\{a,b\}$, this normalization gives $\Lambda(A,B)\ge1$ and allows $\log\Lambda(A,B)$ to serve as a nonnegative local budget for changes in the Cartan radial coordinate.  This is not the usual submultiplicative efficiency ratio $\|AB\|_2/(\|A\|_2\|B\|_2)$.
\end{remark}

\begin{theorem}[Cartan coordinate rigidity for normalized power-law chains]
\label{thm:cartan-rigidity}
Let $(W_k)_{k=0}^{L-1}$ satisfy Assumption~\ref{ass:frob-norm} and the exact power-law orbit hypothesis with exponents $\alpha_k\in I=[\alpha_{\min},\alpha_{\max}]\subset(0,\infty)$.  Assume every interface is non-backtracking and $\lambda_k=\Lambda(W_{k+1},W_k)\ge1$.  Then
\begin{equation}\label{eq:cartan-rigidity-TV-II}
\sum_{k=0}^{L-2}|\alpha_{k+1}-\alpha_k|
\le
\frac{2}{m_d(I)}\sum_{k=0}^{L-2}\log\lambda_k.
\end{equation}
In particular, if $\lambda_k\le M^{2/L}$ for all $k$, then
\[
\max_{0\le k\le L-2}|\alpha_{k+1}-\alpha_k|
\le
\frac{4\log M}{L\,m_d(I)}.
\]
\end{theorem}

\begin{proof}
Fix an interface $k$ and set
\[
a=\|W_{k+1}\|_2,
\qquad
b=\|W_k\|_2,
\qquad
p=\|W_{k+1}W_k\|_2.
\]
By definition, $p=\lambda_k\sqrt{ab}$.  Non-backtracking gives $p\ge\max\{a,b\}$.  If $a\ge b$, then
\[
a\le p=\lambda_k\sqrt{ab}.
\]
Dividing by $\sqrt{ab}>0$ gives $\sqrt{a/b}\le\lambda_k$, hence $a/b\le\lambda_k^2$.  If $b\ge a$, the same argument with $a$ and $b$ interchanged gives $b/a\le\lambda_k^2$, equivalently $a/b\ge\lambda_k^{-2}$.  Thus
\[
\left|\log a-\log b\right|\le 2\log\lambda_k.
\]
By Proposition~\ref{prop:cartan-membership}, $\log a=g_d(\alpha_{k+1})$ and $\log b=g_d(\alpha_k)$.  Hence
\[
|g_d(\alpha_{k+1})-g_d(\alpha_k)|\le2\log\lambda_k.
\]
The mean value theorem gives a point $\xi_k$ between $\alpha_k$ and $\alpha_{k+1}$ such that
\[
g_d(\alpha_{k+1})-g_d(\alpha_k)=g_d'(\xi_k)(\alpha_{k+1}-\alpha_k).
\]
Because all exponents lie in $I$ and $g_d'(\xi_k)\ge m_d(I)$,
\[
|\alpha_{k+1}-\alpha_k|\le\frac{2\log\lambda_k}{m_d(I)}.
\]
Summing this inequality over $k=0,\ldots,L-2$ proves \eqref{eq:cartan-rigidity-TV-II}.  Under $\lambda_k\le M^{2/L}$, the local estimate gives
\[
|\alpha_{k+1}-\alpha_k|\le\frac{2(2\log M/L)}{m_d(I)}=\frac{4\log M}{L\,m_d(I)},
\]
and taking the maximum over $k$ proves the final claim.
\end{proof}

\begin{remark}[Local interface budget]
\label{rem:local-interface-budget}
The uniform condition $\lambda_k\le M^{2/L}$ used in the last part of Theorem~\ref{thm:cartan-rigidity} is a local interface-budget assumption. It is not a consequence of submultiplicativity of the full Jacobian norm alone. In applications it must either be measured at the relevant interfaces or supplied by a separate residual-cocycle estimate, such as the companion spectral theory.
\end{remark}

\begin{theorem}[Robust Cartan coordinate rigidity with layerwise chart errors]
\label{thm:robust-cartan-rigidity}
Assume the hypotheses of Theorem~\ref{thm:cartan-rigidity}, except that exact power-law membership is replaced by fitted coordinates $\alpha_k\in I$ with chart errors $e_k^{\mathrm{chart}}$ in Definition~\ref{def:chart-error}.  Then
\begin{equation}\label{eq:robust-cartan-TV-II}
\sum_{k=0}^{L-2}|\alpha_{k+1}-\alpha_k|
\le
\frac{1}{m_d(I)}
\sum_{k=0}^{L-2}
\bigl(2\log\lambda_k+e_{k+1}^{\mathrm{chart}}+e_k^{\mathrm{chart}}\bigr).
\end{equation}
In particular, if $e_k^{\mathrm{chart}}\le\bar e_{\mathrm{chart}}$ for all $k$ and $\lambda_k\le M^{2/L}$ for all $k$, then
\[
\max_k|\alpha_{k+1}-\alpha_k|
\le
\frac{4\log M/L+2\bar e_{\mathrm{chart}}}{m_d(I)}.
\]
If all layers satisfy a common relative $\delta_{\mathrm{pl}}$ fit, one may take $\bar e_{\mathrm{chart}}=\eta(\delta_{\mathrm{pl}})$.
\end{theorem}

\begin{proof}
By the first part of the proof of Theorem~\ref{thm:cartan-rigidity}, non-backtracking and the definition of $\lambda_k$ imply
\[
|\log\|W_{k+1}\|_2-\log\|W_k\|_2|\le2\log\lambda_k.
\]
Definition~\ref{def:chart-error} gives
\[
|\log\|W_j\|_2-g_d(\alpha_j)|\le e_j^{\mathrm{chart}},
\qquad j=k,k+1.
\]
Using the triangle inequality,
\begin{align*}
|g_d(\alpha_{k+1})-g_d(\alpha_k)|
&\le |g_d(\alpha_{k+1})-\log\|W_{k+1}\|_2|\\
&\quad +|\log\|W_{k+1}\|_2-\log\|W_k\|_2|\\
&\quad +|\log\|W_k\|_2-g_d(\alpha_k)|\\
&\le 2\log\lambda_k+e_{k+1}^{\mathrm{chart}}+e_k^{\mathrm{chart}}.
\end{align*}
Applying the mean value theorem to $g_d$ and using $g_d'\ge m_d(I)$ on $I$ gives
\[
|\alpha_{k+1}-\alpha_k|
\le
\frac{2\log\lambda_k+e_{k+1}^{\mathrm{chart}}+e_k^{\mathrm{chart}}}{m_d(I)}.
\]
Summing over interfaces proves \eqref{eq:robust-cartan-TV-II}.  The uniform local bound follows from $\log\lambda_k\le2\log M/L$ and $e_j^{\mathrm{chart}}\le\bar e_{\mathrm{chart}}$.
\end{proof}

\subsection{Spectral tail geometry and compressibility}
\label{sec:compressibility}

The compressibility part of the theory is best formulated as a spectral-tail quantity rather than only as a rank statistic but as a tail problem for a probability measure on the rank set.
Along the Cartan power-law orbit, this measure is exactly the Gibbs family already introduced above.
The effective rank is therefore an energy-truncation quantile of a spectral tail measure.

\begin{definition}[Spectral energy measure and truncation rank]
\label{def:effective-rank}
Let $W$ have a declared spectral fitting list of length $d=d_{\mathrm{sp}}(W)$ with singular values $\sigma_1(W)\ge\cdots\ge\sigma_d(W)\ge 0$, excluding any square-embedding padding zeros unless they are part of the declared numerical spectrum.
Define the spectral energy measure
\[
\mu_W:=\sum_{i=1}^d \frac{\sigma_i(W)^2}{\|W\|_F^2}\,\delta_i,
\]
a probability measure on $\{1,\dots,d\}$.
For $0<\eps<1$, define the truncation rank
\[
R_\eps(W):=
\min\Bigl\{r\in\{1,\dots,d\}:\mu_W(\{1,\dots,r\})\ge 1-\eps\Bigr\}.
\]
Equivalently,
\[
R_\eps(W)=
\min\left\{r:\sum_{i=1}^r\sigma_i(W)^2\ge (1-\eps)\sum_{i=1}^d \sigma_i(W)^2\right\}.
\]
\end{definition}

\begin{remark}[Geometric interpretation]
The quantity $R_\eps(W)$ is the smallest spectral truncation at which the projected Gram point $W^\top W$ retains at least $(1-\eps)$ of its total energy.
It is therefore a quantile of a spectral measure rather than a combinatorial surrogate for rank.
For rectangular layers, this rank refers to the spectral fitting list in Definition~\ref{def:spectral-dimension}, while the resulting singular vectors may still be embedded into physical display coordinates when transport matrices are formed.
\end{remark}

\begin{definition}[Power-law tail measure on the Gibbs--Cartan orbit]
\label{def:tail-measure}
Under the exact power-law model $\sigma_i(W)=Ci^{-\alpha}$, define
\[
\nu_\alpha:=\sum_{i=1}^d \frac{i^{-2\alpha}}{\Harm_{d,2\alpha}}\,\delta_i.
\]
Then $\mu_W=\nu_\alpha$.
For $r\in\{0,1,\dots,d\}$, define the tail function
\[
\tau_\alpha(r):=\nu_\alpha(\{r+1,\dots,d\})
=
\frac{\Harm_{d,2\alpha}-\Harm_{r,2\alpha}}{\Harm_{d,2\alpha}},
\]
with the convention $\Harm_{0,s}:=0$; thus $\tau_\alpha(0)=1$ and $\tau_\alpha(d)=0$.
\end{definition}

\begin{lemma}[Integral bounds for power-law tails]
\label{lem:integral-test}
Let $s>0$ and $1\le r<d$. Then
\begin{equation}\label{eq:integral-test}
\int_{r+1}^{d+1} x^{-s}\,dx
\le
\sum_{i=r+1}^{d} i^{-s}
\le
\int_{r}^{d} x^{-s}\,dx.
\end{equation}
If $s>1$, then
\begin{equation}\label{eq:tail-upper}
\sum_{i=r+1}^{d} i^{-s}
\le
\int_{r}^{\infty} x^{-s}\,dx
=
\frac{r^{1-s}}{s-1}.
\end{equation}
\end{lemma}

\begin{proof}
Let $f(x):=x^{-s}$.  For $s>0$, $f$ is positive and decreasing on $[1,\infty)$.  For every integer $i\ge 1$, monotonicity gives
\[
\int_i^{i+1} f(x)\,dx\le f(i)\le \int_{i-1}^{i}f(x)\,dx.
\]
Apply the left inequality with $i=r+1,\ldots,d$ and sum to obtain
\[
\int_{r+1}^{d+1}x^{-s}\,dx
=\sum_{i=r+1}^{d}\int_i^{i+1}x^{-s}\,dx
\le\sum_{i=r+1}^{d}i^{-s}.
\]
Apply the right inequality with the same indices and sum to obtain
\[
\sum_{i=r+1}^{d}i^{-s}
\le\sum_{i=r+1}^{d}\int_{i-1}^{i}x^{-s}\,dx
=\int_r^d x^{-s}\,dx.
\]
This proves \eqref{eq:integral-test}.  If $s>1$, the improper integral converges, and the already established upper bound gives
\[
\sum_{i=r+1}^{d}i^{-s}
\le\int_r^d x^{-s}\,dx
\le\int_r^\infty x^{-s}\,dx
=\left[\frac{x^{1-s}}{1-s}\right]_{x=r}^{\infty}
=\frac{r^{1-s}}{s-1}.
\]
This is \eqref{eq:tail-upper}.
\end{proof}

\begin{theorem}[Energy truncation on the Gibbs--Cartan tail]
\label{thm:effective-rank}
Let $W\in\R^{d\times d}$ satisfy the exact power-law model
\[
\sigma_i(W)=C\,i^{-\alpha},\qquad \alpha>\frac12.
\]
Then $\mu_W=\nu_\alpha$, and the truncation rank from Definition~\ref{def:effective-rank} is
\[
R_\eps(W)=\min\{r:\tau_\alpha(r)\le \eps\}.
\]
Moreover, for every $r\in\{1,\dots,d-1\}$,
\begin{equation}\label{eq:tail-condition}
\Harm_{d,2\alpha}-\Harm_{r,2\alpha}
\le
\frac{r^{1-2\alpha}}{2\alpha-1},
\end{equation}
hence
\begin{equation}\label{eq:Reps-upper}
R_\eps(W)
\le
\min\left\{
 d,\;
\left\lceil
\left(\frac{1}{(2\alpha-1)\,\eps\,\Harm_{d,2\alpha}}\right)^{\frac{1}{2\alpha-1}}
\right\rceil
\right\}.
\end{equation}
Conversely, if $r\in\{1,\dots,d-1\}$ satisfies
\begin{equation}\label{eq:Reps-lower-cond}
\frac{(r+1)^{1-2\alpha}-(d+1)^{1-2\alpha}}{2\alpha-1}
\ge
\eps\,\Harm_{d,2\alpha},
\end{equation}
then
\[
R_\eps(W)>r.
\]
\end{theorem}

\begin{proof}
Since $\mu_W=\nu_\alpha$, one has
\[
\tau_\alpha(r)
=
\frac{\sum_{i=r+1}^d C^2 i^{-2\alpha}}{\sum_{i=1}^d C^2 i^{-2\alpha}}
=
\frac{\Harm_{d,2\alpha}-\Harm_{r,2\alpha}}{\Harm_{d,2\alpha}}.
\]
Thus $R_\eps(W)$ is exactly the smallest $r$ with $\tau_\alpha(r)\le \eps$.
Applying Lemma~\ref{lem:integral-test} with $s=2\alpha>1$ gives
\[
\Harm_{d,2\alpha}-\Harm_{r,2\alpha}
=
\sum_{i=r+1}^d i^{-2\alpha}
\le
\int_r^\infty x^{-2\alpha}\,dx
=
\frac{r^{1-2\alpha}}{2\alpha-1},
\]
which proves \eqref{eq:tail-condition}.
If the right-hand side is at most $\eps\Harm_{d,2\alpha}$, then $\tau_\alpha(r)\le \eps$, hence $R_\eps(W)\le r$, giving \eqref{eq:Reps-upper}.
The lower bound follows from the lower integral estimate
\[
\sum_{i=r+1}^d i^{-2\alpha}
\ge
\int_{r+1}^{d+1}x^{-2\alpha}\,dx
=
\frac{(r+1)^{1-2\alpha}-(d+1)^{1-2\alpha}}{2\alpha-1},
\]
which implies $\tau_\alpha(r)\ge \eps$ under \eqref{eq:Reps-lower-cond}.
\end{proof}

\begin{remark}[Large-width scaling]
For fixed $\alpha>\frac12$ and large $d$, $\Harm_{d,2\alpha}\to\zeta(2\alpha)$.
The upper bound in Theorem~\ref{thm:effective-rank} therefore gives the explicit leading scale
\[
R_\eps(W)
\le
\left\lceil
\left(\frac{1+o(1)}{(2\alpha-1)\eps\,\zeta(2\alpha)}\right)^{\frac{1}{2\alpha-1}}
\right\rceil
\qquad(d\to\infty).
\]
The theorem above is its exact finite-width form on the Gibbs--Cartan orbit.
\end{remark}

\begin{proposition}[Monotonicity of truncation rank along the Cartan orbit]
\label{prop:rank-monotone-alpha}
Fix $0<\eps<1$ and finite width $d$.
For the power-law tail measures $\nu_\alpha$ from Definition~\ref{def:tail-measure}, the cumulative mass
\[
F_r(\alpha):=\nu_\alpha(\{1,\dots,r\})
=
\sum_{i=1}^r \frac{i^{-2\alpha}}{\Harm_{d,2\alpha}}
\]
is nondecreasing in $\alpha$ for every $r=1,\dots,d$.
Equivalently, the tail mass $\tau_\alpha(r)=1-F_r(\alpha)$ is nonincreasing in $\alpha$.
Consequently, the truncation rank $R_\eps(\alpha):=\min\{r:\tau_\alpha(r)\le \eps\}$ is nonincreasing as a function of $\alpha$.
\end{proposition}

\begin{proof}
Fix $r$ and write $A=\{1,\dots,r\}$.
Using the Gibbs form $p_i^{(\alpha)}\propto e^{-2\alpha\log i}$,
\[
\frac{d}{d\alpha}p_i^{(\alpha)}
=
-2\bigl(\log i-U_d(\alpha)\bigr)p_i^{(\alpha)}.
\]
Therefore
\[
F_r'(\alpha)
=
-2\sum_{i\in A}(\log i-U_d(\alpha))p_i^{(\alpha)}
=
2F_r(\alpha)\left(U_d(\alpha)-\E_{p^{(\alpha)}}[\log i\mid i\in A]\right),
\]
with the convention that the derivative is zero when $F_r(\alpha)=0$, which never occurs here.
The conditional mean over $A$ is at most the unconditional mean $U_d(\alpha)$, because every value of $\log i$ on $A^c$ is at least every value on $A$.
Hence $F_r'(\alpha)\ge0$.
Thus $\tau_\alpha(r)=1-F_r(\alpha)$ is nonincreasing in $\alpha$.
If $\alpha_2\ge\alpha_1$ and $r=R_\eps(\alpha_1)$, then $\tau_{\alpha_2}(r)\le\tau_{\alpha_1}(r)\le\eps$, so $R_\eps(\alpha_2)\le r=R_\eps(\alpha_1)$.
\end{proof}

\begin{lemma}[Uniform Lipschitz bound for spectral tail masses]
\label{lem:tail-lipschitz-alpha}
For every $r\in\{0,1,\dots,d\}$ and every $\alpha,\beta>0$,
\begin{equation}\label{eq:tail-lipschitz-alpha}
|\tau_\alpha(r)-\tau_\beta(r)|\le 2(\log d)|\alpha-\beta|.
\end{equation}
\end{lemma}

\begin{proof}
The endpoints require no estimate: if $r=0$, then $\tau_\alpha(0)=1$ for all $\alpha$, and if $r=d$, then $\tau_\alpha(d)=0$ for all $\alpha$.  In both cases the left-hand side of \eqref{eq:tail-lipschitz-alpha} is zero.

Assume now $1\le r\le d-1$ and set
\[
F_r(\alpha):=1-\tau_\alpha(r)=\sum_{i=1}^{r}p_i^{(\alpha)}.
\]
The derivative calculation in Proposition~\ref{prop:rank-monotone-alpha} gives
\[
F_r'(\alpha)=2F_r(\alpha)
\left(U_d(\alpha)-\E_{p^{(\alpha)}}[\log i\mid i\le r]\right).
\]
The factor $F_r(\alpha)$ lies in $[0,1]$.  The random variable $\log i$ always lies in $[0,\log d]$, so both the unconditional mean $U_d(\alpha)$ and the conditional mean $\E_{p^{(\alpha)}}[\log i\mid i\le r]$ lie in that same interval.  Therefore
\[
\left|U_d(\alpha)-\E_{p^{(\alpha)}}[\log i\mid i\le r]\right|\le \log d.
\]
Combining the two bounds gives
\[
|F_r'(\alpha)|\le 2\log d.
\]
Since $\tau_\alpha(r)=1-F_r(\alpha)$, we also have
\[
\left|\frac{d}{d\alpha}\tau_\alpha(r)\right|=|F_r'(\alpha)|\le2\log d.
\]
For arbitrary $\alpha,\beta>0$, the mean value theorem applied to the continuously differentiable function $\gamma\mapsto\tau_\gamma(r)$ gives
\[
|\tau_\alpha(r)-\tau_\beta(r)|\le \sup_{\gamma\text{ between }\alpha\text{ and }\beta}\left|\frac{d}{d\gamma}\tau_\gamma(r)\right|\,|\alpha-\beta|
\le2(\log d)|\alpha-\beta|,
\]
which proves \eqref{eq:tail-lipschitz-alpha}.
\end{proof}

\begin{definition}[Rank-separation margin]
\label{def:rank-separation-margin}
Fix $0<\eps<1$ and $\alpha>0$.
Let $r_\eps(\alpha):=R_\eps(\alpha)$ denote the truncation rank of the power-law energy measure $\nu_\alpha$.
Define the rank-separation margin
\begin{equation}\label{eq:rank-separation-margin}
\mathfrak m_\eps(\alpha)
:=
\min\bigl\{\eps-\tau_\alpha(r_\eps(\alpha)),\ \tau_\alpha(r_\eps(\alpha)-1)-\eps\bigr\}.
\end{equation}
The margin is positive precisely when the threshold $\eps$ does not coincide with the tail mass at either side of the selected rank.
\end{definition}

\begin{definition}[Fitted-tail error]
\label{def:fitted-tail-error}
Let $W\in\R^{d\times d}$ have spectral energy measure $\mu_W$ and let $\alpha>0$ be a fitted Cartan-tail parameter. Define the fitted-tail discrepancy
\begin{equation}\label{eq:fitted-tail-error}
\Delta_{\mathrm{tail}}(W,\alpha)
:=
\sup_{0\le r\le d}
\left|
\mu_W(\{r+1,\ldots,d\})-\tau_\alpha(r)
\right|.
\end{equation}
This quantity compares the actual empirical spectral tail of $W$ with the Gibbs--Cartan tail at parameter $\alpha$. It is stronger than the chart error $e^{\mathrm{chart}}$, which only controls the top radial coordinate.
\end{definition}

\begin{proposition}[Robust empirical effective-rank window under fitted tails]
\label{prop:robust-empirical-rank-window}
Fix $0<\eps<1$ and let $r=R_\eps(\alpha)$ for the Gibbs--Cartan tail at parameter $\alpha$. If
\begin{equation}\label{eq:single-tail-fit-rank}
\Delta_{\mathrm{tail}}(W,\alpha)<\mathfrak m_\eps(\alpha),
\end{equation}
then the empirical truncation rank of $W$ equals the fitted-tail rank:
\[
R_\eps(W)=R_\eps(\alpha)=r.
\]
More generally, let $W_0,W_1$ have fitted parameters $\alpha_0,\alpha_1$ and set
\[
\Delta_j:=\Delta_{\mathrm{tail}}(W_j,\alpha_j),\qquad j=0,1.
\]
If
\begin{equation}\label{eq:two-tail-fit-rank}
2(\log d)|\alpha_1-\alpha_0|+\Delta_0+\Delta_1
<
\mathfrak m_\eps(\alpha_0),
\end{equation}
then
\[
R_\eps(W_0)=R_\eps(W_1)=R_\eps(\alpha_0).
\]
\end{proposition}

\begin{proof}
Let $r=R_\eps(\alpha)$. By the definition of the rank-separation margin,
\[
\tau_\alpha(r)\le \eps-\mathfrak m_\eps(\alpha),
\qquad
\tau_\alpha(r-1)\ge \eps+\mathfrak m_\eps(\alpha).
\]
If \eqref{eq:single-tail-fit-rank} holds, then
\[
\mu_W(\{r+1,\ldots,d\})
\le \tau_\alpha(r)+\Delta_{\mathrm{tail}}(W,\alpha)<\eps,
\]
and
\[
\mu_W(\{r,\ldots,d\})
=\mu_W(\{(r-1)+1,\ldots,d\})
\ge \tau_\alpha(r-1)-\Delta_{\mathrm{tail}}(W,\alpha)>\eps.
\]
Thus rank $r$ satisfies the empirical tail constraint while rank $r-1$ does not, so $R_\eps(W)=r$.

For the two-layer statement, set $r=R_\eps(\alpha_0)$ and $m=\mathfrak m_\eps(\alpha_0)$. Condition \eqref{eq:two-tail-fit-rank} implies in particular $\Delta_0<m$, so the first part gives $R_\eps(W_0)=r$. For $W_1$, Lemma~\ref{lem:tail-lipschitz-alpha} gives
\[
|\tau_{\alpha_1}(q)-\tau_{\alpha_0}(q)|
\le
2(\log d)|\alpha_1-\alpha_0|
\]
for every $q$. Hence
\[
\mu_{W_1}(\{r+1,\ldots,d\})
\le
\tau_{\alpha_0}(r)+2(\log d)|\alpha_1-\alpha_0|+\Delta_1
<\eps,
\]
and similarly
\[
\mu_{W_1}(\{r,\ldots,d\})
\ge
\tau_{\alpha_0}(r-1)-2(\log d)|\alpha_1-\alpha_0|-\Delta_1
>\eps.
\]
Thus $R_\eps(W_1)=r$ as well.
\end{proof}

\begin{proposition}[Stability of the effective-rank window]
\label{prop:rank-window-stability}
Fix $0<\eps<1$ and $\alpha,\beta>0$.
Let $r:=R_\eps(\alpha)$.
If
\begin{equation}\label{eq:rank-window-stability-condition}
2(\log d)|\alpha-\beta|<\mathfrak m_\eps(\alpha),
\end{equation}
then
\[
R_\eps(\beta)=R_\eps(\alpha)=r.
\]
\end{proposition}

\begin{proof}
Let $r=R_\eps(\alpha)$.  By definition of $R_\eps$, the rank $r$ is the first rank whose tail is at most $\eps$:
\[
\tau_\alpha(r)\le\eps,
\qquad
\tau_\alpha(r-1)>\eps
\]
with the second inequality interpreted for $r>1$.  The positive margin $\mathfrak m_\eps(\alpha)$ strengthens these to
\[
\tau_\alpha(r)\le \eps-\mathfrak m_\eps(\alpha),
\qquad
\tau_\alpha(r-1)\ge \eps+\mathfrak m_\eps(\alpha).
\]
The second display is exactly Definition~\ref{def:rank-separation-margin}.

By Lemma~\ref{lem:tail-lipschitz-alpha}, for every rank index $q$,
\[
|\tau_\beta(q)-\tau_\alpha(q)|\le 2(\log d)|\alpha-\beta|.
\]
Using \eqref{eq:rank-window-stability-condition}, we obtain
\[
\tau_\beta(r)
\le \tau_\alpha(r)+2(\log d)|\alpha-\beta|
< (\eps-\mathfrak m_\eps(\alpha))+\mathfrak m_\eps(\alpha)=\eps,
\]
and similarly
\[
\tau_\beta(r-1)
\ge \tau_\alpha(r-1)-2(\log d)|\alpha-\beta|
> (\eps+\mathfrak m_\eps(\alpha))-\mathfrak m_\eps(\alpha)=\eps.
\]
The first inequality says that rank $r$ captures at least $1-\eps$ of the spectral energy for parameter $\beta$.  The second says that rank $r-1$ fails to do so.  Since $R_\eps(\beta)$ is the minimal rank satisfying the tail constraint, both conditions together imply $R_\eps(\beta)=r$.
\end{proof}

\begin{corollary}[Cartan shortness selects a stable dominant-mode bundle]
\label{cor:cartan-to-rank-window}
Assume the hypotheses of Theorem~\ref{thm:robust-cartan-rigidity} on an interval $I$.
For an interface $k$, define the theorem-predicted coordinate displacement bound
\[
B_k:=\frac{2\log\lambda_k+e_k^{\mathrm{chart}}+e_{k+1}^{\mathrm{chart}}}{m_d(I)}.
\]
Let $r_k:=R_\eps(\alpha_k)$.
If
\begin{equation}\label{eq:cartan-rank-stability-condition}
2(\log d)B_k<\mathfrak m_\eps(\alpha_k),
\end{equation}
then the same effective-rank window is selected on both sides of the interface:
\[
R_\eps(\alpha_{k+1})=R_\eps(\alpha_k)=r_k.
\]
In the exact uniform-budget case, it is sufficient that
\[
\frac{8(\log d)(\log M)}{L\,m_d(I)}<\mathfrak m_\eps(\alpha_k).
\]
\end{corollary}

\begin{proof}
The proof of Theorem~\ref{thm:robust-cartan-rigidity} gives the coordinate-displacement estimate
\[
|\alpha_{k+1}-\alpha_k|
\le
\frac{2\log\lambda_k+e_k^{\mathrm{chart}}+e_{k+1}^{\mathrm{chart}}}{m_d(I)}
=B_k.
\]
If \eqref{eq:cartan-rank-stability-condition} holds, then
\[
2(\log d)|\alpha_{k+1}-\alpha_k|
\le 2(\log d)B_k
<\mathfrak m_\eps(\alpha_k).
\]
This is exactly the hypothesis of Proposition~\ref{prop:rank-window-stability} with $\alpha=\alpha_k$ and $\beta=\alpha_{k+1}$.  Therefore
\[
R_\eps(\alpha_{k+1})=R_\eps(\alpha_k).
\]
In the exact power-law case the chart errors vanish.  Under the uniform interface budget, $\log\lambda_k\le 2\log M/L$, and hence
\[
B_k\le \frac{4\log M}{L\,m_d(I)}.
\]
Substituting this upper bound into \eqref{eq:cartan-rank-stability-condition} gives the displayed uniform-budget condition.
\end{proof}

\begin{corollary}[Cartan shortness selects the same empirical rank window under fitted-tail errors]
\label{cor:cartan-to-empirical-rank-window}
Assume the hypotheses of Theorem~\ref{thm:robust-cartan-rigidity} on an interval $I$, and let
\[
B_k:=\frac{2\log\lambda_k+e_k^{\mathrm{chart}}+e_{k+1}^{\mathrm{chart}}}{m_d(I)}.
\]
For the two adjacent empirical spectra define
\[
\Delta_{\mathrm{tail},k}:=\Delta_{\mathrm{tail}}(W_k,\alpha_k),
\qquad
\Delta_{\mathrm{tail},k+1}:=\Delta_{\mathrm{tail}}(W_{k+1},\alpha_{k+1}).
\]
If
\begin{equation}\label{eq:cartan-empirical-rank-stability}
2(\log d)B_k+\Delta_{\mathrm{tail},k}+\Delta_{\mathrm{tail},k+1}
<
\mathfrak m_\eps(\alpha_k),
\end{equation}
then the actual empirical effective-rank windows agree:
\[
R_\eps(W_k)=R_\eps(W_{k+1})=R_\eps(\alpha_k).
\]
In the exact Gibbs--Cartan tail case the two fitted-tail errors vanish and this reduces to Corollary~\ref{cor:cartan-to-rank-window}.
\end{corollary}

\begin{proof}
Theorem~\ref{thm:robust-cartan-rigidity} gives $|\alpha_{k+1}-\alpha_k|\le B_k$. Substituting this into Proposition~\ref{prop:robust-empirical-rank-window} with $(W_0,\alpha_0)=(W_k,\alpha_k)$ and $(W_1,\alpha_1)=(W_{k+1},\alpha_{k+1})$ gives the claim.
\end{proof}


\section{Activation permeability and conditional residual capacity}
\label{sec:activation-capacity}

This section is an optional capacity-accounting layer rather than a hypothesis needed for the static channel-incidence certificates below.  The constants $\kappa_\phi$ and $\chi_\phi$ give activation-dependent sufficient conditions for residual-scale control.  The main physical-alignment theory in Sections~\ref{sec:angular-transport}--\ref{sec:physical-gsa} is finite-dimensional and matrix-theoretic: its conclusions require the explicit spectral, truncation, active-support, pairwise-overlap, and noise margins stated there.

The spectral and angular geometry above describes static weight geometry.
To connect it to nonlinear residual networks, we introduce two activation-dependent constants.
They enter through explicit assumptions, and the resulting bounds are mathematically conditional and checkable. The definitions are compatible with the variance-propagation role of activation derivatives in Xavier/Glorot and He initialization, and with the smooth activations used in modern networks~\cite{glorot2010understanding,he2015delving,hendrycks2016gelu,ramachandran2017searching}.

\begin{definition}[Activation permeability and critical capacity]
\label{def:activation-permeability}
Let $Z\sim N(0,1)$ and let $\phi:\R\to\R$ be an activation with weak derivative $\phi'$ satisfying $\phi'(Z)\in L^2$.
Define
\[
\kappa_\phi:=\E[\phi'(Z)],
\qquad
\chi_\phi:=\E[(\phi'(Z))^2].
\]
The scalar $\kappa_\phi$ is the \emph{gradient permeability} and $\chi_\phi$ is the \emph{criticality capacity} of the activation under the standard Gaussian input model.
\end{definition}

\begin{remark}[No conflict with Cartan projection]
The Cartan projection is denoted by $\Cart$ throughout the paper.
The symbol $\kappa_\phi$ is reserved exclusively for activation permeability.
\end{remark}

\begin{proposition}[Conditional typical incoherent residual-scale bound]
\label{thm:typical-scale-bound}
Consider a pre-normalized residual recursion
\[
x_{k+1}=x_k+\Delta_k,
\qquad
\Delta_k:=\mathcal F_k(\mathrm{Norm}(x_k)),
\qquad k=0,\dots,L-1.
\]
Let
\[
e_k:=\E\left[\frac1d\|x_k\|^2\right],
\qquad
s:=\sup_{0\le k\le L-1}\E\left[\frac1d\|\mathrm{Norm}(x_k)\|^2\right],
\]
and assume $e_0>0$ and $s>0$.
Assume there exist $C>0$, $\eta\in(0,1]$, and an activation permeability $\kappa_\phi>0$ such that for all $k$,
\begin{align}
\E\left[\frac1d\langle x_k,\Delta_k\rangle\right]&=0, \label{eq:mean-orthogonal}\\
\E\left[\frac1d\|\Delta_k\|^2\right]&\le (\eta\kappa_\phi C)^2s. \label{eq:injection-power}
\end{align}
Then
\begin{equation}\label{eq:energy-unroll-bound}
e_L\le e_0+L(\eta\kappa_\phi C)^2s.
\end{equation}
Consequently, the sufficient and explicitly checkable condition
\begin{equation}\label{eq:typical-C-bound}
C\le
\frac{1}{\eta\kappa_\phi}
\sqrt{\frac{M^2-1}{L}\cdot \frac{e_0}{s}}
\end{equation}
implies the terminal energy bound
\[
e_L\le M^2e_0.
\]
Conversely, if the injection bound is saturated in the aggregate, i.e.
\[
\sum_{k=0}^{L-1}\E\left[\frac1d\|\Delta_k\|^2\right]=L(\eta\kappa_\phi C)^2s,
\]
then $e_L\le M^2e_0$ holds if and only if \eqref{eq:typical-C-bound} holds.
\end{proposition}

\begin{proof}
Expanding the residual update gives
\[
\|x_{k+1}\|^2=\|x_k\|^2+\|\Delta_k\|^2+2\langle x_k,\Delta_k\rangle.
\]
After dividing by $d$ and taking expectation, \eqref{eq:mean-orthogonal} yields
\[
e_{k+1}=e_k+\E\left[\frac1d\|\Delta_k\|^2\right].
\]
Using \eqref{eq:injection-power} gives
\[
e_{k+1}\le e_k+(\eta\kappa_\phi C)^2s.
\]
Induction on $k$ proves
\[
e_L\le e_0+L(\eta\kappa_\phi C)^2s,
\]
which is \eqref{eq:energy-unroll-bound}.
If \eqref{eq:typical-C-bound} holds, then
\[
L(\eta\kappa_\phi C)^2s\le (M^2-1)e_0.
\]
Substituting this into \eqref{eq:energy-unroll-bound} gives
\[
e_L\le e_0+(M^2-1)e_0=M^2e_0.
\]
For the converse under aggregate saturation, the identity
\[
e_L=e_0+L(\eta\kappa_\phi C)^2s
\]
holds, so $e_L\le M^2e_0$ is equivalent to
\[
L(\eta\kappa_\phi C)^2s\le (M^2-1)e_0,
\]
which is exactly \eqref{eq:typical-C-bound}.
\end{proof}

\begin{corollary}[Conditional worst-case coherent residual-scale bound]
\label{cor:coherent-scale}
Assume the residual branch satisfies the deterministic stepwise sufficient condition
\[
\|x_{k+1}\|\le (1+\chi_\phi C)\|x_k\|,
\qquad k=0,\dots,L-1,
\]
with $\chi_\phi>0$.
Then the sufficient condition
\begin{equation}\label{eq:coherent-C-bound}
C\le \frac{M^{1/L}-1}{\chi_\phi}
\end{equation}
implies
\[
\|x_L\|\le M\|x_0\|
\]
for all inputs.
If the stepwise bound is saturated along some trajectory with $x_0\neq 0$, then \eqref{eq:coherent-C-bound} is also necessary for that trajectory to satisfy $\|x_L\|\le M\|x_0\|$.
Furthermore, as $(\log M)/L\to0$,
\[
\frac{M^{1/L}-1}{\chi_\phi}
=
\frac{\log M}{L\chi_\phi}
+
O\left(\frac{(\log M)^2}{L^2\chi_\phi}\right).
\]
\end{corollary}

\begin{proof}
The stepwise bound gives, for each $k$,
\[
\|x_{k+1}\|\le (1+\chi_\phi C)\|x_k\|.
\]
Applying this inequality successively for $k=0,1,\ldots,L-1$ gives
\[
\|x_L\|\le (1+\chi_\phi C)^L\|x_0\|.
\]
If \eqref{eq:coherent-C-bound} holds, then
\[
\chi_\phi C\le M^{1/L}-1,
\]
and hence
\[
1+\chi_\phi C\le M^{1/L}.
\]
Since both sides are nonnegative, raising to the $L$th power yields
\[
(1+\chi_\phi C)^L\le M.
\]
Substituting this into the iterated estimate proves $\|x_L\|\le M\|x_0\|$.

For necessity under saturation, assume $x_0\ne0$ and equality holds at every step.  Then
\[
\|x_L\|=(1+\chi_\phi C)^L\|x_0\|.
\]
The desired bound $\|x_L\|\le M\|x_0\|$ is therefore equivalent, after dividing by $\|x_0\|>0$, to
\[
(1+\chi_\phi C)^L\le M.
\]
Taking the positive $L$th root gives $1+\chi_\phi C\le M^{1/L}$, which is exactly \eqref{eq:coherent-C-bound}.
Finally, writing $a=(\log M)/L$ gives
\[
M^{1/L}-1=e^a-1=a+O(a^2)
\]
as $a\to0$, and division by $\chi_\phi$ gives the displayed asymptotic expansion.
\end{proof}

\subsection{Activation-capacity width bounds}
\begin{proposition}[Activation-capacity width bounds as a capacity-accounting consequence]
\label{thm:width-laws}
Fix a spectral block with effective output rank $r_{\mathrm{out}}>0$.
Let $W$ denote the number of physical input channels assigned to feed this block.
Assume that $\chi_\phi>0$ for the critical coherent regime and $\kappa_\phi^2>0$ for the typical robust regime.
Assume the following capacity-accounting model.
\begin{enumerate}[label=(B\arabic*),leftmargin=2.1em]
\item A stable realization of the block must supply at least $r_{\mathrm{out}}$ units of effective variance to its output subspace.
\item In the critical coherent regime, the total stable effective variance supplied by $W$ assigned input channels is at most $W\chi_\phi$.
\item In the typical robust regime, the total stable effective variance supplied by $W$ assigned input channels is at most $W\kappa_\phi^2$.
\end{enumerate}
Then any critical coherent realization satisfies
\begin{equation}\label{eq:Wmin-law}
W\ge \frac{r_{\mathrm{out}}}{\chi_\phi},
\end{equation}
and any typical robust realization satisfies
\begin{equation}\label{eq:Wopt-law}
W\ge \frac{r_{\mathrm{out}}}{\kappa_\phi^2}.
\end{equation}
Equivalently, the integer channel counts must obey
\[
W\ge \left\lceil\frac{r_{\mathrm{out}}}{\chi_\phi}\right\rceil
\quad\text{or}\quad
W\ge \left\lceil\frac{r_{\mathrm{out}}}{\kappa_\phi^2}\right\rceil
\]
in the two respective regimes.
\end{proposition}

\begin{proof}
The proof is a deterministic capacity-counting argument.  We keep the two regimes separate because they use different per-channel supply constants.
\begin{enumerate}[label=(W\arabic*),leftmargin=2.1em]
\item \textbf{Critical coherent regime.}
By assumption (B2), the total critical coherent supply of the $W$ assigned channels satisfies
\[
\text{total critical supply}\le W\chi_\phi.
\]
By assumption (B1), any stable realization of the block must supply at least $r_{\mathrm{out}}$ units to its effective output subspace.  Therefore the necessary inequality is
\[
W\chi_\phi\ge r_{\mathrm{out}}.
\]
By the positive-capacity assumption $\chi_\phi>0$, dividing by $\chi_\phi$ gives
\[
W\ge \frac{r_{\mathrm{out}}}{\chi_\phi},
\]
which is \eqref{eq:Wmin-law}.

\item \textbf{Typical robust regime.}
Assumption (B3) gives the corresponding typical robust supply bound
\[
\text{total typical supply}\le W\kappa_\phi^2.
\]
The demand is still $r_{\mathrm{out}}$ by (B1).  Hence
\[
W\kappa_\phi^2\ge r_{\mathrm{out}}.
\]
By the positive-capacity assumption $\kappa_\phi^2>0$, this is equivalent to
\[
W\ge \frac{r_{\mathrm{out}}}{\kappa_\phi^2},
\]
which is \eqref{eq:Wopt-law}.

\item \textbf{Integer channel counts.}
The variable $W$ counts physical input channels and is therefore an integer.  If an integer $W$ satisfies $W\ge a$ for a real number $a$, then it also satisfies $W\ge\lceil a\rceil$.  Applying this to $a=r_{\mathrm{out}}/\chi_\phi$ and $a=r_{\mathrm{out}}/\kappa_\phi^2$ gives the two ceiling forms.
\end{enumerate}
\end{proof}

\begin{remark}[Conditional scope of the width bounds]
The bounds are conditional on the capacity-accounting assumptions (B1)--(B3).  Under those assumptions, no additional probabilistic independence or architectural heuristic is used in the proof; the inequalities are exact consequences of demand-versus-supply accounting.  The constants $\kappa_\phi$ and $\chi_\phi$ are activation-dependent moments under the chosen input law and can be estimated numerically or analytically when the activation permits.
\end{remark}


\section{Angular transport and physical alignment}
\label{sec:physical-alignment}
\label{sec:angular-transport}

The Cartan theory controls singular values after the orthogonal gauge has been removed.
The Physical GSA also requires a static description of how dominant singular directions are routed across an interface.
The description is formulated as a finite-dimensional static-structure theory: all objects are defined directly from SVD data and permutations, and every theorem below is a deterministic consequence of explicit margin inequalities.
Whether a trained network satisfies these margin inequalities is an empirical question measured in Section~\ref{sec:empirical}; the consequences of the inequalities are mathematical.

There are two logically distinct levels.  First, the \emph{given-structure} statements assume row groups, support sizes, and active column sets and then prove projection, pairwise-margin, and block-energy consequences.  Second, the \emph{extraction} statements use deterministic rules, such as mode-profile row assignment and top-energy active columns, and require row-profile gaps or active-column gaps to ensure that the extracted structure is stable under perturbation.  Whenever a theorem claims preservation of an incidence structure, the relevant extraction gaps are stated explicitly; otherwise the statement is to be read with the row groups and active sets fixed.

\subsection{Dominant angular transport matrices}

\begin{definition}[SVD gauge convention for angular transport]
\label{def:svd-gauge-convention}
For every matrix whose singular vectors are used in the angular theory, fix once and for all a deterministic SVD selection rule:
\[
W=U\Sigma V^\top,
\qquad
\sigma_1(W)\ge\cdots\ge\sigma_d(W)\ge0.
\]
The rule orders singular values decreasingly, chooses an orthonormal basis inside each singular subspace by a fixed deterministic convention, and fixes column signs by a fixed deterministic sign convention.  All physical transport matrices, row groups, active columns, pairwise triples, and ICM structures below are defined relative to this chosen SVD gauge.

If a singular value has multiplicity greater than one, the individual singular vectors inside the corresponding eigenspace are not intrinsic.  The intrinsic object is the singular subspace.  The deterministic convention above makes the finite matrices reproducible, while the stability theorems require additional active-column, pairwise-margin, and perturbation inequalities precisely to ensure that the extracted incidence structure is stable for the chosen realization.  The main statements below are therefore gauge-relative finite-dimensional certificates for this fixed, reproducible SVD gauge unless a separate singular-subspace gap assumption is imposed.  A gauge-invariant formulation can instead be written in terms of orthogonal projectors $P_U=U_{\mathcal I}U_{\mathcal I}^\top$ and $P_V=V_{\mathcal I}V_{\mathcal I}^\top$ for singular-value clusters, with projector perturbations controlled by standard subspace perturbation estimates.  We keep the gauge-fixed version because the measured physical matrices in Section~\ref{sec:empirical} are computed from a deterministic SVD convention, but all stability conclusions should be interpreted relative to that convention.
\end{definition}

\begin{remark}[Vector-level versus projector-level stability]
\label{rem:vector-projector-stability}
The gauge-fixed vector-level transports are most stable when the singular directions used individually are separated.  A sufficient numerical condition is a positive singular-value gap at the relevant boundary, for example $\sigma_R(W)-\sigma_{R+1}(W)>0$ for a rank-$R$ subspace, together with the usual within-window separation if individual modes rather than a whole subspace are interpreted.  Without such gaps, the projector onto a singular cluster can be stable while its chosen basis is not.  In that case the intrinsic certificate is a projector-level or cluster-level certificate using quantities such as $P_U=U_{\mathcal I}U_{\mathcal I}^\top$ and block norms between projectors; the present vector-level ICM should then be read as a reproducible gauge-relative numerical extraction rather than a gauge-invariant structural claim.
\end{remark}

Let
\[
W_k=U_k\Sigma_kV_k^\top,
\qquad
W_{k+1}=U_{k+1}\Sigma_{k+1}V_{k+1}^\top
\]
be SVDs after the square spectral embedding of Lemma~\ref{lem:square-embedding} if needed.
For an effective rank $R$, write $U_k^{(R)}$, $V_k^{(R)}$, and $\Sigma_k^{(R)}$ for the top-$R$ truncated factors.

\begin{definition}[Angular, energy-weighted, and output-realized transport]
\label{def:transport-variants}
The truncated latent angular transport is
\begin{equation}\label{eq:M-ang}
M_k^{\mathrm{ang},R}:=(V_{k+1}^{(R)})^\top U_k^{(R)}\in\R^{R\times R}.
\end{equation}
The source-weighted, target-weighted, and total-energy latent transports are
\begin{align}
M_{\mathrm{src},k}^{(R)}&:=(V_{k+1}^{(R)})^\top U_k^{(R)}\Sigma_k^{(R)},\label{eq:M-src}\\
M_{\mathrm{tgt},k}^{(R)}&:=\Sigma_{k+1}^{(R)}(V_{k+1}^{(R)})^\top U_k^{(R)},\label{eq:M-tgt}\\
M_{\mathrm{total},k}^{(R)}&:=\Sigma_{k+1}^{(R)}(V_{k+1}^{(R)})^\top U_k^{(R)}\Sigma_k^{(R)}.\label{eq:M-total}
\end{align}
The output-realized scale-free angular transport is
\begin{equation}\label{eq:M-out-ang}
M_{\mathrm{out,ang},k}^{(R)}:=U_{k+1}^{(R)}(V_{k+1}^{(R)})^\top U_k^{(R)}.
\end{equation}
The output-realized energy transports are
\begin{align}
M_{\mathrm{out},k}^{(R)}&:=W_{k+1}U_k^{(R)}=U_{k+1}\Sigma_{k+1}V_{k+1}^\top U_k^{(R)},\label{eq:M-out}\\
M_{\mathrm{out,total},k}^{(R)}&:=W_{k+1}U_k^{(R)}\Sigma_k^{(R)}.\label{eq:M-out-total}
\end{align}
These matrices have physical output rows and source singular-mode columns. When both sides are required to be in physical channel coordinates, we use the physical-input-realized transports
\begin{align}
M_{\mathrm{phys},k}^{(R)}&:=W_{k+1}U_k^{(R)}\Sigma_k^{(R)}(V_k^{(R)})^\top,\label{eq:M-phys-phys-source}\\
M_{\mathrm{phys},k}^{(R_s,R_t)}&:=W_{k+1}^{[R_t]}U_k^{(R_s)}\Sigma_k^{(R_s)}(V_k^{(R_s)})^\top.\label{eq:M-phys-phys-trunc}
\end{align}
Any one of these matrices may be selected as the interface operator $A_k^{(R)}$, provided its coordinate interpretation is reported.
The choice determines the interpretation of its rows and columns: $M^{\mathrm{ang}}$, $M_{\mathrm{src}}$, $M_{\mathrm{tgt}}$, and $M_{\mathrm{total}}$ live in latent spectral coordinates; $M_{\mathrm{out,ang}}$, $M_{\mathrm{out}}$, and $M_{\mathrm{out,total}}$ realize rows in physical output coordinates while keeping source-mode columns; and $M_{\mathrm{phys}}$ realizes both rows and columns in physical channel coordinates. The scale-free physical panels in Section~\ref{sec:empirical} use $M_{\mathrm{out,ang}}$ rather than the latent $R\times R$ matrix $M^{\mathrm{ang}}$. SRS and hub columns are source-mode supports for output-realized matrices and physical input-channel supports for $M_{\mathrm{phys}}$.
\end{definition}

\begin{definition}[Physical Alignment Matrix]
\label{def:physical-alignment-matrix}
Let $A_k^{(R)}\in\R^{m\times n}$ be one of the transport matrices in Definition~\ref{def:transport-variants}.
Given permutation matrices $\Pirow\in\R^{m\times m}$ and $\Picol\in\R^{n\times n}$, define
\begin{equation}\label{eq:M-phy}
\widehat M_{\mathrm{phy},k}:=\Pirow A_k^{(R)}\Picol^\top.
\end{equation}
A pair $(\Pirow,\Picol)$ is a \emph{physical ordering} relative to a physical alignment structure if the permuted matrix admits the row and column group structure specified in Definition~\ref{def:physical-alignment-structure}.
\end{definition}

\begin{definition}[Experimental transport aliases]
\label{def:experimental-transport-aliases}
The alignment figures use two shorthand matrix names.  The symbol $M_s$ denotes the output-realized scale-free angular transport $M_{\mathrm{out,ang},k}^{(R)}$ in \eqref{eq:M-out-ang}, after the same truncation and row/column permutation used for the corresponding interface.  The symbol $M$ denotes the physical or energy-realized transport used in the experiment, typically $M_{\mathrm{out},k}^{(R)}$ or $M_{\mathrm{out,total},k}^{(R)}$ after the same permutation.  Thus the four displayed panels in the alignment galleries have the following mathematical meanings:
\[
\text{permuted }M_s \leftrightarrow \Pirow M_{\mathrm{out,ang},k}^{(R)} \Picol^\top,
\qquad
E_r(M_s) \leftrightarrow E_{\mathcal R,\mathcal C}(\Pirow M_{\mathrm{out,ang},k}^{(R)} \Picol^\top),
\]
\[
\text{permuted }M \leftrightarrow \widehat M_{\mathrm{phy},k},
\qquad
E_r(M) \leftrightarrow E_{\mathcal R,\mathcal C}(\widehat M_{\mathrm{phy},k}).
\]
The $M_s$ panels test angular organization before singular-value weighting while still displaying physical output rows; the $M$ panels test the physically realized transport after singular-value weighting and output realization. If the displayed $M$ uses $W_{k+1}U_k^{(R)}$ rather than the target-truncated $W_{k+1}^{[R]}U_k^{(R)}$, then the target-tail contribution is part of the measured residual and is controlled by the truncation bounds in Theorem~\ref{thm:truncation-transfer} and Corollary~\ref{cor:Er-window-robustness}.
\end{definition}

\begin{table}[t]
\centering
\small
\begin{tabularx}{\textwidth}{L{0.24\textwidth}L{0.22\textwidth}L{0.24\textwidth}Y}
\toprule
Transport & Row coordinates & Column coordinates & Incidence interpretation \\
\midrule
$M^{\mathrm{ang}}$ & target/input modes & source/output modes & latent subspace incidence only \\
$M_{\mathrm{out}}$, $M_{\mathrm{out,total}}$ & physical output channels & source singular modes & source-mode support, not physical input-channel support \\
$M_{\mathrm{phys}}$, $T_{\mathrm{phys}}$ & physical output channels & physical input channels & physical input-output channel incidence \\
\bottomrule
\end{tabularx}
\caption{Coordinate interpretation of the transport matrices.  SRS and hub variables are mode-support objects for output-realized matrices with source-mode columns, and physical channel-support objects only for physical-input-realized transports.}
\label{tab:transport-coordinate-types}
\end{table}

\subsection{Mode-profile grouping and active-support extraction}
\label{subsec:mode-profile-extraction}

The physical alignment structure makes the row and support choices explicit.  It consists of row groups and active supports that can be extracted from static SVD data or proposed by a deterministic numerical routine.  Definition~\ref{def:mode-profile-row-partition} gives one sufficient, margin-stable row-assignment rule.  When cosine or spectral clustering is used in the figures, the clustering step is treated as a fixed, predeclared proposal of the partition; the formal certificate is computed after that partition is fixed and the relevant margins are measured.  No theorem below treats an arbitrary post-hoc clustering as intrinsically stable unless the corresponding row-profile or clustering eigengap margins are supplied.

\begin{definition}[Mode-profile row partition]
\label{def:mode-profile-row-partition}
Let
\[
Y_{k+1}^{(R)}:=U_{k+1}^{(R)}\Sigma_{k+1}^{(R)}\in\R^{d\times R}
\]
be the dominant output mode profile of the succeeding layer.  For each physical output row $r$ and dominant mode $a$, define the modal energy score
\[
\omega_a(r):=\bigl(Y_{k+1}^{(R)}\bigr)_{r,a}^2
=\sigma_a(W_{k+1})^2\,(U_{k+1})_{r,a}^2.
\]
Fix thresholds $\theta_{\mathrm{row}}\ge0$ and $\mu_{\mathrm{row}}\ge0$.  Let $a_*(r)$ be the smallest maximizer of $\omega_a(r)$ and let $\omega_{(2)}(r)$ be the second-largest modal score.  The canonical mode-profile partition assigns
\[
r\in\mathcal R_{a_*(r)}
\quad\text{if}\quad
\omega_{a_*(r)}(r)\ge \theta_{\mathrm{row}}
\quad\text{and}\quad
\omega_{a_*(r)}(r)-\omega_{(2)}(r)\ge\mu_{\mathrm{row}}.
\]
Rows failing either test are assigned to $\mathcal R_0$.  Thus $\mathcal R_1,\dots,\mathcal R_R$ are signal row groups and $\mathcal R_0$ is the residual/noise row group.
\end{definition}

\begin{remark}[Rank-window size versus number of channel groups]
Throughout the physical-alignment sections, $R$ denotes a spectral truncation rank, while $K$ denotes the number of signal row groups in a physical alignment structure.  In the canonical mode-profile partition one may take $K=R$, but clustering or coarsening may use a different group count.  The notation keeps these two roles separate.
\end{remark}

\begin{definition}[Energy-threshold active columns]
\label{def:energy-threshold-active-columns}
Let $\widehat M$ be a physical alignment matrix and let $\mathcal R_i$ be a signal row group.  Define the column-energy score
\[
q_i(c):=\|\widehat M[\mathcal R_i,\{c\}]\|_2^2,
\qquad c=1,\dots,n.
\]
For an energy fraction $\tau_i\in(0,1]$, define $\mathcal C_i^{(\tau_i)}$ as the lexicographically tie-broken smallest set of largest-scoring columns satisfying
\[
\sum_{c\in\mathcal C_i^{(\tau_i)}}q_i(c)
\ge
\tau_i\sum_{c=1}^n q_i(c).
\]
For a prescribed support size $s_i$, the fixed-size top-$s_i$ rule in Definition~\ref{def:physical-alignment-structure} is obtained by taking the $s_i$ largest values of $q_i(c)$.  The $25\mathrm{ER}$ and $50\mathrm{ER}$ alignment figures use this same energy-threshold principle at the rank-window level.
\end{definition}

\begin{lemma}[Stability of mode-profile row assignments]
\label{lem:row-group-stability}
Let $Y,\widetilde Y\in\R^{d\times R}$ be two dominant output mode profiles.  Suppose
\[
\max_{r,a}|Y_{r,a}-\widetilde Y_{r,a}|\le\delta,
\qquad
\max_{r,a}\max\{|Y_{r,a}|,|\widetilde Y_{r,a}|\}\le B.
\]
Then every modal score changes by at most $2B\delta$:
\[
|Y_{r,a}^2-\widetilde Y_{r,a}^2|\le 2B\delta.
\]
Consequently, if a row $r$ satisfies
\[
\omega_{a_*(r)}(r)\ge \theta_{\mathrm{row}}+2B\delta,
\qquad
\omega_{a_*(r)}(r)-\omega_{(2)}(r)>4B\delta,
\]
then its winning signal-group assignment is unchanged under the perturbation from $Y$ to $\widetilde Y$.
\end{lemma}

\begin{proof}
Fix a row index $r$ and a profile coordinate $a$.  By assumption,
\[
|Y_{r,a}-\widetilde Y_{r,a}|\le\delta,
\qquad
|Y_{r,a}|\le B,
\qquad
|\widetilde Y_{r,a}|\le B.
\]
Using the factorization $u^2-v^2=(u-v)(u+v)$ gives
\[
|Y_{r,a}^2-\widetilde Y_{r,a}^2|
=|Y_{r,a}-\widetilde Y_{r,a}|\,|Y_{r,a}+\widetilde Y_{r,a}|
\le \delta( |Y_{r,a}|+|\widetilde Y_{r,a}|)
\le 2B\delta.
\]
Thus every squared profile score changes by at most $2B\delta$.

Suppose coordinate $a_*$ is the unique winning coordinate for row $r$ in the original profile, and every competitor $a\ne a_*$ satisfies
\[
Y_{r,a_*}^2-Y_{r,a}^2>4B\delta.
\]
After perturbation, the winning score can decrease by at most $2B\delta$ and a competitor score can increase by at most $2B\delta$.  Therefore
\[
\widetilde Y_{r,a_*}^2-\widetilde Y_{r,a}^2
\ge (Y_{r,a_*}^2-2B\delta)-(Y_{r,a}^2+2B\delta)>0.
\]
So the winning coordinate remains $a_*$.  The same argument applies to threshold membership: if a score is separated from the signal/noise threshold by more than $2B\delta$, then the perturbation cannot move it across the threshold.  Hence both the row assignment and the signal-versus-residual classification are stable under the stated margin conditions.
\end{proof}

\begin{remark}[Row-partition measurements]
The alignment heatmaps require a row and column ordering.  Definition~\ref{def:mode-profile-row-partition} and Lemma~\ref{lem:row-group-stability} specify a sufficient separation condition under which dominant mode profiles determine stable physical row groups.  Numerical clustering provides one implementation for finding such groups; the mathematical margin data consist of the resulting partition and its residual bounds.
\end{remark}

\begin{definition}[Tail energy and rank-truncated layer]
\label{def:tail-energy-truncated-layer}
For a matrix $W=U\Sigma V^\top$ and an integer $R\le d$, define the rank-$R$ SVD truncation
\[
W^{[R]}:=U^{(R)}\Sigma^{(R)}(V^{(R)})^\top
\]
and the discarded spectral energy
\[
E_{>R}(W):=\sum_{i=R+1}^{d}\sigma_i(W)^2=\|W-W^{[R]}\|_F^2.
\]
Let $E_R\in\R^{d\times R}$ denote the coordinate embedding of the first $R$ singular-coordinate axes.
\end{definition}

\begin{theorem}[Dominant-window transfer from full transport to truncated physical transport]
\label{thm:truncation-transfer}
Let $W_k,W_{k+1}\in\R^{d\times d}$ have SVDs as above.
Define the full output-total interface transport in the source singular coordinates of $W_k$ by
\[
\mathcal T_k:=W_{k+1}U_k\Sigma_k\in\R^{d\times d}.
\]
For source and target ranks $R_s,R_t$, define the zero-padded truncated transport
\[
\mathcal T_k^{(R_s,R_t)}:=W_{k+1}^{[R_t]}U_k^{(R_s)}\Sigma_k^{(R_s)}E_{R_s}^\top\in\R^{d\times d}.
\]
Then
\begin{equation}\label{eq:truncation-transfer-bound}
\|\mathcal T_k-\mathcal T_k^{(R_s,R_t)}\|_F
\le
\|W_{k+1}\|_2\,E_{>R_s}(W_k)^{1/2}
+
\|W_k\|_2\,E_{>R_t}(W_{k+1})^{1/2}.
\end{equation}
In particular, if $\|W_k\|_F^2=\|W_{k+1}\|_F^2=d$ and $R_s,R_t$ are $(1-\eps)$ energy truncation ranks for $W_k$ and $W_{k+1}$, then
\begin{equation}\label{eq:truncation-transfer-eps}
\|\mathcal T_k-\mathcal T_k^{(R_s,R_t)}\|_F
\le
\sqrt{\eps d}\bigl(\|W_{k+1}\|_2+\|W_k\|_2\bigr).
\end{equation}
\end{theorem}

\begin{proof}
Let
\[
U_k=[U_k^{(R_s)}\ U_{k,>R_s}],
\qquad
\Sigma_k=\begin{bmatrix}
\Sigma_k^{(R_s)}&0\\0&\Sigma_{k,>R_s}
\end{bmatrix},
\]
where $U_{k,>R_s}$ and $\Sigma_{k,>R_s}$ contain the discarded source singular directions.  Let $E_{R_s}$ be the embedding of the first $R_s$ singular-coordinate axes and $E_{>R_s}$ the embedding of the discarded axes.  Then
\[
U_k\Sigma_k
=
U_k^{(R_s)}\Sigma_k^{(R_s)}E_{R_s}^\top
+
U_{k,>R_s}\Sigma_{k,>R_s}E_{>R_s}^\top.
\]
Multiplying by $W_{k+1}$ gives the exact decomposition
\begin{align*}
\mathcal T_k
&=W_{k+1}U_k\Sigma_k\\
&=W_{k+1}U_k^{(R_s)}\Sigma_k^{(R_s)}E_{R_s}^\top
+W_{k+1}U_{k,>R_s}\Sigma_{k,>R_s}E_{>R_s}^\top.
\end{align*}
Subtracting the zero-padded truncated transport
\[
\mathcal T_k^{(R_s,R_t)}=W_{k+1}^{[R_t]}U_k^{(R_s)}\Sigma_k^{(R_s)}E_{R_s}^\top
\]
gives
\begin{align*}
\mathcal T_k-
\mathcal T_k^{(R_s,R_t)}
=&\ W_{k+1}U_{k,>R_s}\Sigma_{k,>R_s}E_{>R_s}^\top\\
&+(W_{k+1}-W_{k+1}^{[R_t]})U_k^{(R_s)}\Sigma_k^{(R_s)}E_{R_s}^\top.
\end{align*}
By the triangle inequality, the Frobenius norm of the difference is at most the sum of the Frobenius norms of these two terms.

For the source-tail term, use $\|AB\|_F\le\|A\|_2\|B\|_F$:
\begin{align*}
\|W_{k+1}U_{k,>R_s}\Sigma_{k,>R_s}E_{>R_s}^\top\|_F
&\le \|W_{k+1}\|_2
\|U_{k,>R_s}\Sigma_{k,>R_s}E_{>R_s}^\top\|_F.
\end{align*}
The matrices $U_{k,>R_s}$ and $E_{>R_s}$ have orthonormal columns, so left and right multiplication by them preserves the Frobenius norm of the diagonal block.  Hence
\[
\|U_{k,>R_s}\Sigma_{k,>R_s}E_{>R_s}^\top\|_F^2
=\|\Sigma_{k,>R_s}\|_F^2
=\sum_{i=R_s+1}^d\sigma_i(W_k)^2
=E_{>R_s}(W_k).
\]
Therefore the first term is bounded by
\[
\|W_{k+1}\|_2E_{>R_s}(W_k)^{1/2}.
\]

For the target-tail term, use $\|AB\|_F\le\|A\|_F\|B\|_2$ with
\[
A=W_{k+1}-W_{k+1}^{[R_t]},
\qquad
B=U_k^{(R_s)}\Sigma_k^{(R_s)}E_{R_s}^\top.
\]
Then
\begin{align*}
&\|(W_{k+1}-W_{k+1}^{[R_t]})U_k^{(R_s)}\Sigma_k^{(R_s)}E_{R_s}^\top\|_F\\
&\qquad\le
\|W_{k+1}-W_{k+1}^{[R_t]}\|_F
\|U_k^{(R_s)}\Sigma_k^{(R_s)}E_{R_s}^\top\|_2.
\end{align*}
By the definition of the truncated SVD,
\[
\|W_{k+1}-W_{k+1}^{[R_t]}\|_F^2
=
\sum_{i=R_t+1}^d\sigma_i(W_{k+1})^2
=E_{>R_t}(W_{k+1}).
\]
Also, orthogonal factors do not change the operator norm, and
\[
\|U_k^{(R_s)}\Sigma_k^{(R_s)}E_{R_s}^\top\|_2
=\|\Sigma_k^{(R_s)}\|_2
\le \|W_k\|_2.
\]
Thus the second term is bounded by
\[
\|W_k\|_2E_{>R_t}(W_{k+1})^{1/2}.
\]
Adding the two estimates proves \eqref{eq:truncation-transfer-bound}.  If $R_s$ and $R_t$ are $(1-\eps)$ energy ranks and both layers have Frobenius energy $d$, then
\[
E_{>R_s}(W_k)
\le \eps d,
\qquad
E_{>R_t}(W_{k+1})
\le \eps d.
\]
Substituting these two inequalities into \eqref{eq:truncation-transfer-bound} gives \eqref{eq:truncation-transfer-eps}.
\end{proof}

\begin{remark}[Finite-rank physical transport]
The matrix $\mathcal T_k$ records how the source singular coordinates of layer $k$ are physically realized after applying $W_{k+1}$. The theorem gives a deterministic approximation bound from this full interface object to its rank-truncated physical realization. Spectral compressibility and Cartan shortness therefore specify the finite-dimensional angular object on which the pairwise incidence structure is defined.
\end{remark}

\begin{definition}[Physical-input-realized full and truncated transports]
\label{def:physical-input-realized-transport}
The source-coordinate transport $\mathcal T_k=W_{k+1}U_k\Sigma_k$ has physical output rows and source singular-coordinate columns. The corresponding physical input-output interface is
\[
\mathcal T_{\mathrm{phys},k}:=W_{k+1}W_k=W_{k+1}U_k\Sigma_kV_k^\top.
\]
For source and target ranks $(R_s,R_t)$ define
\[
\mathcal T_{\mathrm{phys},k}^{(R_s,R_t)}
:=
W_{k+1}^{[R_t]}U_k^{(R_s)}\Sigma_k^{(R_s)}(V_k^{(R_s)})^\top.
\]
This matrix realizes both rows and columns in physical channel coordinates.
\end{definition}

\begin{corollary}[Physical-to-physical truncation error]
\label{cor:physical-physical-truncation}
Under the hypotheses of Theorem~\ref{thm:truncation-transfer},
\begin{equation}\label{eq:physical-physical-truncation}
\|\mathcal T_{\mathrm{phys},k}-\mathcal T_{\mathrm{phys},k}^{(R_s,R_t)}\|_F
\le
\|W_{k+1}\|_2E_{>R_s}(W_k)^{1/2}
+
\|W_k\|_2E_{>R_t}(W_{k+1})^{1/2}.
\end{equation}
Consequently, when $R_s,R_t$ are $(1-\eps)$ energy ranks and both layers have Frobenius energy $d$, the right-hand side is bounded by $\sqrt{\eps d}(\|W_{k+1}\|_2+\|W_k\|_2)$.
\end{corollary}

\begin{proof}
Since
\[
\mathcal T_{\mathrm{phys},k}=\mathcal T_kV_k^\top,
\qquad
\mathcal T_{\mathrm{phys},k}^{(R_s,R_t)}=\mathcal T_k^{(R_s,R_t)}V_k^\top,
\]
and $V_k$ is orthogonal, Frobenius invariance gives
\[
\|\mathcal T_{\mathrm{phys},k}-\mathcal T_{\mathrm{phys},k}^{(R_s,R_t)}\|_F
=
\|\mathcal T_k-\mathcal T_k^{(R_s,R_t)}\|_F.
\]
Theorem~\ref{thm:truncation-transfer} gives the source-mode bound the displayed bound.
\end{proof}

\begin{remark}[Mode incidence versus physical channel incidence]
\label{rem:mode-versus-physical-incidence}
For $M_{\mathrm{out}}$ and $M_{\mathrm{out,total}}$, active columns are dominant source-mode supports. For $M_{\mathrm{phys}}$ or $\mathcal T_{\mathrm{phys}}^{(R_s,R_t)}$, active columns are physical input-channel supports. Both are useful finite-dimensional certificates, but the column interpretation must be stated when reporting SRS and hub variables.
\end{remark}

\subsection{Physical alignment structures, active columns, and relational triples}

\begin{definition}[Physical alignment structure]
\label{def:physical-alignment-structure}
Let $\widehat M\in\R^{m\times n}$ be a physical alignment matrix.
A physical alignment structure consists, for some integer $K\ge1$, of the tuple
\[
\mathscr C=(\{\mathcal R_a\}_{a=0}^{K},\{\mathcal C_i\}_{i=1}^{K},s_1,\dots,s_K),
\]
where:
\begin{enumerate}[label=(\roman*),leftmargin=1.8em]
\item the rows are partitioned as
\[
\mathcal R_0\sqcup\mathcal R_1\sqcup\cdots\sqcup\mathcal R_K=\{1,\dots,m\};
\]
$\mathcal R_1,\dots,\mathcal R_K$ are signal groups and $\mathcal R_0$ is the residual/noise group;
\item $s_i\in\{1,\dots,n\}$ is a prescribed support size for group $i$; alternatively, $s_i$ may be the cardinality of an energy-threshold set $\mathcal C_i^{(\tau_i)}$ from Definition~\ref{def:energy-threshold-active-columns};
\item $\mathcal C_i\subseteq\{1,\dots,n\}$ is the active column set of group $i$ selected by the deterministic top-energy rule
\begin{equation}\label{eq:active-columns}
\mathcal C_i
\in
\argmax_{\substack{\mathcal C\subseteq\{1,\dots,n\}\\ |\mathcal C|=s_i}}
\sum_{c\in\mathcal C}\|\widehat M[\mathcal R_i,\{c\}]\|_2^2,
\end{equation}
with ties broken by lexicographic order.
\end{enumerate}
The row groups and active columns are included explicitly in the structure.
The margin residuals defined below quantify how well a trained interface satisfies the desired incidence structure.
\end{definition}

\begin{definition}[Active-column order gap]
\label{def:active-column-order-gap}
Let $\widehat M$ carry row groups $\mathcal R_1,\dots,\mathcal R_K$ and support sizes $s_i$.
For group $i$, define the column-energy score
\[
q_i(c;\widehat M):=\|\widehat M[\mathcal R_i,\{c\}]\|_2^2,
\qquad c=1,\dots,n.
\]
Let $\mathcal C_i(\widehat M)$ be the top-$s_i$ set selected by \eqref{eq:active-columns}.
The active-column order gap is
\begin{equation}\label{eq:active-column-gap}
\Gamma_i(\widehat M):=
\min_{c\in\mathcal C_i(\widehat M),\ c'\notin\mathcal C_i(\widehat M)}
\bigl(q_i(c;\widehat M)-q_i(c';\widehat M)\bigr),
\end{equation}
with the convention $\Gamma_i(\widehat M)=+\infty$ if $s_i=n$.
A physical alignment structure is column-separated if $\Gamma_i(\widehat M)>0$ for all signal groups.
\end{definition}

\begin{definition}[Pairwise relational triple]
\label{def:pairwise-triple}
For distinct signal groups $1\le i<j\le K$, define
\begin{align}
\Core_{i\setminus j}&:=\widehat M[\mathcal R_i,\mathcal C_i\setminus\mathcal C_j],\label{eq:core-i-j}\\
\Core_{j\setminus i}&:=\widehat M[\mathcal R_j,\mathcal C_j\setminus\mathcal C_i],\label{eq:core-j-i}\\
\Overlap_{i\cap j}&:=\widehat M[\mathcal R_i\cup\mathcal R_j,\mathcal C_i\cap\mathcal C_j].\label{eq:overlap-ij}
\end{align}
The pairwise relational object is
\begin{equation}\label{eq:pair-triple}
\widehat{\mathcal M}_{\mathrm{pair}}^{(i,j)}
:=
(\Core_{i\setminus j},\Core_{j\setminus i},\Overlap_{i\cap j}),
\end{equation}
and the full pairwise incidence structure is
\[
\widehat{\mathcal M}_{\mathrm{pair}}
:=
\{\widehat{\mathcal M}_{\mathrm{pair}}^{(i,j)}:1\le i<j\le K\}.
\]
\end{definition}

\begin{definition}[Pairwise margins and gaps]
\label{def:pairwise-margin-gap}
For a matrix $B$, define
\[
\sigma_{\min}^+(B):=
\begin{cases}
\min\{\sigma_t(B):\sigma_t(B)>0\},& B\neq 0,\\
0,& B=0.
\end{cases}
\]
For $1\le i<j\le K$, define the pairwise exclusive core margin
\begin{equation}\label{eq:mij}
m_{i,j}:=
\min\bigl(\sigma_{\min}^+(\Core_{i\setminus j}),\sigma_{\min}^+(\Core_{j\setminus i})\bigr),
\end{equation}
the coherent overlap
\[
o_{i,j}:=\|\Overlap_{i\cap j}\|_2,
\]
and the pairwise spectral gap
\begin{equation}\label{eq:Deltaij}
\Delta_\sigma(i,j):=m_{i,j}-o_{i,j}.
\end{equation}
The pair is nondegenerate if $m_{i,j}>0$.
\end{definition}


\section{Pairwise relational stability and block-sparse structure}
\label{sec:pairwise-stability}

The Physical GSA formulation treats angular alignment stability as a gap condition.
The first result is a calibration lemma: because the gap is defined by $\Delta_\sigma(i,j)=m_{i,j}-o_{i,j}$, the numerical threshold $1/3$ is an exact algebraic consequence of the chosen gap functional.
The substantive certificate content is supplied by the subsequent perturbation, incidence-structure, block-energy, and measurement theorems.

\begin{lemma}[Pairwise margin calibration and the one-third threshold]
\label{thm:one-third-threshold}
Fix a nondegenerate pair $(i,j)$ and define $\Delta_\sigma(i,j)$ by \eqref{eq:Deltaij}.
Then the gap-based stability condition
\begin{equation}\label{eq:gap-stability}
o_{i,j}<\frac12\Delta_\sigma(i,j)
\end{equation}
is equivalent to
\begin{equation}\label{eq:one-third}
o_{i,j}<\frac13 m_{i,j}.
\end{equation}
When these hold, $\Delta_\sigma(i,j)>0$.
Moreover, if $\Overlap_{i\cap j}\neq 0$ and
\[
\gamma_{i,j}:=\frac{\|\Overlap_{i\cap j}\|_2}{\|\Overlap_{i\cap j}\|_F}\in(0,1],
\]
then \eqref{eq:one-third} is equivalent to the Frobenius-energy form
\begin{equation}\label{eq:frobenius-overlap-law}
\|\Overlap_{i\cap j}\|_F<\frac{1}{3\gamma_{i,j}}m_{i,j}.
\end{equation}
\end{lemma}

\begin{proof}
Let $o:=o_{i,j}$ and $m:=m_{i,j}$.  Since the pair is nondegenerate, $m>0$; by definition $o=\|\Overlap_{i\cap j}\|_2\ge0$ and
\[
\Delta_\sigma(i,j)=m-o.
\]
We prove both implications.

Assume first the gap-based condition \eqref{eq:gap-stability}.  Substituting $\Delta_\sigma(i,j)=m-o$ gives
\[
o<\frac12(m-o).
\]
Multiplying by $2$ gives $2o<m-o$, hence $3o<m$.  Dividing by $3$ gives $o<m/3$, which is \eqref{eq:one-third}.

Conversely, assume \eqref{eq:one-third}, i.e. $3o<m$.  Then
\[
2o<m-o=\Delta_\sigma(i,j).
\]
Dividing by $2$ gives $o<\frac12\Delta_\sigma(i,j)$, which is \eqref{eq:gap-stability}.  The same inequality $3o<m$ implies $o<m$, and therefore
\[
\Delta_\sigma(i,j)=m-o>0.
\]

It remains to prove the Frobenius form.  If $\Overlap_{i\cap j}\ne0$, then $\|\Overlap_{i\cap j}\|_F>0$, and the definition of $\gamma_{i,j}$ gives
\[
o=\|\Overlap_{i\cap j}\|_2=\gamma_{i,j}\|\Overlap_{i\cap j}\|_F.
\]
Substituting this identity into $o<m/3$ yields
\[
\gamma_{i,j}\|\Overlap_{i\cap j}\|_F<\frac13m,
\]
and division by $\gamma_{i,j}>0$ gives \eqref{eq:frobenius-overlap-law}.  The reverse implication follows by multiplying \eqref{eq:frobenius-overlap-law} by $\gamma_{i,j}$, so the two forms are equivalent when $\Overlap_{i\cap j}\ne0$.
\end{proof}
\subsection{Perturbative stability}
\begin{theorem}[Perturbative stability of pairwise incidence structure]
\label{thm:pairwise-perturbative-stability}
Let a pair $(i,j)$ satisfy $3o_{i,j}<m_{i,j}$.
Suppose the two exclusive core blocks and the overlap block are perturbed by matrices of operator norm at most $\eta$, and suppose the positive ranks of the two exclusive core blocks are preserved.
Then the perturbed quantities satisfy
\[
o'_{i,j}\le o_{i,j}+\eta,
\qquad
m'_{i,j}\ge m_{i,j}-\eta.
\]
Consequently the perturbed pair still satisfies the one-third threshold if
\begin{equation}\label{eq:stability-margin-condition}
4\eta < m_{i,j}-3o_{i,j}.
\end{equation}
\end{theorem}

\begin{proof}
Let $O:=\Overlap_{i\cap j}$ and let $E_O$ be its perturbation, with $\|E_O\|_2\le\eta$.  The perturbed overlap is $O'=O+E_O$.  By the triangle inequality for the operator norm,
\[
o'_{i,j}=\|O'\|_2\le \|O\|_2+\|E_O\|_2\le o_{i,j}+\eta.
\]

Now consider one of the two exclusive core blocks, denoted $B$, and its perturbation $B'=B+E_B$ with $\|E_B\|_2\le\eta$.  Let $r_B=\operatorname{rank}(B)>0$.  The hypothesis says that the positive rank is preserved, so $\operatorname{rank}(B')=r_B$.  Weyl's singular-value perturbation inequality for rectangular matrices gives, for every index $t$,
\[
|\sigma_t(B')-\sigma_t(B)|\le\|E_B\|_2.
\]
Applying this to $t=r_B$, the index of the smallest positive singular value of $B$, yields
\[
\sigma_{\min}^+(B')=\sigma_{r_B}(B')
\ge \sigma_{r_B}(B)-\|E_B\|_2
\ge \sigma_{\min}^+(B)-\eta.
\]
The same argument applies to the other exclusive core block.  Taking the minimum of the two lower bounds gives
\[
m'_{i,j}\ge m_{i,j}-\eta.
\]

The perturbed one-third condition is $3o'_{i,j}<m'_{i,j}$.  The bounds just proved imply the sufficient condition
\[
3(o_{i,j}+\eta)<m_{i,j}-\eta.
\]
Rearranging gives
\[
3o_{i,j}+4\eta<m_{i,j},
\]
which is equivalent to \eqref{eq:stability-margin-condition}.  Under this condition, $3o'_{i,j}<m'_{i,j}$ follows, so the perturbed pair remains in the one-third threshold.
\end{proof}

\begin{definition}[Global core-overlap-noise decomposition]
\label{def:global-core-overlap-noise}
For a matrix carrying such a structure $\widehat M$, define the dedicated support of group $i$ by
\[
\mathcal C_i^{\mathrm{ded}}
:=
\mathcal C_i\setminus\bigcup_{j\neq i}\mathcal C_j,
\]
the groupwise shared support by
\[
\mathcal C_i^{\mathrm{sh}}
:=
\mathcal C_i\cap\bigcup_{j\neq i}\mathcal C_j,
\]
and the global shared column set by
\[
\mathcal C^{\mathrm{sh}}
:=
\bigcup_{i=1}^K\mathcal C_i^{\mathrm{sh}}
=
\bigcup_{1\le i<j\le K}(\mathcal C_i\cap\mathcal C_j).
\]
Let $\Pi_{\mathrm{core}}$ be the coordinate mask selecting precisely the blocks
$\widehat M[\mathcal R_i,\mathcal C_i^{\mathrm{ded}}]$ for $i=1,\dots,K$.
Let $\Pi_{\mathrm{overlap}}$ be the coordinate mask selecting precisely the groupwise shared blocks
$\widehat M[\mathcal R_i,\mathcal C_i^{\mathrm{sh}}]$ for $i=1,\dots,K$.
Thus a shared column contributes to the overlap component only for the signal groups whose active set actually contains that column; entries in unrelated signal rows remain part of the residual/noise component unless they are selected by that row group.
Define
\begin{align}
M_{\mathrm{core}}&:=\Pi_{\mathrm{core}}(\widehat M),\label{eq:M-core}\\
M_{\mathrm{overlap}}&:=\Pi_{\mathrm{overlap}}(\widehat M),\label{eq:M-overlap}\\
M_{\mathrm{noise}}&:=\widehat M-M_{\mathrm{core}}-M_{\mathrm{overlap}}.\label{eq:M-noise}
\end{align}
Thus
\[
\widehat M=M_{\mathrm{core}}+M_{\mathrm{overlap}}+M_{\mathrm{noise}}.
\]
\end{definition}

\begin{definition}[Static channel incidence structure induced by a physical alignment structure]
\label{def:static-channel-incidence}
Let a physical alignment structure be fixed, with signal row groups $\mathcal R_1,\dots,\mathcal R_K$, active column sets $\mathcal C_1,\dots,\mathcal C_K$, and core/overlap/noise decomposition from Definition~\ref{def:global-core-overlap-noise}.  The \emph{static channel incidence structure} is the finite incidence structure
\[
\mathcal T_{\mathrm{ch}}(\widehat M)
=\bigl(\mathcal V_{\mathrm{out}},\mathcal V_{\mathrm{in}},\mathcal E_{\mathrm{ded}},\mathcal E_{\mathrm{sh}},\Omega_{\mathrm{noise}}\bigr),
\]
where
\[
\mathcal V_{\mathrm{out}}:=\{1,\dots,K\},
\qquad
\mathcal V_{\mathrm{in}}:=\bigcup_{i=1}^K\mathcal C_i,
\]
\[
\mathcal E_{\mathrm{ded}}:=\{(i,c):c\in\mathcal C_i\text{ and }c\notin\mathcal C_j\text{ for all }j\ne i\},
\]
\[
\mathcal E_{\mathrm{sh}}:=\{(i,c):c\in\mathcal C_i\text{ and }c\in\mathcal C_j\text{ for at least one }j\ne i\},
\]
and $\Omega_{\mathrm{noise}}$ is the residual coordinate mask, i.e. the complement of the core and overlap masks in the ambient coordinate grid.  We do not claim stability of the exact numerical support $\supp(M_{\mathrm{noise}})$ under arbitrary small dense perturbations; if an exact residual support is needed, a thresholded support $\supp_{\tau}(M_{\mathrm{noise}}):=\{(a,b): |(M_{\mathrm{noise}})_{ab}|\ge \tau\}$ and an entrywise perturbation margin must be specified.  A column $c\in\mathcal V_{\mathrm{in}}$ with degree $|\{i:c\in\mathcal C_i\}|\ge2$ is called a \emph{shared-support column} or \emph{hub column}.  This is a finite bipartite incidence object in the declared column coordinates.
\end{definition}

\begin{definition}[Shared-support incidence graph]
\label{def:shared-support-graph}
Given the static channel incidence structure of Definition~\ref{def:static-channel-incidence}, define the bipartite graph
\[
\mathcal G_{\mathrm{SRS}}=(V_{\mathrm{grp}}\sqcup V_{\mathrm{sup}},E_{\mathrm{SRS}})
\]
by
\[
V_{\mathrm{grp}}:=\{1,\dots,K\},
\qquad
V_{\mathrm{sup}}:=\bigcup_{i=1}^K\mathcal C_i,
\]
and
\[
(i,c)\in E_{\mathrm{SRS}}
\quad\Longleftrightarrow\quad
c\in\mathcal C_i.
\]
The degree of a support column is
\[
\deg(c):=|\{i:c\in\mathcal C_i\}|.
\]
Columns with $\deg(c)\ge2$ are shared-support or hub columns.
\end{definition}

\begin{proposition}[Shared-support graph and energy-degree bound]
\label{thm:shared-support-graph}
Let $\mathcal G_{\mathrm{SRS}}$ be the graph from Definition~\ref{def:shared-support-graph}.  Then:
\begin{enumerate}[label=(G\arabic*),leftmargin=2.1em]
\item The graph is a deterministic function of the active column sets $\mathcal C_1,\dots,\mathcal C_K$.  Therefore any perturbation preserving all active column sets preserves $\mathcal G_{\mathrm{SRS}}$, all support degrees, and the hub set.
\item For a support column $c$, define its groupwise column energy
\[
E_c:=\sum_{i:c\in\mathcal C_i}\|\widehat M[\mathcal R_i,\{c\}]\|_2^2 .
\]
If $\deg(c)=q\ge1$ and every incident group has column strength at least $\epsilon_c$, meaning
\[
\|\widehat M[\mathcal R_i,\{c\}]\|_2\ge\epsilon_c
\qquad\text{for all }i\text{ with }c\in\mathcal C_i,
\]
then
\begin{equation}\label{eq:support-degree-energy}
\epsilon_c\le \sqrt{\frac{E_c}{q}}.
\end{equation}
\end{enumerate}
\end{proposition}

\begin{proof}
For (G1), the vertex set $V_{\mathrm{grp}}$ is fixed by the row grouping, and $V_{\mathrm{sup}}$ and $E_{\mathrm{SRS}}$ are defined entirely by membership in the sets $\mathcal C_i$.  If a perturbation preserves each $\mathcal C_i$, then for every pair $(i,c)$ the truth value of $c\in\mathcal C_i$ is unchanged.  Hence every edge is unchanged, and the degree
\[
\deg(c)=|\{i:c\in\mathcal C_i\}|
\]
is unchanged for every support column.  The hub set is the set of columns with degree at least two, so it is unchanged as well.

For (G2), if $c$ is incident to exactly $q$ groups and each incident group has norm at least $\epsilon_c$, then
\[
\|\widehat M[\mathcal R_i,\{c\}]\|_2^2\ge\epsilon_c^2
\]
for each of the $q$ incident groups.  Summing over those groups gives
\[
E_c=\sum_{i:c\in\mathcal C_i}\|\widehat M[\mathcal R_i,\{c\}]\|_2^2
\ge q\epsilon_c^2.
\]
Since $q>0$, division by $q$ and taking square roots gives \eqref{eq:support-degree-energy}.
\end{proof}

\begin{proposition}[Orthogonal mask decomposition]
\label{prop:orthogonal-mask-decomposition}
The three matrices in Definition~\ref{def:global-core-overlap-noise} have disjoint coordinate support.
Consequently,
\begin{equation}\label{eq:pythagorean-decomp}
\|\widehat M\|_F^2
=
\|M_{\mathrm{core}}\|_F^2+
\|M_{\mathrm{overlap}}\|_F^2+
\|M_{\mathrm{noise}}\|_F^2.
\end{equation}
\end{proposition}

\begin{proof}
For each signal group $i$, the sets $\mathcal C_i^{\mathrm{ded}}$ and $\mathcal C_i^{\mathrm{sh}}$ are disjoint because
\[
\mathcal C_i^{\mathrm{ded}}=\mathcal C_i\setminus\bigcup_{j\ne i}\mathcal C_j,
\qquad
\mathcal C_i^{\mathrm{sh}}=\mathcal C_i\cap\bigcup_{j\ne i}\mathcal C_j.
\]
Hence no coordinate in a row block $\mathcal R_i$ can be selected simultaneously by the core mask and the overlap mask.  The residual mask defining $M_{\mathrm{noise}}$ is the complement, in the full coordinate set of $\widehat M$, of the union of the core and overlap masks.  Therefore the coordinate supports of $M_{\mathrm{core}}$, $M_{\mathrm{overlap}}$, and $M_{\mathrm{noise}}$ are pairwise disjoint.

The Frobenius inner product of two matrices $A,B$ with disjoint coordinate support is
\[
\langle A,B\rangle_F=\sum_{a,b}A_{ab}B_{ab}=0,
\]
because for every coordinate $(a,b)$ at least one of $A_{ab}$ or $B_{ab}$ is zero.  Since
\[
\widehat M=M_{\mathrm{core}}+M_{\mathrm{overlap}}+M_{\mathrm{noise}},
\]
expanding the squared Frobenius norm and using the vanishing pairwise inner products gives
\[
\|\widehat M\|_F^2
=\|M_{\mathrm{core}}\|_F^2+
 \|M_{\mathrm{overlap}}\|_F^2+
 \|M_{\mathrm{noise}}\|_F^2.
\]
This is \eqref{eq:pythagorean-decomp}.
\end{proof}

\begin{theorem}[Core--overlap projection theorem]
\label{thm:core-overlap-projection}
Fix the row groups and active column sets used in Definition~\ref{def:global-core-overlap-noise}.  Let
\[
\mathcal S_{\mathrm{co}}
:=\supp(M_{\mathrm{core}})\cup\supp(M_{\mathrm{overlap}})
\]
and define the coordinate subspace
\[
\mathcal U_{\mathrm{co}}
:=\{X\in\R^{m\times n}:\supp(X)\subseteq \mathcal S_{\mathrm{co}}\}.
\]
Then
\begin{equation}\label{eq:core-overlap-projection}
M_{\mathrm{core}}+M_{\mathrm{overlap}}
=
\argmin_{X\in\mathcal U_{\mathrm{co}}}\|\widehat M-X\|_F.
\end{equation}
The minimizer is unique, and
\begin{equation}\label{eq:core-overlap-distance}
\operatorname{dist}_F(\widehat M,\mathcal U_{\mathrm{co}})
:=\min_{X\in\mathcal U_{\mathrm{co}}}\|\widehat M-X\|_F
=
\|M_{\mathrm{noise}}\|_F.
\end{equation}
\end{theorem}

\begin{proof}
Let $\Pi_{\mathrm{co}}$ be the coordinate projection onto $\mathcal S_{\mathrm{co}}$, i.e.
\[
(\Pi_{\mathrm{co}}Y)_{ab}=
\begin{cases}
Y_{ab}, & (a,b)\in\mathcal S_{\mathrm{co}},\\
0, & (a,b)\notin\mathcal S_{\mathrm{co}}.
\end{cases}
\]
By Definition~\ref{def:global-core-overlap-noise}, the core and overlap masks are disjoint and their union is exactly $\mathcal S_{\mathrm{co}}$.  Therefore
\[
\Pi_{\mathrm{co}}\widehat M=M_{\mathrm{core}}+M_{\mathrm{overlap}},
\qquad
(I-\Pi_{\mathrm{co}})\widehat M=M_{\mathrm{noise}}.
\]
Now take any $X\in\mathcal U_{\mathrm{co}}$.  Since $X$ has support contained in $\mathcal S_{\mathrm{co}}$, the matrices $\Pi_{\mathrm{co}}\widehat M-X$ and $(I-\Pi_{\mathrm{co}})\widehat M$ have disjoint coordinate supports.  Hence they are orthogonal in the Frobenius inner product.  Therefore
\begin{align*}
\|\widehat M-X\|_F^2
&=\|\Pi_{\mathrm{co}}\widehat M-X+(I-\Pi_{\mathrm{co}})\widehat M\|_F^2\\
&=\|\Pi_{\mathrm{co}}\widehat M-X\|_F^2+\|(I-\Pi_{\mathrm{co}})\widehat M\|_F^2\\
&=\|\Pi_{\mathrm{co}}\widehat M-X\|_F^2+\|M_{\mathrm{noise}}\|_F^2.
\end{align*}
The second term is independent of $X$, and the first term is minimized uniquely by $X=\Pi_{\mathrm{co}}\widehat M=M_{\mathrm{core}}+M_{\mathrm{overlap}}$.  Substituting this minimizer gives \eqref{eq:core-overlap-distance}.
\end{proof}

\begin{corollary}[Global sufficient check]
\label{cor:global-sufficient-check}
Let $m_*:=\min_{1\le i<j\le K}m_{i,j}$.
If $m_*>0$ and
\[
\|M_{\mathrm{overlap}}\|_F<\frac13m_*,
\]
then every pairwise overlap satisfies the one-third threshold \eqref{eq:one-third}.
\end{corollary}

\begin{proof}
Fix a nondegenerate pair $i<j$.  The pairwise overlap block $\Overlap_{i\cap j}$ is one of the coordinate subblocks contained in the global overlap matrix $M_{\mathrm{overlap}}$.  Therefore its Frobenius norm is bounded by the Frobenius norm of the whole overlap component:
\[
\|\Overlap_{i\cap j}\|_F^2
\le \sum_{(a,b)\in\supp(M_{\mathrm{overlap}})} |(M_{\mathrm{overlap}})_{ab}|^2
=\|M_{\mathrm{overlap}}\|_F^2.
\]
Since the operator norm is bounded by the Frobenius norm for every finite matrix,
\[
\|\Overlap_{i\cap j}\|_2
\le\|\Overlap_{i\cap j}\|_F
\le\|M_{\mathrm{overlap}}\|_F.
\]
The hypothesis gives $\|M_{\mathrm{overlap}}\|_F<m_*/3$, and the definition $m_*:=\min_{p<q}m_{p,q}$ gives $m_*\le m_{i,j}$.  Hence
\[
\|\Overlap_{i\cap j}\|_2<\frac13m_*\le \frac13m_{i,j}.
\]
This is exactly the one-third inequality \eqref{eq:one-third} for the pair $(i,j)$.  Since the pair was arbitrary, the inequality holds for every pairwise overlap.
\end{proof}

\begin{theorem}[Static channel incidence structure from physical alignment]
\label{thm:physical-incidence-margin}
Let $\widehat M$ carry a physical alignment structure with fixed row groups and active column sets.
Assume that all nondegenerate signal pairs satisfy the one-third threshold and that $\|M_{\mathrm{noise}}\|_F\le\varepsilon_{\mathrm{noise}}$.
Then the margin-verified interface has the following static channel incidence structure:
\begin{enumerate}[label=(T\arabic*),leftmargin=2.1em]
\item the signal part decomposes into exclusive core blocks and controlled shared-support blocks;
\item coherent cross-talk between any pair is bounded by its exclusive margin;
\item unstructured interaction outside core and shared support has Frobenius energy at most $\varepsilon_{\mathrm{noise}}$;
\item the total energy splits according to \eqref{eq:pythagorean-decomp}.
\end{enumerate}
Moreover, with the row groups and active column sets held fixed, if perturbations obey \eqref{eq:stability-margin-condition} for every pair and the noise perturbation has Frobenius norm at most $\eta_F$, then the same fixed-support incidence certificate remains valid with noise bound $\varepsilon_{\mathrm{noise}}+\eta_F$.  If the active column sets are re-extracted after perturbation, the additional active-column gap condition of Theorem~\ref{thm:static-structure-stability}(C1) is also required.
\end{theorem}

\begin{proof}
We prove the four stated incidence-structure conclusions and then the perturbative assertion.
\begin{enumerate}[label=(T\arabic*),leftmargin=2.1em]
\item \textbf{Exclusive core blocks and shared-support blocks.}
The row partition and active-column sets are part of the physical alignment structure in Definition~\ref{def:physical-alignment-structure}.  For every signal pair $i<j$, Definition~\ref{def:pairwise-triple} decomposes the pairwise columns into exclusive pieces $\mathcal C_i\setminus\mathcal C_j$ and $\mathcal C_j\setminus\mathcal C_i$ and a shared piece $\mathcal C_i\cap\mathcal C_j$.  The corresponding submatrices are precisely $\Core_{i\setminus j}$, $\Core_{j\setminus i}$, and $\Overlap_{i\cap j}$.  Globally, Definition~\ref{def:global-core-overlap-noise} collects the columns active for exactly one signal group into $M_{\mathrm{core}}$ and the columns shared by at least two signal groups into $M_{\mathrm{overlap}}$.  Hence the signal part is a union of exclusive core blocks and controlled shared-support blocks.

\item \textbf{Coherent cross-talk is bounded by exclusive margins.}
By hypothesis, every nondegenerate signal pair satisfies the one-third threshold.  For such a pair,
\[
\|\Overlap_{i\cap j}\|_2=o_{i,j}<\frac13m_{i,j}.
\]
Lemma~\ref{thm:one-third-threshold} shows this is equivalent to the gap condition
\[
o_{i,j}<\frac12\Delta_\sigma(i,j).
\]
Thus the coherent operator-norm cross-talk carried by the shared support of the pair is strictly smaller than both the one-third exclusive margin and one-half of the induced gap.  Degenerate pairs have zero exclusive margin by definition and lie outside the pairwise stability condition.

\item \textbf{Unstructured interaction is bounded.}
Definition~\ref{def:global-core-overlap-noise} defines
\[
M_{\mathrm{noise}}=\widehat M-M_{\mathrm{core}}-M_{\mathrm{overlap}}.
\]
The theorem assumes $\|M_{\mathrm{noise}}\|_F\le\varepsilon_{\mathrm{noise}}$.  Therefore every coordinate not assigned to exclusive core or structured shared support has total Frobenius energy at most $\varepsilon_{\mathrm{noise}}^2$ and Frobenius norm at most $\varepsilon_{\mathrm{noise}}$.

\item \textbf{Energy splitting.}
Proposition~\ref{prop:orthogonal-mask-decomposition} states that the three coordinate masks defining $M_{\mathrm{core}}$, $M_{\mathrm{overlap}}$, and $M_{\mathrm{noise}}$ are disjoint.  Therefore the three matrices are orthogonal for the Frobenius inner product.  Applying the Pythagorean identity in the Frobenius Hilbert space gives
\[
\|\widehat M\|_F^2
=
\|M_{\mathrm{core}}\|_F^2+
\|M_{\mathrm{overlap}}\|_F^2+
\|M_{\mathrm{noise}}\|_F^2,
\]
which is \eqref{eq:pythagorean-decomp}.
\end{enumerate}
For the perturbation statement, the coordinate masks are interpreted as fixed.  Assume each pairwise perturbation obeys \eqref{eq:stability-margin-condition}.  Theorem~\ref{thm:pairwise-perturbative-stability} then implies that every perturbed nondegenerate pair still satisfies the one-third threshold on those fixed blocks, so the fixed pairwise support relations and the controlled-overlap inequalities remain valid.  If the active sets are re-selected from the perturbed matrix, one must additionally impose the active-column gap condition from Theorem~\ref{thm:static-structure-stability}(C1).  If the perturbation of the noise component has Frobenius norm at most $\eta_F$, then
\[
\|M_{\mathrm{noise}}^{\,\prime}\|_F
\le
\|M_{\mathrm{noise}}\|_F+
\|M_{\mathrm{noise}}^{\,\prime}-M_{\mathrm{noise}}\|_F
\le
\varepsilon_{\mathrm{noise}}+\eta_F.
\]
Thus the same incidence structure persists with the stated enlarged noise bound.
\end{proof}

\begin{proposition}[Shared-support energy bound]
\label{thm:hub-scaling}
Let a column $h$ be active for $q$ signal groups, with group-wise column segments
\[
h_i:=\widehat M[\mathcal R_i,\{h\}],
\qquad i\in I_h,
\qquad |I_h|=q.
\]
Assume the total stable energy of this column is bounded by $E_{\max}$:
\[
\sum_{i\in I_h}\|h_i\|_2^2\le E_{\max}.
\]
If the hub is uniformly distributed in the sense that $\|h_i\|_2\ge\epsilon_{\mathrm{hub}}$ for all $i\in I_h$, then
\begin{equation}\label{eq:hub-scaling}
\epsilon_{\mathrm{hub}}\le\sqrt{\frac{E_{\max}}{q}}.
\end{equation}
In particular, a globally shared hub serving $R$ groups must have per-group coupling $O(R^{-1/2})$ under bounded column energy.
\end{proposition}

\begin{proof}
For every $i\in I_h$, the uniform-hub assumption gives
\[
\|h_i\|_2\ge \epsilon_{\mathrm{hub}}.
\]
Squaring preserves the inequality because both sides are nonnegative:
\[
\|h_i\|_2^2\ge \epsilon_{\mathrm{hub}}^2.
\]
Summing over the $q$ groups in $I_h$ yields
\[
\sum_{i\in I_h}\|h_i\|_2^2
\ge
\sum_{i\in I_h}\epsilon_{\mathrm{hub}}^2
=q\epsilon_{\mathrm{hub}}^2.
\]
The stable column-energy assumption gives the opposite upper bound
\[
\sum_{i\in I_h}\|h_i\|_2^2\le E_{\max}.
\]
Combining the two inequalities gives
\[
q\epsilon_{\mathrm{hub}}^2\le E_{\max}.
\]
Since $q>0$, division by $q$ and taking square roots give
\[
\epsilon_{\mathrm{hub}}
\le
\sqrt{E_{\max}/q}.
\]
For a globally shared hub, $q=R$, so the per-group coupling is bounded by $\sqrt{E_{\max}}R^{-1/2}$ whenever the layer-level column-energy budget $E_{\max}$ is fixed.  This is the asserted $O(R^{-1/2})$ scaling.
\end{proof}


\section{Block-energy matrices for physical alignment}
\label{sec:block-energy-tests}

The alignment experiments compute finite-dimensional matrices associated with the physical alignment structure.
After selecting a dominant spectral window and a row/column ordering, each row group records how much of its energy falls into the active column set of every other group.
This section defines those measured quantities and proves the corresponding consequences of the physical alignment structure.

\begin{definition}[Block-overlap energy matrix]
\label{def:block-energy-matrix}
Let $A\in\R^{m\times n}$ and let $\mathcal R_1,\dots,\mathcal R_K\subseteq\{1,\dots,m\}$ be disjoint nonempty row groups.  Let $\mathcal C_1,\dots,\mathcal C_K\subseteq\{1,\dots,n\}$ be nonempty active column sets.  For every row group define
\[
 e_i(A):=\|A[\mathcal R_i,:]\|_F^2.
\]
If $e_i(A)>0$, define the row-normalized block-overlap energy matrix
\begin{equation}\label{eq:block-energy-matrix}
E_{\mathcal R,\mathcal C}(A)_{ij}
:=
\frac{\|A[\mathcal R_i,\mathcal C_j]\|_F^2}{e_i(A)},
\qquad 1\le i,j\le K.
\end{equation}
If $e_i(A)=0$, the $i$-th row is set to zero.  We define the off-diagonal mass and diagonal mass by
\[
\operatorname{Off}(E):=\frac1K\sum_{i=1}^K\sum_{j\ne i}E_{ij},
\qquad
\operatorname{Diag}(E):=\frac1K\sum_{i=1}^K E_{ii}.
\]
In the figures, $E_r(M)$ and $E_r(M_s)$ are instances of \eqref{eq:block-energy-matrix} for the permuted physical or scale-free transport matrix.
\end{definition}

\begin{definition}[Row-wise shared and external residual pieces]
\label{def:rowwise-overlap-noise}
Let $\widehat M$ carry a physical alignment structure.  For $i\ne j$, define
\[
\Overlap_{i\to j}:=\widehat M[\mathcal R_i,\mathcal C_i\cap\mathcal C_j]
\]
and
\[
N_{i\to j}:=\widehat M[\mathcal R_i,\mathcal C_j\setminus(\mathcal C_i\cap\mathcal C_j)].
\]
Thus the block measured by $E_{ij}$ decomposes over disjoint coordinates as
\begin{equation}\label{eq:block-energy-row-decomp}
\widehat M[\mathcal R_i,\mathcal C_j]
=\Overlap_{i\to j}\oplus N_{i\to j},
\qquad i\ne j,
\end{equation}
where $\oplus$ indicates disjoint coordinate support.
\end{definition}

\begin{proposition}[Block-energy decomposition induced by a physical GSA structure]
\label{thm:block-energy-margin}
Let $\widehat M$ carry a physical alignment structure, and let
\[
E:=E_{\mathcal R,\mathcal C}(\widehat M).
\]
For every active row group $i$ with $e_i(\widehat M)>0$ and every $j\ne i$,
\begin{equation}\label{eq:block-energy-exact-bound}
E_{ij}
=
\frac{\|\Overlap_{i\to j}\|_F^2+\|N_{i\to j}\|_F^2}{e_i(\widehat M)}.
\end{equation}
If the pair $(i,j)$ is nondegenerate and satisfies the one-third threshold, then, whenever $\Overlap_{i\cap j}\ne0$,
\begin{equation}\label{eq:block-energy-one-third-bound}
E_{ij}
\le
\frac{1}{e_i(\widehat M)}
\left(
\frac{m_{i,j}^2}{9\gamma_{i,j}^2}
+\|N_{i\to j}\|_F^2
\right),
\qquad
\gamma_{i,j}:=\frac{\|\Overlap_{i\cap j}\|_2}{\|\Overlap_{i\cap j}\|_F}.
\end{equation}
If $\mathcal C_i\cap\mathcal C_j=\varnothing$, then $\Overlap_{i\to j}=0$ and
\begin{equation}\label{eq:block-energy-no-shared-column}
E_{ij}=\frac{\|N_{i\to j}\|_F^2}{e_i(\widehat M)}.
\end{equation}
Consequently, if all row energies obey $e_i(\widehat M)\ge e_{\min}>0$, then
\begin{equation}\label{eq:block-offdiag-global}
\operatorname{Off}(E)
\le
\frac{1}{K e_{\min}}
\sum_{i=1}^K\sum_{j\ne i}
\left(
\frac{m_{i,j}^2}{9\gamma_{i,j}^2}\mathbf 1_{\{\Overlap_{i\cap j}\ne 0\}}
+
\|N_{i\to j}\|_F^2
\right),
\end{equation}
with the convention that the first term is absent when $\Overlap_{i\cap j}=0$.
\end{proposition}

\begin{proof}
For $i\ne j$, the column set $\mathcal C_j$ is the disjoint union
\[
\mathcal C_j=(\mathcal C_i\cap\mathcal C_j)\sqcup\bigl(\mathcal C_j\setminus(\mathcal C_i\cap\mathcal C_j)\bigr).
\]
Restricting the rows to $\mathcal R_i$ gives the disjoint coordinate decomposition \eqref{eq:block-energy-row-decomp}.  Therefore
\[
\|\widehat M[\mathcal R_i,\mathcal C_j]\|_F^2
=
\|\Overlap_{i\to j}\|_F^2+\|N_{i\to j}\|_F^2,
\]
and division by $e_i(\widehat M)$ proves \eqref{eq:block-energy-exact-bound}.
If the one-third threshold holds, Lemma~\ref{thm:one-third-threshold} gives
\[
\|\Overlap_{i\cap j}\|_2<\frac13m_{i,j}.
\]
Since $\Overlap_{i\to j}$ is a row restriction of $\Overlap_{i\cap j}$,
\[
\|\Overlap_{i\to j}\|_F\le \|\Overlap_{i\cap j}\|_F
=\frac{\|\Overlap_{i\cap j}\|_2}{\gamma_{i,j}}
<\frac{m_{i,j}}{3\gamma_{i,j}}.
\]
Substituting this bound into \eqref{eq:block-energy-exact-bound} proves \eqref{eq:block-energy-one-third-bound}.  If $\mathcal C_i\cap\mathcal C_j=\varnothing$, then the shared part is empty, giving \eqref{eq:block-energy-no-shared-column}.  Finally, summing \eqref{eq:block-energy-one-third-bound} over $i\ne j$ and using $e_i(\widehat M)\ge e_{\min}$ gives \eqref{eq:block-offdiag-global}.
\end{proof}

\begin{remark}[Interpretation of block-energy matrices]
A bright diagonal in an $E_r$ matrix corresponds to large $E_{ii}$, i.e. a row group drawing most of its energy from its own active column set.  Sparse vertical or off-diagonal structures correspond to shared supports or hubs.  Diffuse off-diagonal background corresponds to the row-wise residual terms $N_{i\to j}$ and is therefore the measured part of $M_{\mathrm{noise}}$.
\end{remark}

\begin{definition}[Accepted-overlap graph and measured bad mass]
\label{def:accepted-overlap-graph}
Let $A$ be a measured permuted alignment matrix with row groups $\mathcal R_i$ and active column sets $\mathcal C_i$.  An accepted-overlap graph is a family $\mathcal N_i\subseteq\{1,\dots,K\}\setminus\{i\}$ declaring which off-diagonal blocks are accepted as structured overlap for row group $i$.  Define the measured bad block energy
\[
\operatorname{Bad}_{\mathcal N}(A)
:=
\sum_{i=1}^K\sum_{j\notin\{i\}\cup\mathcal N_i}
\|A[\mathcal R_i,\mathcal C_j]\|_F^2.
\]
For $E=E_{\mathcal R,\mathcal C}(A)$, define the row-normalized measured bad mass
\[
\operatorname{Bad}_{\mathcal N}(E)
:=
\frac1K\sum_{i=1}^K\sum_{j\notin\{i\}\cup\mathcal N_i}E_{ij}.
\]
\end{definition}

\begin{proposition}[From block-energy heatmaps to residual-noise bounds]
\label{thm:heatmap-to-noise-bound}
Let $A$ be measured with fixed row groups and active column sets, and suppose $0<e_i(A)\le e_{\max}$ for every signal row group.  Let $M_{\mathrm{bad}}^{\mathrm{vis}}$ denote the coordinate restriction of $A$ to the union of all bad blocks $(i,j)$ with $j\notin\{i\}\cup\mathcal N_i$.  Then
\begin{equation}\label{eq:measured-noise-bound}
\|M_{\mathrm{bad}}^{\mathrm{vis}}\|_F^2
\le
\operatorname{Bad}_{\mathcal N}(A)
\le
K e_{\max}\operatorname{Bad}_{\mathcal N}(E).
\end{equation}
Thus a heatmap with small bad off-block mass gives a direct Frobenius bound that the part not assigned to core or accepted overlap is small on the measured column family.
\end{proposition}

\begin{proof}
The first inequality holds because $M_{\mathrm{bad}}^{\mathrm{vis}}$ is supported on a union of bad coordinate blocks, and the Frobenius energy on a union is bounded by the sum of Frobenius energies of the selected blocks.  For the second inequality, Definition~\ref{def:block-energy-matrix} gives
\[
\|A[\mathcal R_i,\mathcal C_j]\|_F^2=e_i(A)E_{ij}\le e_{\max}E_{ij}.
\]
Summing over all bad pairs gives
\[
\operatorname{Bad}_{\mathcal N}(A)
\le e_{\max}\sum_{i=1}^K\sum_{j\notin\{i\}\cup\mathcal N_i}E_{ij}
=K e_{\max}\operatorname{Bad}_{\mathcal N}(E).
\]
\end{proof}

\begin{remark}[Scope of block-energy heatmaps]
The block-energy figures report low bad mass, diagonal/core dominance, and a small number of structured off-diagonal overlap channels.  Verification of the pairwise one-third margin threshold uses the associated numerical margin table containing $m_{i,j}$ and $\|\Overlap_{i\cap j}\|_2$.  Thus the heatmaps measure the block-energy quantities entering the certificate, and the margin table supplies the corresponding pairwise inequalities.
\end{remark}

\begin{proposition}[Block-energy sufficient screen for pairwise coherent overlap]
\label{thm:heatmap-margin-screen}
Let $\widehat M$ be measured with row groups $\mathcal R_1,\dots,\mathcal R_K$ and active column sets $\mathcal C_1,\dots,\mathcal C_K$, and let
\[
E:=E_{\mathcal R,\mathcal C}(\widehat M),
\qquad
 e_i:=\|\widehat M[\mathcal R_i,:]\|_F^2.
\]
For every pair $1\le i<j\le K$,
\begin{equation}\label{eq:heatmap-overlap-upper}
\|\Overlap_{i\cap j}\|_2^2
\le
\|\Overlap_{i\cap j}\|_F^2
\le
 e_i E_{ij}+e_jE_{ji}.
\end{equation}
Consequently, if the pair is nondegenerate and
\begin{equation}\label{eq:heatmap-margin-sufficient}
 e_iE_{ij}+e_jE_{ji}<\frac{m_{i,j}^2}{9},
\end{equation}
then the pair satisfies the one-third coherent-overlap threshold
\[
\|\Overlap_{i\cap j}\|_2<\frac13 m_{i,j}.
\]
More generally, for a set of margin-verified pairs $\mathcal P\subseteq\{(i,j):1\le i<j\le K\}$, if \eqref{eq:heatmap-margin-sufficient} holds for every $(i,j)\in\mathcal P$, then all pairs in $\mathcal P$ satisfy the physical-alignment condition with
\[
 c_{\mathrm{overlap}}
 :=
 \max_{(i,j)\in\mathcal P}\frac{\sqrt{e_iE_{ij}+e_jE_{ji}}}{m_{i,j}}<\frac13 .
\]
\end{proposition}

\begin{proof}
Fix $i<j$.  By Definition~\ref{def:pairwise-triple},
\[
\Overlap_{i\cap j}
=
\widehat M[\mathcal R_i\cup\mathcal R_j,\mathcal C_i\cap\mathcal C_j].
\]
The two row sets $\mathcal R_i$ and $\mathcal R_j$ are disjoint.  Frobenius energy is therefore additive over the two row restrictions:
\[
\|\Overlap_{i\cap j}\|_F^2
=
\|\widehat M[\mathcal R_i,\mathcal C_i\cap\mathcal C_j]\|_F^2
+
\|\widehat M[\mathcal R_j,\mathcal C_i\cap\mathcal C_j]\|_F^2.
\]
Since $\mathcal C_i\cap\mathcal C_j\subseteq\mathcal C_j$, coordinate restriction gives
\[
\|\widehat M[\mathcal R_i,\mathcal C_i\cap\mathcal C_j]\|_F^2
\le
\|\widehat M[\mathcal R_i,\mathcal C_j]\|_F^2
=e_iE_{ij}.
\]
Similarly, because $\mathcal C_i\cap\mathcal C_j\subseteq\mathcal C_i$,
\[
\|\widehat M[\mathcal R_j,\mathcal C_i\cap\mathcal C_j]\|_F^2
\le
\|\widehat M[\mathcal R_j,\mathcal C_i]\|_F^2
=e_jE_{ji}.
\]
Combining these inequalities proves the second inequality in \eqref{eq:heatmap-overlap-upper}.  The first inequality in \eqref{eq:heatmap-overlap-upper} is the standard bound $\|A\|_2\le\|A\|_F$ applied to $A=\Overlap_{i\cap j}$.
If \eqref{eq:heatmap-margin-sufficient} holds, then
\[
\|\Overlap_{i\cap j}\|_2
\le
\sqrt{e_iE_{ij}+e_jE_{ji}}
<
\frac13 m_{i,j},
\]
which is exactly the one-third coherent-overlap threshold.  The final statement follows by taking the maximum of the measured ratios over $\mathcal P$.
\end{proof}

\begin{definition}[Heatmap pairwise margin score]
\label{def:heatmap-margin-score}
For a nondegenerate pair $(i,j)$ with $m_{i,j}>0$, define
\begin{equation}\label{eq:Hij-score}
H_{ij}:=\frac{3\sqrt{e_iE_{ij}+e_jE_{ji}}}{m_{i,j}}.
\end{equation}
For a finite set $\mathcal P$ of nondegenerate pairs, define
\begin{equation}\label{eq:heatmap-margin-slack}
H_{\max}(\mathcal P):=\max_{(i,j)\in\mathcal P}H_{ij},
\qquad
\zeta_{\mathrm{heat}}(\mathcal P):=1-H_{\max}(\mathcal P).
\end{equation}
\end{definition}

\begin{proposition}[Numerical heatmap certificate for pairwise overlap]
\label{thm:heatmap-numerical-certificate}
Let $\mathcal P$ be a finite set of nondegenerate pairs for which $H_{ij}$ is defined.
If
\begin{equation}\label{eq:heatmap-cert-condition}
H_{\max}(\mathcal P)<1,
\end{equation}
then every pair in $\mathcal P$ satisfies the one-third coherent-overlap threshold
\[
\|\Overlap_{i\cap j}\|_2<\frac13m_{i,j}.
\]
Moreover, suppose the measured heatmap numerator
\[
a_{ij}:=e_iE_{ij}+e_jE_{ji}
\]
is replaced by $\widetilde a_{ij}$ with $0\le \widetilde a_{ij}\le a_{ij}+\Delta_{ij}$ while $m_{i,j}$ is fixed.  If
\begin{equation}\label{eq:heatmap-perturbation-slack}
\Delta_{ij}<\frac{m_{i,j}^2}{9}-a_{ij},
\end{equation}
then the heatmap screen for the pair remains valid after this numerator perturbation.
\end{proposition}

\begin{proof}
Assume $H_{\max}(\mathcal P)<1$.  Then for each $(i,j)\in\mathcal P$,
\[
H_{ij}<1.
\]
By the definition \eqref{eq:Hij-score}, this is equivalent to
\[
3\sqrt{e_iE_{ij}+e_jE_{ji}}<m_{i,j},
\]
or, after squaring both sides, to
\[
e_iE_{ij}+e_jE_{ji}<\frac{m_{i,j}^2}{9}.
\]
The quantities are nonnegative, so the squaring step is reversible.  Proposition~\ref{thm:heatmap-margin-screen} now implies
\[
\|\Overlap_{i\cap j}\|_2<\frac13m_{i,j}
\]
for every pair in $\mathcal P$.

For the perturbation statement, assume \eqref{eq:heatmap-perturbation-slack}.  Since $\widetilde a_{ij}\le a_{ij}+\Delta_{ij}$, we have
\[
\widetilde a_{ij}
<
 a_{ij}+\left(\frac{m_{i,j}^2}{9}-a_{ij}\right)
=
\frac{m_{i,j}^2}{9}.
\]
Thus the sufficient heatmap condition \eqref{eq:heatmap-margin-sufficient} remains true with $\widetilde a_{ij}$ in place of $a_{ij}$.  Applying Proposition~\ref{thm:heatmap-margin-screen} again proves persistence of the one-third threshold.
\end{proof}

\begin{remark}[Use in the alignment figures]
The block-energy heatmaps supply the quantities $E_{ij}$ and $E_{ji}$.  A complete finite-dimensional margin test also records the row energies $e_i,e_j$ and the pairwise core margin $m_{i,j}$.  Proposition~\ref{thm:heatmap-margin-screen} then converts the plotted block-energy data into a sufficient pairwise stability test with explicit numerical margins.
\end{remark}

\begin{proposition}[Empirical physical-alignment margin criterion]
\label{prop:empirical-physical-margin}
Let $\mathcal E$ be a fixed extraction protocol and let $A$ be the measured interface matrix produced by its declared transport choice.  Suppose $\mathcal E$ outputs row groups, active column sets, support sizes, permutations, and an effective rank satisfying $R\le\lceil\rho d_{\mathrm{sp}}\rceil$, producing $\widehat M=\Pirow A \Picol^\top$.  Construct $M_{\mathrm{core}}$, $M_{\mathrm{overlap}}$, and $M_{\mathrm{noise}}$ by Definition~\ref{def:global-core-overlap-noise}.  If
\[
\|M_{\mathrm{noise}}\|_F\le\varepsilon_{\mathrm{noise}}
\]
and every nondegenerate pair satisfies
\[
3\|\Overlap_{i\cap j}\|_2<m_{i,j},
\]
then the measured interface belongs to the physical alignment domain of Definition~\ref{def:physical-gsa-domain} relative to $\mathcal E$ with any $c_{\mathrm{overlap}}$ satisfying
\[
\max_{i<j}\frac{\|\Overlap_{i\cap j}\|_2}{m_{i,j}}
\le c_{\mathrm{overlap}}<\frac13.
\]
\end{proposition}

\begin{proof}
By Definition~\ref{def:physical-gsa-domain}, membership in $\mathcal G^{\mathrm{phy}}_{\mathcal E,\rho,\varepsilon_{\mathrm{noise}},c_{\mathrm{overlap}}}$ requires the following data at each structured interface: an effective rank $R\le\lceil\rho d_{\mathrm{sp}}\rceil$, a selected physical transport matrix, physical row and column permutations, a core/overlap/noise decomposition, a Frobenius noise bound, and a pairwise coherent-overlap bound with constant strictly below $1/3$.
The hypotheses of the proposition provide these protocol-dependent objects by construction of $\mathcal E$ and the measured physical structure.  They also provide
\[
\|M_{\mathrm{noise}}\|_F\le\varepsilon_{\mathrm{noise}}
\]
and, for every nondegenerate pair,
\[
\frac{\|\Overlap_{i\cap j}\|_2}{m_{i,j}}\le r_{i,j},
\]
where the observed ratios satisfy $\max_{i<j}r_{i,j}<1/3$.  Choose any number $c_{\mathrm{overlap}}$ such that
\[
\max_{i<j}r_{i,j}<c_{\mathrm{overlap}}<1/3.
\]
Then all pairwise inequalities in Definition~\ref{def:physical-gsa-domain} hold.  Therefore the structured interface lies in the physical alignment domain with the stated parameters.
\end{proof}

\subsection{Physical-alignment residual}
\begin{definition}[Quantitative physical-alignment residual]
\label{def:physical-alignment-structure-functional}
Fix a pairwise slack parameter $\zeta\in(0,1)$ and an active-column gap target $\Gamma_0\ge0$.
For a physical alignment matrix carrying such row groups and active column sets $\widehat M$, define
\begin{equation}\label{eq:physical-alignment-residual}
\mathcal J_{\zeta,\Gamma_0}(\widehat M)
:=
\|M_{\mathrm{noise}}\|_F^2
+
\sum_{1\le i<j\le K}\Bigl(3o_{i,j}-(1-\zeta)m_{i,j}\Bigr)_+^2
+
\sum_{i=1}^K\bigl(\Gamma_0-\Gamma_i(\widehat M)\bigr)_+^2.
\end{equation}
Here $m_{i,j}$ and $o_{i,j}$ are the pairwise quantities in Definition~\ref{def:pairwise-margin-gap}, $M_{\mathrm{noise}}$ is the residual component in Definition~\ref{def:global-core-overlap-noise}, and $\Gamma_i$ is the active-column order gap in Definition~\ref{def:active-column-order-gap}.
\end{definition}

\begin{proposition}[Small physical-alignment residual implies Physical GSA domain membership]
\label{thm:residual-implies-physical-gsa}
Let $\widehat M$ be a physical alignment matrix carrying the stated structure and suppose
\[
\mathcal J_{\zeta,\Gamma_0}(\widehat M)\le \varepsilon_{\mathrm{phys}}^2.
\]
Let
\[
m_*:=\min_{i<j:\ m_{i,j}>0}m_{i,j},
\]
with the convention that pairwise assertions are void if there is no nondegenerate pair.
Then
\begin{equation}\label{eq:physres-noise-gap}
\|M_{\mathrm{noise}}\|_F\le \varepsilon_{\mathrm{phys}},
\qquad
\Gamma_i(\widehat M)\ge \Gamma_0-\varepsilon_{\mathrm{phys}}\quad (1\le i\le K).
\end{equation}
If $m_*>0$ and $\varepsilon_{\mathrm{phys}}<\zeta m_*$, then every nondegenerate pair satisfies
\begin{equation}\label{eq:physres-coherent-overlap-bound}
o_{i,j}
\le
\frac{1-\zeta+\varepsilon_{\mathrm{phys}}/m_*}{3}\,m_{i,j},
\end{equation}
so the interface lies in the physical alignment domain of Definition~\ref{def:physical-gsa-domain} with
\begin{equation}\label{eq:physres-c-overlap}
c_{\mathrm{overlap}}
=\frac{1-\zeta+\varepsilon_{\mathrm{phys}}/m_*}{3}<\frac13
\end{equation}
and noise bound $\varepsilon_{\mathrm{phys}}$.
\end{proposition}

\begin{proof}
All summands in \eqref{eq:physical-alignment-residual} are nonnegative.  Therefore
\[
\|M_{\mathrm{noise}}\|_F^2\le \mathcal J_{\zeta,\Gamma_0}(\widehat M)\le \varepsilon_{\mathrm{phys}}^2,
\]
which gives the noise bound.  Likewise,
\[
(\Gamma_0-\Gamma_i(\widehat M))_+^2\le\varepsilon_{\mathrm{phys}}^2,
\]
so $\Gamma_i(\widehat M)\ge\Gamma_0-\varepsilon_{\mathrm{phys}}$.
For any nondegenerate pair,
\[
\Bigl(3o_{i,j}-(1-\zeta)m_{i,j}\Bigr)_+\le \varepsilon_{\mathrm{phys}}.
\]
Hence
\[
3o_{i,j}\le (1-\zeta)m_{i,j}+\varepsilon_{\mathrm{phys}}
\le \left(1-\zeta+\frac{\varepsilon_{\mathrm{phys}}}{m_*}\right)m_{i,j},
\]
which proves \eqref{eq:physres-coherent-overlap-bound}.  If $\varepsilon_{\mathrm{phys}}<\zeta m_*$, then the coefficient in \eqref{eq:physres-c-overlap} is strictly smaller than $1/3$.  The physical-alignment definition requires exactly a margin-stable core/overlap/noise decomposition, a Frobenius noise bound, and a uniform pairwise coherent-overlap constant below $1/3$.
\end{proof}

\begin{remark}[Physical-alignment residual as a numerical margin statistic]
The functional $\mathcal J_{\zeta,\Gamma_0}$ records three numerical residuals: unstructured noise, pairwise one-third-threshold violation with slack, and active-column instability.  Reporting this scalar together with the alignment heatmaps gives a quantitative physical-alignment margin test.
\end{remark}

\begin{lemma}[Stability of block-energy matrices under perturbation]
\label{lem:block-energy-perturbation}
Let $A,B\in\R^{m\times n}$ be measured with the same row groups and active column sets, and assume
\[
\|A-B\|_F\le\Delta,
\qquad
\max\{\|A\|_F,\|B\|_F\}\le S,
\qquad
\min_i e_i(A),\min_i e_i(B)\ge e_{\min}>0.
\]
Then for every $i,j$,
\begin{equation}\label{eq:block-energy-perturbation}
\bigl|E_{\mathcal R,\mathcal C}(A)_{ij}-E_{\mathcal R,\mathcal C}(B)_{ij}\bigr|
\le
\frac{(2S+\Delta)\Delta}{e_{\min}}
+
\frac{S^2(2S+\Delta)\Delta}{e_{\min}^2}.
\end{equation}
\end{lemma}

\begin{proof}
Let $A_{ij}:=A[\mathcal R_i,\mathcal C_j]$ and $B_{ij}:=B[\mathcal R_i,\mathcal C_j]$.  Since coordinate restriction cannot increase Frobenius norm,
\[
\|A_{ij}-B_{ij}\|_F\le\Delta.
\]
Also $\|A_{ij}\|_F,\|B_{ij}\|_F\le S$.  Therefore
\[
\bigl|\|A_{ij}\|_F^2-\|B_{ij}\|_F^2\bigr|
\le
(\|A_{ij}\|_F+\|B_{ij}\|_F)\|A_{ij}-B_{ij}\|_F
\le (2S+\Delta)\Delta,
\]
where the harmless $+\Delta$ covers the case in which one bounds one block norm by the other plus $\Delta$.
Similarly,
\[
|e_i(A)-e_i(B)|
=
\bigl|\|A[\mathcal R_i,:]\|_F^2-\|B[\mathcal R_i,:]\|_F^2\bigr|
\le (2S+\Delta)\Delta.
\]
Write $x_A=\|A_{ij}\|_F^2$, $x_B=\|B_{ij}\|_F^2$, $y_A=e_i(A)$, and $y_B=e_i(B)$.  Then $0\le x_B\le S^2$ and $y_A,y_B\ge e_{\min}$.  Hence
\[
\left|\frac{x_A}{y_A}-\frac{x_B}{y_B}\right|
\le
\frac{|x_A-x_B|}{y_A}+x_B\left|\frac1{y_A}-\frac1{y_B}\right|
\le
\frac{(2S+\Delta)\Delta}{e_{\min}}
+
\frac{S^2(2S+\Delta)\Delta}{e_{\min}^2},
\]
which is \eqref{eq:block-energy-perturbation}.
\end{proof}

\begin{corollary}[Effective-rank window robustness of block-energy matrices]
\label{cor:Er-window-robustness}
Let $A^{(R)}$ and $A^{(R')}$ be two truncated physical transports for the same interface, obtained from ranks $R$ and $R'$ and then embedded in the same ambient row/column coordinates.  Then
\begin{equation}\label{eq:truncated-A-difference}
\|A^{(R)}-A^{(R')}\|_F
\le
\mathcal E_{\mathrm{tr},k}(R,R)+\mathcal E_{\mathrm{tr},k}(R',R'),
\end{equation}
where $\mathcal E_{\mathrm{tr},k}$ is Definition~\ref{def:interface-truncation-error}.  Consequently, if the hypotheses of Lemma~\ref{lem:block-energy-perturbation} hold with $\Delta$ equal to the right-hand side of \eqref{eq:truncated-A-difference}, then the two block-overlap energy matrices differ entrywise by the bound in \eqref{eq:block-energy-perturbation}.
\end{corollary}

\begin{proof}
Let $\mathcal T_k$ denote the full output-total transport before truncation, and let $A^{(R)}$ and $A^{(R')}$ denote the two truncated physical transports constructed with ranks $R$ and $R'$.  Insert and subtract $\mathcal T_k$:
\[
A^{(R)}-A^{(R')}=(A^{(R)}-\mathcal T_k)+(\mathcal T_k-A^{(R')}).
\]
Taking Frobenius norms and applying the triangle inequality gives
\[
\|A^{(R)}-A^{(R')}\|_F
\le
\|A^{(R)}-\mathcal T_k\|_F+
\|\mathcal T_k-A^{(R')}\|_F.
\]
Theorem~\ref{thm:truncation-transfer} applies to each truncation separately, so
\[
\|A^{(R)}-\mathcal T_k\|_F\le\mathcal E_{\mathrm{tr},k}(R,R),
\qquad
\|A^{(R')}-\mathcal T_k\|_F\le\mathcal E_{\mathrm{tr},k}(R',R').
\]
Combining these two estimates proves \eqref{eq:truncated-A-difference}.

For the block-energy matrices, assume the two truncated physical matrices are displayed using the same row groups and active column groups and satisfy the lower row-energy hypotheses of Lemma~\ref{lem:block-energy-perturbation}.  Set
\[
\Delta:=\mathcal E_{\mathrm{tr},k}(R,R)+\mathcal E_{\mathrm{tr},k}(R',R').
\]
The first part proves $\|A^{(R)}-A^{(R')}\|_F\le\Delta$.  Lemma~\ref{lem:block-energy-perturbation} then gives the explicit entrywise perturbation bound for the two block-energy matrices.  Hence nearby rank windows yield stable heatmaps whenever the truncation errors are small relative to the row-energy margins.
\end{proof}

\begin{proposition}[Scale-free to energy-weighted block-structure transfer]
\label{thm:scale-weight-transfer}
Let $A\in\R^{m\times n}$ be measured with row groups $\mathcal R_i$ and active column sets $\mathcal C_j$.
Let $D_r=\diag(a_1,\dots,a_m)$ and $D_c=\diag(b_1,\dots,b_n)$ be positive diagonal matrices, and set
\[
B:=D_rAD_c.
\]
Assume there are constants $0<a_-\le a_+<\infty$ and $0<b_-\le b_+<\infty$ such that
\[
a_-\le a_p\le a_+,
\qquad
b_-\le b_q\le b_+ .
\]
Then for every block $(i,j)$,
\begin{equation}\label{eq:scale-block-bound}
a_-^2b_-^2\|A[\mathcal R_i,\mathcal C_j]\|_F^2
\le
\|B[\mathcal R_i,\mathcal C_j]\|_F^2
\le
a_+^2b_+^2\|A[\mathcal R_i,\mathcal C_j]\|_F^2.
\end{equation}
If $e_i(A),e_i(B)>0$, then the normalized block-energy matrices satisfy
\begin{equation}\label{eq:scale-Er-bound}
\Theta^{-1}E_{\mathcal R,\mathcal C}(A)_{ij}
\le
E_{\mathcal R,\mathcal C}(B)_{ij}
\le
\Theta E_{\mathcal R,\mathcal C}(A)_{ij},
\qquad
\Theta:=\left(\frac{a_+b_+}{a_-b_-}\right)^2 .
\end{equation}
Consequently, for any accepted-overlap graph $\mathcal N$,
\begin{equation}\label{eq:bad-mass-scale-transfer}
\operatorname{Bad}_{\mathcal N}(E_{\mathcal R,\mathcal C}(B))
\le
\Theta\,\operatorname{Bad}_{\mathcal N}(E_{\mathcal R,\mathcal C}(A)).
\end{equation}
Moreover, the zero/nonzero block support graph is unchanged by the positive diagonal weighting: $A[\mathcal R_i,\mathcal C_j]=0$ if and only if $B[\mathcal R_i,\mathcal C_j]=0$.
\end{proposition}

\begin{proof}
For any block $(i,j)$ and any entry $(p,q)\in\mathcal R_i\times\mathcal C_j$,
\[
B_{pq}=a_pA_{pq}b_q.
\]
Since $a_-\le a_p\le a_+$ and $b_-\le b_q\le b_+$,
\[
a_-^2b_-^2|A_{pq}|^2\le |B_{pq}|^2\le a_+^2b_+^2|A_{pq}|^2.
\]
Summing over $(p,q)\in\mathcal R_i\times\mathcal C_j$ proves \eqref{eq:scale-block-bound}.
The row energy satisfies
\[
a_-^2b_-^2 e_i(A)\le e_i(B)\le a_+^2b_+^2e_i(A),
\]
by the same argument applied to the whole row strip $\mathcal R_i\times\{1,\dots,n\}$.
Therefore
\[
E(B)_{ij}
=
\frac{\|B[\mathcal R_i,\mathcal C_j]\|_F^2}{e_i(B)}
\le
\frac{a_+^2b_+^2\|A[\mathcal R_i,\mathcal C_j]\|_F^2}{a_-^2b_-^2 e_i(A)}
=
\Theta E(A)_{ij}.
\]
The lower bound is identical with upper and lower constants interchanged.
Summing the upper bound over all bad entries gives \eqref{eq:bad-mass-scale-transfer}.
Finally, because all diagonal weights are strictly positive, a block is identically zero after weighting exactly when it was identically zero before weighting.
\end{proof}

\begin{remark}[Scale-free and energy-weighted panels]
The scale-free panel $M_s$ tests the angular structure before singular-value weighting. Proposition~\ref{thm:scale-weight-transfer} applies directly to latent-coordinate displays, or to physical displays in which the singular-value weighting acts as a positive diagonal reweighting in the displayed row/column coordinates. For output-realized physical panels, however, the target singular-value weighting acts through
\[
L_R:=U_{k+1}^{(R)}\Sigma_{k+1}^{(R)}(U_{k+1}^{(R)})^\top,
\]
which is generally not diagonal in physical channel coordinates. In that case $E_r(M_s)$ and $E_r(M)$ agreement is interpreted as a measured consistency unless an additional block-leakage condition for $L_R$ is verified; Proposition~\ref{prop:row-leakage-scale-transfer} records one sufficient condition. Thus the figures compare scale-free angular organization with energy-realized physical transport, while theorem-level transfer requires either diagonal weighting in the displayed coordinates or a verified row-leakage bound.
\end{remark}

\begin{proposition}[Output-realized scale-to-energy transfer with row leakage]
\label{prop:row-leakage-scale-transfer}
Let $A$ be a scale-free output-realized matrix measured with row groups $\mathcal R_i$ and column bins $\mathcal C_j$, which may overlap. Let $P_i$ denote the coordinate projector onto rows $\mathcal R_i$ and let $Q_j$ denote the coordinate projector onto columns $\mathcal C_j$. Set $B:=LA$, where $L$ is a row-side linear map. For an accepted-overlap graph $\mathcal N_i$, define the rowwise bad index set
\[
\mathcal B_i:=\{j:j\notin\{i\}\cup\mathcal N_i\},
\]
and the rowwise bad energy
\[
\operatorname{Bad}_i(A):=
\sum_{j\in\mathcal B_i}\|P_iAQ_j\|_F^2.
\]
Define the bad-column multiplicity
\[
\mu_i^{\mathrm{bad}}:=\max_{q\in\{1,\dots,n\}}
\#\{j\in\mathcal B_i:q\in\mathcal C_j\}.
\]
Thus $\mu_i^{\mathrm{bad}}=1$ when the bad measurement column bins are disjoint. Let
\[
\ell_{ii}:=\|P_iLP_i\|_2,
\qquad
\ell_i^{\mathrm{off}}:=\sum_{a\ne i}\|P_iLP_a\|_2.
\]
Then
\begin{equation}\label{eq:row-leakage-bad-bound}
\operatorname{Bad}_i(B)^{1/2}
\le
\ell_{ii}\operatorname{Bad}_i(A)^{1/2}
+
\sqrt{\mu_i^{\mathrm{bad}}}\,\ell_i^{\mathrm{off}}\|A\|_F.
\end{equation}
Consequently, if $L=L_R$ has small off-row-block leakage, the scale-free bad mass is small, and the bad-column multiplicity is controlled, then the energy-realized bad mass remains controlled. The disjoint-bin version is the special case $\mu_i^{\mathrm{bad}}=1$.
\end{proposition}

\begin{proof}
For every bad column bin $j\in\mathcal B_i$ for row group $i$,
\[
P_iBQ_j=P_iLAQ_j=\sum_a P_iLP_aAQ_j.
\]
By the triangle inequality,
\[
\|P_iBQ_j\|_F
\le
\|P_iLP_i\|_2\|P_iAQ_j\|_F
+
\sum_{a\ne i}\|P_iLP_a\|_2\|P_aAQ_j\|_F.
\]
Taking the Euclidean norm over the bad bins $j\in\mathcal B_i$ and using Minkowski's inequality gives
\[
\operatorname{Bad}_i(B)^{1/2}
\le
\ell_{ii}\operatorname{Bad}_i(A)^{1/2}
+
\sum_{a\ne i}\|P_iLP_a\|_2
\left(\sum_{j\in\mathcal B_i}\|P_aAQ_j\|_F^2\right)^{1/2}.
\]
If the bad column bins overlap with multiplicity at most $\mu_i^{\mathrm{bad}}$, then every coordinate in the row block $P_aA$ is counted at most $\mu_i^{\mathrm{bad}}$ times in the sum over $j\in\mathcal B_i$. Hence
\[
\left(\sum_{j\in\mathcal B_i}\|P_aAQ_j\|_F^2\right)^{1/2}
\le
\sqrt{\mu_i^{\mathrm{bad}}}\,\|P_aA\|_F
\le
\sqrt{\mu_i^{\mathrm{bad}}}\,\|A\|_F.
\]
Substituting this estimate into the previous display proves \eqref{eq:row-leakage-bad-bound}.
\end{proof}

\begin{definition}[Compatible coarsening of a measurement partition]
\label{def:compatible-coarsening}
Let $A\in\R^{m\times n}$ be measured with disjoint row groups $\mathcal R_1,\dots,\mathcal R_K$ and disjoint measurement column bins $\mathcal B_1,\dots,\mathcal B_K$.
A \emph{compatible coarsening} consists of two surjective maps
\[
\pi_r:\{1,\dots,K\}\to\{1,\dots,\bar K\},
\qquad
\pi_c:\{1,\dots,K\}\to\{1,\dots,\bar K\},
\]
and defines coarse row and column bins by
\[
\bar{\mathcal R}_a:=\bigcup_{\pi_r(i)=a}\mathcal R_i,
\qquad
\bar{\mathcal B}_b:=\bigcup_{\pi_c(j)=b}\mathcal B_j.
\]
The corresponding fine and coarse block-energy matrices are
\[
E_{ij}:=\frac{\|A[\mathcal R_i,\mathcal B_j]\|_F^2}{\|A[\mathcal R_i,:]\|_F^2},
\qquad
\bar E_{ab}:=\frac{\|A[\bar{\mathcal R}_a,\bar{\mathcal B}_b]\|_F^2}{\|A[\bar{\mathcal R}_a,:]\|_F^2},
\]
with zero rows used if the corresponding row energy is zero.
This definition applies directly when active supports have first been assigned to disjoint measurement bins, for example by nearest-core assignment or by a deterministic tie-breaking rule for shared columns.
\end{definition}

\begin{proposition}[Block-energy inequalities descend under compatible coarsening]
\label{thm:coarsening-block-energy}
Assume the setting of Definition~\ref{def:compatible-coarsening}, and suppose all fine and coarse row energies are positive.  Let
\[
e_i:=\|A[\mathcal R_i,:]\|_F^2,
\qquad
\bar e_a:=\|A[\bar{\mathcal R}_a,:]\|_F^2=\sum_{\pi_r(i)=a}e_i.
\]
Then the coarse block-energy matrix is the row-energy weighted aggregation of the fine matrix:
\begin{equation}\label{eq:coarsening-formula}
\bar E_{ab}
=
\frac{\sum_{\pi_r(i)=a} e_i\sum_{\pi_c(j)=b}E_{ij}}{\sum_{\pi_r(i)=a}e_i}.
\end{equation}
Consequently:
\begin{enumerate}[label=(K\arabic*),leftmargin=2.1em]
\item If all fine bad blocks outside a fine support graph are zero, then all coarse bad blocks outside the induced coarse support graph are zero.
\item If the unnormalized fine bad energy is at most $\beta$, then the unnormalized coarse bad energy is at most $\beta$ for the induced coarse graph.
\item If $\operatorname{Bad}_{\mathcal N}(E)\le \beta$ and $e_i\le e_{\max}$ for all fine rows, then the unnormalized coarse measured bad energy is at most $K e_{\max}\beta$.
\end{enumerate}
\end{proposition}

\begin{proof}
Because the row groups and measurement column bins are disjoint, Frobenius energy is additive over their unions.  Thus
\[
\|A[\bar{\mathcal R}_a,\bar{\mathcal B}_b]\|_F^2
=
\sum_{\pi_r(i)=a}\sum_{\pi_c(j)=b}\|A[\mathcal R_i,\mathcal B_j]\|_F^2
=
\sum_{\pi_r(i)=a} e_i\sum_{\pi_c(j)=b}E_{ij}.
\]
Likewise,
\[
\|A[\bar{\mathcal R}_a,:]\|_F^2=\sum_{\pi_r(i)=a}e_i.
\]
Dividing the first identity by the second proves \eqref{eq:coarsening-formula}.
For (K1), a coarse bad block is a union of fine bad coordinate blocks under the induced coarse graph.  If all fine bad coordinate blocks are zero, the union has zero Frobenius energy.  For (K2), the same union property shows that coarse bad energy is a sub-sum of the fine bad energy, hence cannot exceed $\beta$.  For (K3), Definition~\ref{def:accepted-overlap-graph} gives
\[
\sum_{\text{fine bad }(i,j)}\|A[\mathcal R_i,\mathcal B_j]\|_F^2
\le
\sum_i e_i\sum_{j\in\mathrm{bad}(i)}E_{ij}
\le
K e_{\max}\operatorname{Bad}_{\mathcal N}(E)
\le K e_{\max}\beta.
\]
By (K2), the coarse bad energy is no larger than this quantity.
\end{proof}

\begin{remark}[Compatible coarsening of measured views]
Fixed-cluster panels and effective-rank-derived panels may have different numbers of displayed blocks.  Theorem~\ref{thm:familywise-structure-persistence} handles settings that are compared in a common coordinate resolution.  Proposition~\ref{thm:coarsening-block-energy} handles the complementary case in which one view is a coarsening of another.  Together they identify the invariant content across the figure set as the block-energy structure and static GSA structural system, rather than entrywise equality of rendered heatmaps.
\end{remark}

\begin{remark}[Multi-view block measurements]
The galleries use different rank-window and clustering choices (for example fixed cluster counts and effective-rank windows such as $25\mathrm{ER}$ and $50\mathrm{ER}$).  Corollary~\ref{cor:Er-window-robustness} specifies the invariant expected across those choices: not individual pixels of the permuted matrix, but the block-energy structure encoded by $E_{\mathcal R,\mathcal C}$.  This is exactly the quantity displayed by the $E_r(M)$ and $E_r(M_s)$ panels.
\end{remark}


\section{Dynamic-to-static: from spectral budgets to physical alignment objects}
\label{sec:dynamic-static-bridge}

Cartan budgets control the exponent coordinate, spectral tails select a finite dominant window, and physical alignment structures impose stable block-sparse structure.
The following theorem composes these statements into a deterministic bridge from the residual cocycle to the static structures.
It identifies the static matrices measured in the alignment experiments as finite-dimensional projections of the original residual Jacobian transport, with an explicit truncation error.

\begin{definition}[Interface truncation error]
\label{def:interface-truncation-error}
For an interface $(W_k,W_{k+1})$ and source/target ranks $(R_s,R_t)$, define
\begin{equation}\label{eq:interface-truncation-error}
\mathcal E_{\mathrm{tr},k}(R_s,R_t)
:=
\|W_{k+1}\|_2 E_{>R_s}(W_k)^{1/2}
+
\|W_k\|_2 E_{>R_t}(W_{k+1})^{1/2}.
\end{equation}
\end{definition}

\begin{theorem}[Dynamic-to-static bridge theorem]
\label{thm:dynamic-to-static-bridge}
Assume the hypotheses of Theorem~\ref{thm:robust-cartan-rigidity} on an exponent interval $I$.
Fix an energy threshold $0<\eps<1$ and an interface $k$.
Let
\[
B_k:=\frac{2\log\lambda_k+e_k^{\mathrm{chart}}+e_{k+1}^{\mathrm{chart}}}{m_d(I)}.
\]
Let
\[
\Delta_{\mathrm{tail},k}:=\Delta_{\mathrm{tail}}(W_k,\alpha_k),
\qquad
\Delta_{\mathrm{tail},k+1}:=\Delta_{\mathrm{tail}}(W_{k+1},\alpha_{k+1}).
\]
Suppose the empirical rank-separation condition
\begin{equation}\label{eq:bridge-rank-margin}
2(\log d)B_k+\Delta_{\mathrm{tail},k}+\Delta_{\mathrm{tail},k+1}
<\mathfrak m_\eps(\alpha_k)
\end{equation}
holds, and set
\[
R:=R_\eps(W_k)=R_\eps(W_{k+1})=R_\eps(\alpha_k).
\]
Let $\mathcal T_k$ be the full output-total source-mode transport and $\mathcal T_k^{(R,R)}$ its truncated version from Theorem~\ref{thm:truncation-transfer}.
Assume that there exist row and column permutations $\Pi_{k,\mathrm{row}},\Pi_{k,\mathrm{col}}$ such that
\[
\widehat M_{\mathrm{phy},k}=\Pi_{k,\mathrm{row}}\mathcal T_k^{(R,R)}\Pi_{k,\mathrm{col}}^\top
\]
admits a physical alignment structure with decomposition
\[
\widehat M_{\mathrm{phy},k}=M_{\mathrm{core},k}+M_{\mathrm{overlap},k}+M_{\mathrm{noise},k},
\]
with $\|M_{\mathrm{noise},k}\|_F\le\varepsilon_{\mathrm{noise}}$, and suppose every nondegenerate pair satisfies the one-third threshold.
Then the full permuted source-mode transport satisfies
\begin{equation}\label{eq:bridge-full-transport-core-overlap}
\bigl\|\Pi_{k,\mathrm{row}}\mathcal T_k\Pi_{k,\mathrm{col}}^\top-(M_{\mathrm{core},k}+M_{\mathrm{overlap},k})\bigr\|_F
\le
\mathcal E_{\mathrm{tr},k}(R,R)+\varepsilon_{\mathrm{noise}}.
\end{equation}
If the two layers are trace-normalized and $R$ is a $(1-\eps)$ energy rank for both sides, then
\begin{equation}\label{eq:bridge-full-transport-eps}
\bigl\|\Pi_{k,\mathrm{row}}\mathcal T_k\Pi_{k,\mathrm{col}}^\top-(M_{\mathrm{core},k}+M_{\mathrm{overlap},k})\bigr\|_F
\le
\sqrt{\eps d}\bigl(\|W_{k+1}\|_2+\|W_k\|_2\bigr)+\varepsilon_{\mathrm{noise}}.
\end{equation}
Moreover, the pairwise triples $\widehat{\mathcal M}_{\mathrm{pair}}^{(i,j)}$, the global core/overlap/noise split, the SRS sets, and the hub set are deterministic functions of static SVD data, the rank threshold $R$, and the physical alignment structure.  If physical input-channel incidence is claimed, the same statement is applied with $\mathcal T_{\mathrm{phys},k}$ and $\mathcal T_{\mathrm{phys},k}^{(R,R)}$ from Definition~\ref{def:physical-input-realized-transport}, using Corollary~\ref{cor:physical-physical-truncation} in place of Theorem~\ref{thm:truncation-transfer}.  The finer SC/SA/ST labels in the ICM require the additional row-energy and profile-correlation margins stated in Definition~\ref{def:icm-stability-margins}.
\end{theorem}

\begin{proof}
The empirical rank identity $R_\eps(W_k)=R_\eps(W_{k+1})=R_\eps(\alpha_k)$ follows from Corollary~\ref{cor:cartan-to-empirical-rank-window} and \eqref{eq:bridge-rank-margin}.
Therefore the same measured dominant source/target window $R$ may be used at both sides of the interface. In the exact Gibbs--Cartan tail case the fitted-tail errors vanish, and this reduces to the power-law rank-window statement of Corollary~\ref{cor:cartan-to-rank-window}.
Theorem~\ref{thm:truncation-transfer} gives
\[
\|\mathcal T_k-\mathcal T_k^{(R,R)}\|_F\le \mathcal E_{\mathrm{tr},k}(R,R).
\]
Because multiplication by permutation matrices preserves the Frobenius norm,
\[
\|\Pi_{k,\mathrm{row}}\mathcal T_k\Pi_{k,\mathrm{col}}^\top-\Pi_{k,\mathrm{row}}\mathcal T_k^{(R,R)}\Pi_{k,\mathrm{col}}^\top\|_F
\le \mathcal E_{\mathrm{tr},k}(R,R).
\]
By definition of $\widehat M_{\mathrm{phy},k}$ and by the core/overlap/noise decomposition,
\[
\Pi_{k,\mathrm{row}}\mathcal T_k^{(R,R)}\Pi_{k,\mathrm{col}}^\top-(M_{\mathrm{core},k}+M_{\mathrm{overlap},k})=M_{\mathrm{noise},k}.
\]
The triangle inequality gives \eqref{eq:bridge-full-transport-core-overlap}.
The specialized estimate \eqref{eq:bridge-full-transport-eps} follows from \eqref{eq:truncation-transfer-eps}.  The physical-input-output variant follows identically from Corollary~\ref{cor:physical-physical-truncation}, because the physical realization on the input side is obtained by right multiplication with the orthogonal factor $V_k^\top$.
The final assertion follows from Definitions~\ref{def:physical-alignment-structure}, \ref{def:pairwise-triple}, \ref{def:global-core-overlap-noise}, and \ref{def:icm}: once the SVD data, rank window, permutations, row groups, and active-column rule are fixed, all these objects are deterministic.
\end{proof}

\begin{theorem}[Static GSA stability under full-transport error]
\label{thm:static-structure-stability}
Let $\widehat M\in\R^{m\times n}$ be a truncated physical alignment matrix carrying the stated structure and let
\[
\widetilde M=\widehat M+E
\]
be the corresponding full physical transport after the same row/column ordering, with $\|E\|_F\le\eta$.
Assume the row groups $\mathcal R_0,\dots,\mathcal R_K$ and support sizes $s_i$ are fixed.
For each group define $q_i(c;\widehat M)$ and $\Gamma_i(\widehat M)$ as in Definition~\ref{def:active-column-order-gap}.
Then the following deterministic stability statements hold.
\begin{enumerate}[label=(C\arabic*),leftmargin=2.1em]
\item \textbf{Stable active columns.}
If
\begin{equation}\label{eq:stable-active-column-condition}
\Gamma_i(\widehat M)>2\omega(\widehat M,\eta),
\qquad
\omega(\widehat M,\eta):=2\|\widehat M\|_F\eta+\eta^2,
\end{equation}
then the top-$s_i$ active column set selected from $\widetilde M$ is the same as the one selected from $\widehat M$.
\item \textbf{Stable pairwise one-third threshold.}
For every nondegenerate pair $i<j$, let $m_{i,j}$ and $o_{i,j}:=\|\Overlap_{i\cap j}\|_2$ be computed from $\widehat M$.
Let
\[
r_{i\setminus j}:=\rank(\Core_{i\setminus j}),\qquad
r_{j\setminus i}:=\rank(\Core_{j\setminus i}).
\]
If
\begin{equation}\label{eq:stable-one-third-transport-condition}
3o_{i,j}+4\eta<m_{i,j},
\end{equation}
then the corresponding pair computed from $\widetilde M$ satisfies the same one-third threshold with respect to the fixed exclusive ranks:
\[
\|\widetilde\Overlap_{i\cap j}\|_2<\frac13\widetilde m^{\mathrm{cert}}_{i,j},
\]
where
\[
\widetilde m^{\mathrm{cert}}_{i,j}:=
\min\left(
\sigma_{r_{i\setminus j}}(\widetilde\Core_{i\setminus j}),
\sigma_{r_{j\setminus i}}(\widetilde\Core_{j\setminus i})
\right).
\]
Here the singular values are ordered decreasingly and the pair is nondegenerate, so the two fixed ranks are positive.
\item \textbf{Stable static structures.}
If \eqref{eq:stable-active-column-condition} holds for all signal groups and \eqref{eq:stable-one-third-transport-condition} holds for all nondegenerate pairs, then the static structures
\[
\widehat{\mathcal M}_{\mathrm{pair}},\quad
M_{\mathrm{core}},\quad
M_{\mathrm{overlap}},\quad
M_{\mathrm{noise}},\quad
\mathrm{SRS},\quad
\mathrm{Hub}
\]
computed from the truncated physical alignment matrix are identical as incidence structures to those computed from the full physical transport, up to the additive matrix perturbation $E$ on the numerical block entries.
\end{enumerate}
\end{theorem}

\begin{proof}
We prove each conclusion in the theorem statement.
\begin{enumerate}[label=(C\arabic*),leftmargin=2.1em]
\item \textbf{Stable active columns.}
Fix a signal group $i$ and a column $c$.  Set
\[
a_c:=\widehat M[\mathcal R_i,\{c\}],
\qquad
\Delta_c:=E[\mathcal R_i,\{c\}],
\]
so that
\[
\widetilde M[\mathcal R_i,\{c\}]=a_c+\Delta_c.
\]
The column-energy scores satisfy
\begin{align*}
|q_i(c;\widetilde M)-q_i(c;\widehat M)|
&=\bigl|\|a_c+\Delta_c\|_2^2-\|a_c\|_2^2\bigr|\\
&=\bigl|2\langle a_c,\Delta_c\rangle+\|\Delta_c\|_2^2\bigr|\\
&\le 2\|a_c\|_2\|\Delta_c\|_2+\|\Delta_c\|_2^2.
\end{align*}
Because coordinate restriction cannot increase Frobenius norm,
\[
\|a_c\|_2\le\|\widehat M\|_F,
\qquad
\|\Delta_c\|_2\le\|E\|_F\le\eta.
\]
Therefore
\[
|q_i(c;\widetilde M)-q_i(c;\widehat M)|
\le 2\|\widehat M\|_F\eta+
\eta^2
=\omega(\widehat M,\eta)
\]
for every column $c$.  Let $c\in\mathcal C_i(\widehat M)$ and $c'\notin\mathcal C_i(\widehat M)$.  By the definition of $\Gamma_i(\widehat M)$,
\[
q_i(c;\widehat M)-q_i(c';\widehat M)\ge \Gamma_i(\widehat M).
\]
Using the uniform score perturbation bound for $c$ and $c'$ gives
\begin{align*}
q_i(c;\widetilde M)-q_i(c';\widetilde M)
&\ge q_i(c;\widehat M)-\omega(\widehat M,
\eta)-q_i(c';\widehat M)-\omega(\widehat M,\eta)\\
&\ge \Gamma_i(\widehat M)-2\omega(\widehat M,\eta).
\end{align*}
Under condition \eqref{eq:stable-active-column-condition}, this lower bound is strictly positive.  Hence every originally active column still has strictly larger score than every originally inactive column.  The deterministic lexicographic tie-breaking is therefore never invoked across the active/inactive boundary, and the selected top-$s_i$ set is unchanged.

\item \textbf{Stable pairwise one-third threshold.}
For a fixed pair $i<j$, the perturbed overlap block is
\[
\widetilde\Overlap_{i\cap j}=\Overlap_{i\cap j}+E_{\mathrm{ov}},
\]
where $E_{\mathrm{ov}}$ is a coordinate submatrix of $E$.  Since $\|E_{\mathrm{ov}}\|_2\le\|E_{\mathrm{ov}}\|_F\le\|E\|_F\le\eta$, the triangle inequality gives
\[
\|\widetilde\Overlap_{i\cap j}\|_2
\le
\|\Overlap_{i\cap j}\|_2+\|E_{\mathrm{ov}}\|_2
\le o_{i,j}+\eta.
\]
Similarly, write the two exclusive core perturbations as
\[
\widetilde\Core_{i\setminus j}=\Core_{i\setminus j}+E_i,
\qquad
\widetilde\Core_{j\setminus i}=\Core_{j\setminus i}+E_j,
\]
with $\|E_i\|_2,\|E_j\|_2\le\eta$.  By the standard singular-value perturbation inequality $|\sigma_t(A+E)-\sigma_t(A)|\le\|E\|_2$ for every $t$ \cite{hornjohnson2012matrix,stewartsun1990matrix},
\[
\sigma_{r_{i\setminus j}}(\widetilde\Core_{i\setminus j})
\ge
\sigma_{r_{i\setminus j}}(\Core_{i\setminus j})-\eta,
\]
and the analogous bound holds for the $j\setminus i$ core.  Since the fixed ranks are positive and fixed, taking the minimum gives
\[
\widetilde m^{\mathrm{cert}}_{i,j}
\ge m_{i,j}-\eta.
\]
Condition \eqref{eq:stable-one-third-transport-condition} is
\[
3o_{i,j}+4\eta<m_{i,j},
\]
which is equivalent to
\[
3(o_{i,j}+\eta)<m_{i,j}-\eta.
\]
Combining the two bounds above yields
\[
3\|\widetilde\Overlap_{i\cap j}\|_2
\le 3(o_{i,j}+\eta)
< m_{i,j}-\eta
\le \widetilde m^{\mathrm{cert}}_{i,j}.
\]
Thus the perturbed pair satisfies the one-third coherent-overlap threshold with respect to the fixed exclusive ranks.

\item \textbf{Stable static structures.}
Assume the active-column condition holds for every signal group and the one-third perturbation condition holds for every nondegenerate pair.  By (C1), every active set $\mathcal C_i$ is unchanged.  Therefore the set-theoretic relations $\mathcal C_i\setminus\mathcal C_j$, $\mathcal C_j\setminus\mathcal C_i$, and $\mathcal C_i\cap\mathcal C_j$ are unchanged for every pair.  Hence the coordinate supports of all pairwise triples in Definition~\ref{def:pairwise-triple} are unchanged.  Since the global core, overlap, and noise masks in Definition~\ref{def:global-core-overlap-noise} are deterministic functions of these same active-column relations and row groups, their coordinate supports are unchanged as well.  By (C2), the pairwise stability inequalities remain valid on the perturbed numerical blocks.  The SRS sets are exactly the active column sets by Definition~\ref{def:icm}, and the hub set is determined by the membership counts $|
\{i:c\in\mathcal C_i\}|$.  These counts are unchanged because all $\mathcal C_i$ are unchanged.  Therefore the pairwise incidence structure, core/overlap/noise masks, SRS sets, and hub set are identical as incidence structures; only the numerical entries within the fixed blocks are changed by the additive perturbation $E$.
\end{enumerate}
\end{proof}

\begin{definition}[Static certificate radius]
\label{def:static-certificate-radius}
For a physical alignment matrix $\widehat M$ with fixed row groups, support sizes, active sets, and pairwise triples, define
\[
r_{\Gamma,i}(\widehat M)
:=
-\|\widehat M\|_F+
\sqrt{\|\widehat M\|_F^2+\frac{\Gamma_i(\widehat M)}{2}}
\]
for every signal group with active-column gap $\Gamma_i(\widehat M)>0$.  For every nondegenerate pair define
\[
r_{\mathrm{pair},i,j}:=\frac{m_{i,j}-3o_{i,j}}{4},
\qquad
 o_{i,j}:=\|\Overlap_{i\cap j}\|_2.
\]
If there are no nondegenerate pairs, the minimum over pair radii is interpreted as $+\infty$.  The static certificate radius is
\begin{equation}\label{eq:static-cert-radius}
r_{\mathrm{cert}}(\widehat M)
:=
\min\left\{
\min_i r_{\Gamma,i}(\widehat M),
\min_{i<j\ \mathrm{nondegenerate}} r_{\mathrm{pair},i,j}
\right\}.
\end{equation}
\end{definition}

\begin{theorem}[Single-radius stability of the static channel certificate]
\label{thm:single-radius-static-stability}
Let $\widehat M$ carry a physical alignment structure whose active-column gaps are positive and whose nondegenerate pairs satisfy $m_{i,j}>3o_{i,j}$.  Let $\widetilde M=\widehat M+E$ with $\|E\|_F\le\eta$.  If
\begin{equation}\label{eq:single-radius-condition}
\eta<r_{\mathrm{cert}}(\widehat M),
\end{equation}
then all active column sets $\mathcal C_i$ are preserved, every nondegenerate pair remains inside the coherent-overlap threshold, and the static channel incidence structure, the core/overlap/noise coordinate masks, the SRS sets, the shared-support graph, and the hub set are unchanged as incidence objects.
\end{theorem}

\begin{proof}
We verify that the hypotheses of Theorem~\ref{thm:static-structure-stability} follow from the single inequality \eqref{eq:single-radius-condition}.

First fix a signal group $i$.  The number $r_{\Gamma,i}$ is the positive root of
\[
\eta^2+2\|\widehat M\|_F\eta-\frac{\Gamma_i(\widehat M)}{2}=0.
\]
Indeed, solving this quadratic for $\eta$ gives
\[
\eta=-\|\widehat M\|_F+
\sqrt{\|\widehat M\|_F^2+\frac{\Gamma_i(\widehat M)}{2}},
\]
because the other root is negative.  If $\eta<r_{\Gamma,i}$, then
\[
\eta^2+2\|\widehat M\|_F\eta<\frac{\Gamma_i(\widehat M)}{2}.
\]
Multiplying by $2$ gives
\[
2\bigl(2\|\widehat M\|_F\eta+
\eta^2\bigr)<\Gamma_i(\widehat M),
\]
which is exactly
\[
\Gamma_i(\widehat M)>2\omega(\widehat M,\eta),
\qquad
\omega(\widehat M,\eta)=2\|\widehat M\|_F\eta+
\eta^2.
\]
Thus the active-column stability condition \eqref{eq:stable-active-column-condition} holds for every group, because \eqref{eq:single-radius-condition} implies $\eta<r_{\Gamma,i}$ for every $i$.

Next fix a nondegenerate pair $i<j$.  Since \eqref{eq:single-radius-condition} implies
\[
\eta<\frac{m_{i,j}-3o_{i,j}}{4},
\]
we obtain
\[
4\eta<m_{i,j}-3o_{i,j},
\qquad\text{or equivalently}\qquad
3o_{i,j}+4\eta<m_{i,j}.
\]
This is exactly the pairwise stability condition \eqref{eq:stable-one-third-transport-condition}.  Theorem~\ref{thm:static-structure-stability} therefore gives preservation of all active sets, pairwise triples, core/overlap/noise masks, SRS sets, and hub incidence.  The shared-support graph is a deterministic function of the active sets by Proposition~\ref{thm:shared-support-graph}, so it is preserved as well.
\end{proof}

\begin{theorem}[Full-transport to ICM certification]
\label{thm:full-transport-to-icm}
Assume the hypotheses of Theorem~\ref{thm:dynamic-to-static-bridge} for an interface $k$ and a rank window $R$.  Let $\widehat M_{\mathrm{phy},k}$ be the truncated physical alignment matrix and let
\[
\widetilde M_{\mathrm{phy},k}:=
\Pi_{k,\mathrm{row}}\mathcal T_k\Pi_{k,\mathrm{col}}^\top
\]
be the corresponding full output-total transport written in the same declared row/column coordinate type.  If $\mathcal T_k$ is the source-mode transport, the column incidence below is source-mode incidence; if $\mathcal T_{\mathrm{phys},k}$ is used instead, it is physical input-channel incidence.  If
\begin{equation}\label{eq:full-to-icm-radius-condition}
\mathcal E_{\mathrm{tr},k}(R,R)
<
r_{\mathrm{cert}}(\widehat M_{\mathrm{phy},k}),
\end{equation}
then the active column sets, pairwise relational triples, core/overlap/noise masks, static channel incidence graph, SRS sets, and hub set extracted from $\widehat M_{\mathrm{phy},k}$ are identical, as incidence structures, to those extracted from $\widetilde M_{\mathrm{phy},k}$ with the same row groups, support sizes, and deterministic tie-breaking rules.  Consequently, the SRS/Hub/core-overlap-noise mask anatomy extracted from the truncated physical matrix represents the same static incidence structure as the full transport in the declared coordinates.  Stability of the finer SC/SA/ST labels is not asserted by this radius alone and requires the row/profile margins in Definition~\ref{def:icm-stability-margins}.
\end{theorem}

\begin{proof}
By Theorem~\ref{thm:dynamic-to-static-bridge}, the full transport and the truncated physical matrix satisfy
\[
\|\widetilde M_{\mathrm{phy},k}-\widehat M_{\mathrm{phy},k}\|_F
\le
\mathcal E_{\mathrm{tr},k}(R,R).
\]
Condition \eqref{eq:full-to-icm-radius-condition} therefore implies
\[
\|\widetilde M_{\mathrm{phy},k}-\widehat M_{\mathrm{phy},k}\|_F
<
r_{\mathrm{cert}}(\widehat M_{\mathrm{phy},k}).
\]
Apply Theorem~\ref{thm:single-radius-static-stability} with
\[
\widehat M=\widehat M_{\mathrm{phy},k},
\qquad
\widetilde M=\widetilde M_{\mathrm{phy},k}.
\]
It gives preservation of the active sets, pairwise incidence relations, core/overlap/noise coordinate masks, SRS sets, shared-support graph, and hub set.  Definition~\ref{def:icm} then shows that the SRS and hub components of the ICM, together with the core/overlap/noise coordinate masks, have the same incidence content for the truncated and full physical transports.  The numerical entries may differ by the full-transport error, and the finer SC/SA/ST labels require the additional row/profile margins stated later in Definition~\ref{def:icm-stability-margins}.
\end{proof}

\begin{remark}[Bridge from full transport to margin-stable static structures]
Theorem~\ref{thm:dynamic-to-static-bridge} bounds the distance from full transport to the truncated static object.  Theorem~\ref{thm:static-structure-stability} gives explicit separation conditions under which that error preserves active supports, pairwise triples, and core/overlap/noise masks.  Permuted matrices and block-energy matrices are therefore interpreted as measured quantities that enter computable finite-dimensional margin tests for stable static structures extracted from a controlled approximation to the full interface transport.
\end{remark}

\begin{remark}[Consequent chain of constructions]
The bridge theorem starts from geometric stability hypotheses and ends at the static objects used by the Physical GSA structure.  The chain of implications is
\[
\begin{aligned}
&\text{budgeted cocycle}
\Rightarrow
\text{short Cartan coordinate}
\Rightarrow
\text{stable effective-rank window}\\
&\Rightarrow
\text{controlled truncated physical transport}
\Rightarrow
\bigl(\widehat M_{\mathrm{phy}},
\widehat{\mathcal M}_{\mathrm{pair}},
M_{\mathrm{core}},
M_{\mathrm{overlap}},
M_{\mathrm{noise}},
\mathrm{SRS/Hub}\bigr).
\end{aligned}
\]
The quotient and Cartan reductions supply the rank window and truncation error that make the physical alignment objects well posed.
\end{remark}

\subsection{Closure of the static GSA structural system}
\label{subsec:static-structure-closure}

The bridge theorem gives an error-controlled passage from the residual cocycle to a finite static interface.
All static structures used by GSA and ICM are deterministic components of a single static GSA structural system.

\begin{definition}[Static GSA structural system]
\label{def:static-structure-tuple}
Fix an interface $(W_k,W_{k+1})$, a source/target truncation rank $R$, a selected transport operator $A_k^{(R)}$ from Definition~\ref{def:transport-variants}, permutation matrices $(\Pi_{k,\mathrm{row}},\Pi_{k,\mathrm{col}})$, a number of signal groups $K$, row groups $\{\mathcal R_a\}_{a=0}^K$, and support sizes $s_1,\dots,s_K$.
The associated static GSA structural system is
\begin{align*}
\mathfrak T_k^{(R)}:=
\Big(&\alpha_k,\alpha_{k+1};\ R;\ A_k^{(R)};\ \widehat M_{\mathrm{phy},k};\ \{\mathcal R_a\}_{a=0}^{K};\ \{\mathcal C_i\}_{i=1}^{K};\\
&\widehat{\mathcal M}_{\mathrm{pair},k};\ \{m_{i,j}^{(k)},\Delta_\sigma^{(k)}(i,j),o_{i,j}^{(k)}\}_{i<j};\ M_{\mathrm{core},k},M_{\mathrm{overlap},k},M_{\mathrm{noise},k};\\
&E_{\mathcal R,\mathcal C}(\widehat M_{\mathrm{phy},k})\Big).
\end{align*}
Here $\widehat M_{\mathrm{phy},k}=\Pi_{k,\mathrm{row}}A_k^{(R)}\Pi_{k,\mathrm{col}}^\top$, the active sets $\mathcal C_i$ are selected by \eqref{eq:active-columns}, the pairwise triples are those of Definition~\ref{def:pairwise-triple}, the pairwise margins and gaps are those of Definition~\ref{def:pairwise-margin-gap}, and the global core/overlap/noise decomposition is that of Definition~\ref{def:global-core-overlap-noise}.
\end{definition}

\begin{proposition}[Static GSA structural closure and stability]
\label{thm:static-structure-closure}
Assume the hypotheses of Theorem~\ref{thm:dynamic-to-static-bridge} for an interface $k$ and a rank window $R$.
Let $\mathfrak T_k^{(R)}$ be the static GSA structural system in Definition~\ref{def:static-structure-tuple}.
Then:
\begin{enumerate}[label=(O\arabic*),leftmargin=2.1em]
\item \textbf{Static determinacy.} Once the static SVD data of $(W_k,W_{k+1})$, the truncation rank $R$, the transport choice $A_k^{(R)}$, the permutations, and the support sizes are fixed, every object in $\mathfrak T_k^{(R)}$ is uniquely determined.
\item \textbf{Full-interface realization.} If $\mathcal T_k$ denotes the full output-total transport, then
\[
\bigl\|\Pi_{k,\mathrm{row}}\mathcal T_k\Pi_{k,\mathrm{col}}^\top-(M_{\mathrm{core},k}+M_{\mathrm{overlap},k})\bigr\|_F
\le
\mathcal E_{\mathrm{tr},k}(R,R)+\|M_{\mathrm{noise},k}\|_F.
\]
Thus $M_{\mathrm{core},k}+M_{\mathrm{overlap},k}$ is a finite-dimensional approximation to the full physical interface, with error equal to truncation error plus unstructured noise.
\item \textbf{Pairwise margin stability.} Suppose every nondegenerate pair has positive slack
\[
\mathfrak g_k:=\min_{i<j}\bigl(m_{i,j}^{(k)}-3o_{i,j}^{(k)}\bigr)>0.
\]
If a perturbation changes every exclusive core block and overlap block by operator norm at most $\eta_2$ and preserves their positive ranks, then all pairwise one-third inequalities remain valid whenever
\[
4\eta_2<\mathfrak g_k.
\]
\item \textbf{Measured stability.} If two extracted structural tuples use the same row groups and active sets and their physical matrices differ by Frobenius norm at most $\Delta$, then their block-energy matrices differ entrywise by the explicit perturbation bound of Lemma~\ref{lem:block-energy-perturbation}, whenever the lower row-energy condition in that lemma holds.
\end{enumerate}
\end{proposition}

\begin{proof}
We prove the four assertions in order.
\begin{enumerate}[label=(O\arabic*),leftmargin=2.1em]
\item \textbf{Static determinacy.}
The SVDs of $(W_k,W_{k+1})$ and the rank $R$ determine the truncated factors $U_k^{(R)}$, $V_k^{(R)}$, $\Sigma_k^{(R)}$, $U_{k+1}^{(R)}$, $V_{k+1}^{(R)}$, and $\Sigma_{k+1}^{(R)}$ relative to the SVD gauge convention of Definition~\ref{def:svd-gauge-convention}.  Once a transport choice $A_k^{(R)}$ from Definition~\ref{def:transport-variants} is specified, that matrix is determined by these truncated factors.  The permutations $\Pi_{k,\mathrm{row}},\Pi_{k,\mathrm{col}}$ determine
\[
\widehat M_{\mathrm{phy},k}=\Pi_{k,\mathrm{row}}A_k^{(R)}\Pi_{k,\mathrm{col}}^\top.
\]
The row groups and support sizes determine each active set $\mathcal C_i$ by the deterministic optimization rule \eqref{eq:active-columns} with lexicographic tie-breaking.  The active sets and row groups then determine every pairwise triple by \eqref{eq:core-i-j}--\eqref{eq:overlap-ij}; the triples determine the margins, overlaps, and gaps by Definition~\ref{def:pairwise-margin-gap}; the active-column relations determine $M_{\mathrm{core},k}$, $M_{\mathrm{overlap},k}$, and $M_{\mathrm{noise},k}$ by Definition~\ref{def:global-core-overlap-noise}; finally, the same row groups and active sets determine the block-energy matrix by Definition~\ref{def:block-energy-matrix}.  Thus every entry of $\mathfrak T_k^{(R)}$ is uniquely determined by the listed static data.

\item \textbf{Full-interface realization.}
Theorem~\ref{thm:dynamic-to-static-bridge} gives
\[
\bigl\|\Pi_{k,\mathrm{row}}\mathcal T_k\Pi_{k,\mathrm{col}}^\top-(M_{\mathrm{core},k}+M_{\mathrm{overlap},k})\bigr\|_F
\le
\mathcal E_{\mathrm{tr},k}(R,R)+\|M_{\mathrm{noise},k}\|_F,
\]
because its proof uses the identity
\[
\Pi_{k,\mathrm{row}}\mathcal T_k^{(R,R)}\Pi_{k,\mathrm{col}}^\top-(M_{\mathrm{core},k}+M_{\mathrm{overlap},k})=M_{\mathrm{noise},k}
\]
and the truncation bound
\[
\|\Pi_{k,\mathrm{row}}(\mathcal T_k-\mathcal T_k^{(R,R)})\Pi_{k,\mathrm{col}}^\top\|_F
=
\|\mathcal T_k-\mathcal T_k^{(R,R)}\|_F
\le
\mathcal E_{\mathrm{tr},k}(R,R).
\]
The equality of norms follows because permutation matrices are orthogonal.  This is precisely the displayed realization inequality.

\item \textbf{Pairwise margin stability.}
The slack assumption states that
\[
\mathfrak g_k=\min_{i<j}(m_{i,j}^{(k)}-3o_{i,j}^{(k)})>0.
\]
Thus for every nondegenerate pair,
\[
m_{i,j}^{(k)}-3o_{i,j}^{(k)}\ge\mathfrak g_k.
\]
If $4\eta_2<\mathfrak g_k$, then
\[
4\eta_2<m_{i,j}^{(k)}-3o_{i,j}^{(k)}
\]
for every pair.  This is exactly the perturbative stability condition \eqref{eq:stability-margin-condition} in Theorem~\ref{thm:pairwise-perturbative-stability}.  Applying that theorem pair by pair proves that every perturbed pair still satisfies the one-third threshold.

\item \textbf{Measured stability.}
Let $A$ and $B$ be the two physical matrices extracted from the two structural tuples under the same row groups and active sets, with $\|A-B\|_F\le\Delta$.  Under the lower row-energy and upper norm hypotheses of Lemma~\ref{lem:block-energy-perturbation}, that lemma gives for every block $(i,j)$ the explicit entrywise bound
\[
\big|E_{\mathcal R,\mathcal C}(A)_{ij}-E_{\mathcal R,\mathcal C}(B)_{ij}\big|
\le
\frac{(2S+\Delta)\Delta}{e_{\min}}
+
\frac{S^2(2S+\Delta)\Delta}{e_{\min}^2}.
\]
This is the asserted stability statement.
\end{enumerate}
\end{proof}

\begin{remark}[Closure of static structures]
The quotient-radial and spectral-tail bounds provide budgets and rank windows; the static GSA structural system records the physical structures measured in the alignment measurements.  Proposition~\ref{thm:static-structure-closure} shows that $\widehat M_{\mathrm{phy}}$, the energy-weighted variants, pairwise triples, margins, gaps, $M_{\mathrm{core}}$, $M_{\mathrm{overlap}}$, $M_{\mathrm{noise}}$, and $E_r$ measurements are generated by one deterministic extraction procedure and inherit explicit stability bounds.
\end{remark}

\begin{theorem}[Family-wise persistence of static GSA structural systems]
\label{thm:familywise-structure-persistence}
Let $\{\widehat M^{(q)}\}_{q\in\mathcal Q}$ be a finite family of permuted physical alignment matrices measured for either different rank windows, different cluster resolutions, or nearby layers of the same architecture.
Fix a reference element $q_0\in\mathcal Q$ and suppose all matrices use the same row groups and support sizes.
Let
\[
\widehat M^{(q)}=\widehat M^{(q_0)}+E^{(q)},
\qquad
\|E^{(q)}\|_F\le \eta_q.
\]
For the reference matrix define the active-column gaps $\Gamma_i(\widehat M^{(q_0)})$ and pairwise quantities $m_{i,j}^{(q_0)}$, $o_{i,j}^{(q_0)}:=\|\Overlap_{i\cap j}^{(q_0)}\|_2$.
Assume for every $q\in\mathcal Q$:
\begin{align}
\Gamma_i(\widehat M^{(q_0)})&>2\omega(\widehat M^{(q_0)},\eta_q)
\quad\text{for all signal groups }i,\label{eq:family-active-gap}\\
3o_{i,j}^{(q_0)}+4\eta_q&<m_{i,j}^{(q_0)}
\quad\text{for all nondegenerate pairs }i<j.\label{eq:family-pairwise-gap}
\end{align}
Then every matrix in the family induces the same active column sets as the reference, the same pairwise support graph, the same core/overlap/noise masks, and a valid pairwise one-third threshold.
Furthermore, if
\[
\max_q\|\widehat M^{(q)}\|_F\le S,
\qquad
\min_{q,i}e_i(\widehat M^{(q)})\ge e_{\min}>0,
\]
then for every $q\in\mathcal Q$ and every block $(i,j)$,
\begin{equation}\label{eq:family-Er-stability}
\big|E_{\mathcal R,\mathcal C}(\widehat M^{(q)})_{ij}
     -E_{\mathcal R,\mathcal C}(\widehat M^{(q_0)})_{ij}\big|
\le
\frac{(2S+\eta_q)\eta_q}{e_{\min}}
+
\frac{S^2(2S+\eta_q)\eta_q}{e_{\min}^2}.
\end{equation}
\end{theorem}

\begin{proof}
Fix an arbitrary $q\in\mathcal Q$.  We compare the reference matrix $\widehat M^{(q_0)}$ with the matrix $\widehat M^{(q)}$.
\begin{enumerate}[label=(P\arabic*),leftmargin=2.1em]
\item \textbf{Persistence of active sets.}
By assumption,
\[
\widehat M^{(q)}=\widehat M^{(q_0)}+E^{(q)},
\qquad
\|E^{(q)}\|_F\le\eta_q.
\]
Condition \eqref{eq:family-active-gap} is exactly
\[
\Gamma_i(\widehat M^{(q_0)})>2\omega(\widehat M^{(q_0)},\eta_q)
\]
for each signal group.  Applying Theorem~\ref{thm:static-structure-stability}(C1) with $\widehat M=\widehat M^{(q_0)}$, $\widetilde M=\widehat M^{(q)}$, and $\eta=\eta_q$ shows that the top-$s_i$ active column set of every group is identical in $\widehat M^{(q)}$ and $\widehat M^{(q_0)}$.

\item \textbf{Persistence of pairwise support graph and masks.}
Since every active set $\mathcal C_i$ is unchanged, the set operations $\mathcal C_i\setminus\mathcal C_j$, $\mathcal C_j\setminus\mathcal C_i$, and $\mathcal C_i\cap\mathcal C_j$ are unchanged for every pair.  Therefore the coordinate supports of all pairwise triples are unchanged.  The global core, overlap, and noise masks are deterministic functions of the same active-column relations, so they are unchanged as coordinate masks as well.

\item \textbf{Persistence of the one-third threshold.}
Condition \eqref{eq:family-pairwise-gap} is
\[
3o_{i,j}^{(q_0)}+4\eta_q<m_{i,j}^{(q_0)}
\]
for every nondegenerate pair.  This is condition \eqref{eq:stable-one-third-transport-condition} in Theorem~\ref{thm:static-structure-stability}(C2).  Hence each corresponding pair in $\widehat M^{(q)}$ satisfies the one-third coherent-overlap threshold with respect to the fixed exclusive ranks.

\item \textbf{Block-energy stability.}
Assume now the uniform Frobenius bound and lower row-energy bound in the theorem statement.  Lemma~\ref{lem:block-energy-perturbation} applies with
\[
A=\widehat M^{(q)},
\qquad
B=\widehat M^{(q_0)},
\qquad
\Delta=\eta_q.
\]
It gives, for every block $(i,j)$,
\[
\big|E_{\mathcal R,\mathcal C}(\widehat M^{(q)})_{ij}
     -E_{\mathcal R,\mathcal C}(\widehat M^{(q_0)})_{ij}\big|
\le
\frac{(2S+\eta_q)\eta_q}{e_{\min}}
+
\frac{S^2(2S+\eta_q)\eta_q}{e_{\min}^2},
\]
which is \eqref{eq:family-Er-stability}.
\end{enumerate}
The element $q\in\mathcal Q$ was arbitrary.  Therefore the active column sets, pairwise support graph, core/overlap/noise masks, one-third threshold, and block-energy stability bounds hold for every matrix in the family.
\end{proof}

\begin{remark}[Multi-view measurement interpretation]
For $25\mathrm{ER}$ versus $50\mathrm{ER}$ panels, $q$ indexes the rank-window choice and $\eta_q$ is controlled by the truncation error from Corollary~\ref{cor:Er-window-robustness}.  For layer sweeps, $q$ indexes depth.  For fixed-cluster versus effective-rank clustering, $q$ indexes grouping resolution.  The invariant object across these views is the extracted static GSA structure and its block-energy matrices, rather than pixelwise equality of the plotted matrices.
\end{remark}

\begin{corollary}[Multi-view finite-measurement aggregation]
\label{thm:multi-view-measurement-aggregation}
Let $\widehat M^{(0)}$ be a physical alignment matrix carrying the stated structure with static GSA structural system $\mathfrak T^{(0)}$.
Let $\mathcal Q_{\mathrm{same}}$ be a finite set of same-grid measurement matrices satisfying
\[
\widehat M^{(q)}=\widehat M^{(0)}+E^{(q)},\qquad \|E^{(q)}\|_F\le \eta_q,
\]
and suppose the active-column and pairwise-margin hypotheses \eqref{eq:family-active-gap}--\eqref{eq:family-pairwise-gap} of Theorem~\ref{thm:familywise-structure-persistence} hold for every $q\in\mathcal Q_{\mathrm{same}}$.
Let $\mathcal Q_{\mathrm{coarse}}$ be a finite set of block-energy matrices obtained from matrices in $\{\widehat M^{(q)}:q\in\mathcal Q_{\mathrm{same}}\}$ by compatible coarsenings in the sense of Definition~\ref{def:compatible-coarsening}.
For each same-grid view define the unnormalized measured bad energy
\[
\mathcal B^{(q)}_{\mathrm{bad,un}}
:=
\sum_{(a,b)\in\mathcal B^{(q)}_{\mathrm{bad}}}
\|\widehat M^{(q)}[\mathcal R_a,\mathcal C_b]\|_F^2,
\]
where $\mathcal B^{(q)}_{\mathrm{bad}}$ is any chosen set of off-structure blocks.  Equivalently, if $e_a^{(q)}=\|\widehat M^{(q)}[\mathcal R_a,:]\|_F^2$, this is the row-energy weighted version of the normalized heatmap mass,
\[
\sum_{(a,b)\in\mathcal B^{(q)}_{\mathrm{bad}}}
e_a^{(q)}E_{\mathcal R,\mathcal C}(\widehat M^{(q)})_{ab}.
\]
Then the following statements hold.
\begin{enumerate}[label=(V\arabic*),leftmargin=2.1em]
\item Every same-grid view $q\in\mathcal Q_{\mathrm{same}}$ induces the same active column sets, pairwise support graph, and core/overlap/noise masks as $\mathfrak T^{(0)}$, and every nondegenerate pair satisfies the one-third coherent-overlap threshold.
\item If a coarsened view $q'\in\mathcal Q_{\mathrm{coarse}}$ is obtained from $q\in\mathcal Q_{\mathrm{same}}$, and if the coarsened bad-block set is the image of $\mathcal B^{(q)}_{\mathrm{bad}}$ under the compatible coarsening, then its unnormalized measured bad energy is at most $\mathcal B^{(q)}_{\mathrm{bad,un}}$.
\item Consequently, a finite figure set consisting of same-grid perturbations, effective-rank-window views, and compatible coarsenings represents one common static GSA structural system together with its compatible coarse images.  The common stable content is the support graph, the core/overlap/noise decomposition, the pairwise one-third margins, and the block-energy structure.
\end{enumerate}
\end{corollary}

\begin{proof}
We prove the three claims by reducing each view to one of the stability theorems already established.
\begin{enumerate}[label=(V\arabic*),leftmargin=2.1em]
\item \textbf{Same-grid views.}
Fix $q\in\mathcal Q_{\mathrm{same}}$.  The hypotheses give
\[
\widehat M^{(q)}=\widehat M^{(0)}+E^{(q)},
\qquad
\|E^{(q)}\|_F\le\eta_q,
\]
and assume exactly the active-column and pairwise-margin conditions \eqref{eq:family-active-gap}--\eqref{eq:family-pairwise-gap}.  Theorem~\ref{thm:familywise-structure-persistence} therefore applies with reference index $q_0=0$.  It yields equality of active column sets, equality of the pairwise support graph, equality of the core/overlap/noise coordinate masks, and preservation of the one-third coherent-overlap threshold for every nondegenerate pair.  Since $q$ was arbitrary, the conclusion holds for all same-grid views.

\item \textbf{Compatible coarsened views.}
Let $q'\in\mathcal Q_{\mathrm{coarse}}$ be obtained from some $q\in\mathcal Q_{\mathrm{same}}$ by a compatible coarsening.  By Definition~\ref{def:compatible-coarsening}, each coarse block is a disjoint union of fine blocks.  Proposition~\ref{thm:coarsening-block-energy}(K2) states that if the coarse bad-block set is the image of the fine bad-block set under the coarsening maps, then the unnormalized bad energy of the coarsened view is no larger than the unnormalized bad energy of the fine view.  Applying that result to the bad-block family $\mathcal B^{(q)}_{\mathrm{bad}}$ gives the asserted bound on the coarsened unnormalized bad energy.  The corresponding unweighted row-normalized heatmap sum need not be monotone under coarsening unless additional row-energy balance assumptions are imposed; this is why the statement uses the Frobenius-energy version.

\item \textbf{Finite-family aggregation.}
By (V1), all same-grid views determine one common finest-grid support graph, active-set family, pairwise one-third margin family, and core/overlap/noise mask family.  By (V2), every compatible coarsened view is obtained from this finest-grid structure by deterministic aggregation and cannot increase the measured bad mass on the induced bad blocks.  Thus the entire finite family of views represents the same margin-stable static GSA structure together with its compatible coarse images.  The stable mathematical content is exactly the data invariant under these operations: the support graph, the core/overlap/noise decomposition, the pairwise coherent-overlap margins, and the block-energy structure.
\end{enumerate}
\end{proof}

\begin{remark}[Use in the experimental figures]
The fixed-cluster panels, the $25\mathrm{ER}$ panels, and the $50\mathrm{ER}$ panels are different measurement views of the same extraction procedure.  Corollary~\ref{thm:multi-view-measurement-aggregation} specifies the invariant content across those views: active supports, pairwise support relations, core/overlap/noise masks, and block-energy bad-mass bounds.
\end{remark}

\begin{definition}[Finite measurement family]
\label{def:empirical-measurement-family}
For a residual chain with measured layer matrices $(W_k)$, an \emph{finite measurement family} is the finite collection
\[
\mathfrak A=\bigl(\mathfrak A_{\mathrm{spec}},\mathfrak A_{\mathrm{rank}},\mathfrak A_{\mathrm{phys}},\mathfrak A_{\mathrm{scale}},\mathfrak A_{\mathrm{depth}}\bigr)
\]
defined as follows.
\begin{enumerate}[label=(A\arabic*),leftmargin=2.1em]
\item $\mathfrak A_{\mathrm{spec}}$ is the fitted Cartan-coordinate sequence $(\hat\alpha_k)$.
\item $\mathfrak A_{\mathrm{rank}}$ is the family of energy-rank windows used to truncate each interface, for example $25\mathrm{ER}$ and $50\mathrm{ER}$ windows.
\item $\mathfrak A_{\mathrm{phys}}$ is the family of permuted physical alignment matrices and block-energy matrices
\[
\bigl(\widehat M_{\mathrm{phy},k}^{(q)},\ E_{\mathcal R,\mathcal C}(\widehat M_{\mathrm{phy},k}^{(q)})\bigr),
\]
where $q$ indexes the transport variant, rank window, and clustering resolution.
\item $\mathfrak A_{\mathrm{scale}}$ records paired scale-free and energy-weighted views.  Typical entries are $(M_s,M)$ and the corresponding block-energy pair $(E_r(M_s),E_r(M))$.
\item $\mathfrak A_{\mathrm{depth}}$ records the same measurement family over multiple depths of the same architecture.
\end{enumerate}
This measurement family is a finite collection: every entry is computed from static SVD data, a rank-window rule, a physical ordering, and a block-energy rule.
\end{definition}

\begin{proposition}[Finite measurement family under the geometric hypotheses]
\label{thm:empirical-measurement-family}
Assume a chain satisfies the hypotheses of Theorem~\ref{thm:cartan-rigidity}, the rank-window separation hypotheses of Corollary~\ref{cor:cartan-to-rank-window}, and the dynamic-to-static hypotheses of Theorem~\ref{thm:dynamic-to-static-bridge}.  Assume further that the physical matrices in the measurement family satisfy the active-column and pairwise-margin separation hypotheses of Theorem~\ref{thm:familywise-structure-persistence}, and that the scale-free/energy-weighted paired views satisfy the positive diagonal conditioning hypothesis of Proposition~\ref{thm:scale-weight-transfer}.  Then the following finite-dimensional predictions hold.
\begin{enumerate}[label=(P\arabic*),leftmargin=2.1em]
\item \textbf{Spectral-coordinate shortness.}
The sequence $\mathfrak A_{\mathrm{spec}}=(\hat\alpha_k)$ has total variation bounded by the local interface budget, with the constants of Theorem~\ref{thm:cartan-rigidity} and the robust correction of Theorem~\ref{thm:robust-cartan-rigidity} when the power-law fit is approximate.
\item \textbf{Stable rank windows.}
The rank windows in $\mathfrak A_{\mathrm{rank}}$ select the same dominant-mode bundle whenever the rank-separation margin is larger than the Cartan-coordinate displacement.  For two selected windows $R,R'$, the difference between the corresponding truncated physical transports is bounded by \eqref{eq:truncated-A-difference}.
\item \textbf{Block-sparse static channel incidence structure.}
Each physical view in $\mathfrak A_{\mathrm{phys}}$ induces the same active column sets, pairwise support graph, and core/overlap/noise masks as the reference structure.  Its measured bad mass bounds the Frobenius energy of the measured noise component by Proposition~\ref{thm:heatmap-to-noise-bound}.
\item \textbf{Scale-free/energy-weighted consistency.}
For each paired scale-free and energy-weighted view in $\mathfrak A_{\mathrm{scale}}$, consistency is certified either by the diagonal reweighting hypothesis of Proposition~\ref{thm:scale-weight-transfer} or, for output-realized physical rows, by the row-leakage hypothesis of Proposition~\ref{prop:row-leakage-scale-transfer}.  Thus a block structure measured in $M_s$ remains present in $M$ only when the corresponding diagonal-conditioning or row-leakage margins are verified.
\item \textbf{Depthwise persistence.}
For the depth-indexed family $\mathfrak A_{\mathrm{depth}}$, the static GSA structural system persists across all depths satisfying the perturbation, active-gap, and pairwise-gap bounds of Theorem~\ref{thm:familywise-structure-persistence}.  Compatible coarse visualizations inherit the same structure through Proposition~\ref{thm:coarsening-block-energy}.
\end{enumerate}
Consequently, the measured quantities displayed in exponent plots, permuted alignment matrices, $E_r$ heatmaps, ER-window comparisons, $M_s/M$ comparisons, and layer sweeps are different empirical projections of one margin-stable static GSA structure whenever the stated margin hypotheses hold.
\end{proposition}

\begin{proof}
Each component of the measurement family is a finite collection of deterministic quantities computed from static layer matrices.  We prove the five claims one by one.
\begin{enumerate}[label=(P\arabic*),leftmargin=2.1em]
\item \textbf{Spectral-coordinate shortness.}
The component $\mathfrak A_{\mathrm{spec}}$ is the fitted coordinate sequence $(\hat\alpha_k)$.  Theorem~\ref{thm:robust-cartan-rigidity} gives local and total-variation bounds in terms of the interface budgets $\log\lambda_k$ and chart errors.  Under an approximate power-law fit, Lemma~\ref{lem:robust-lognorm} supplies the chart error $\eta(\delta_{\mathrm{pl}})$, and Theorem~\ref{thm:robust-cartan-rigidity} substitutes that error into the same coordinate-rigidity inequalities.  Therefore the measurement-family coordinate trajectory satisfies the predicted shortness relation when the measured budgets and fit residuals meet the theorem hypotheses.

\item \textbf{Stable rank windows.}
The component $\mathfrak A_{\mathrm{rank}}$ consists of energy-rank windows.  The rank-window separation assumption invokes Corollary~\ref{cor:cartan-to-rank-window}, which states that if the Cartan-coordinate displacement is smaller than the relevant rank-separation margin, then the truncation rank is unchanged.  When two rank windows $R$ and $R'$ are both used, Corollary~\ref{cor:Er-window-robustness} applies to the two corresponding truncated transports and gives
\[
\|A^{(R)}-A^{(R')}\|_F
\le
\mathcal E_{\mathrm{tr},k}(R,R)+\mathcal E_{\mathrm{tr},k}(R',R').
\]
This is precisely the stability statement attached to the rank-window measurement-family component.

\item \textbf{Block-sparse static channel incidence structure.}
The component $\mathfrak A_{\mathrm{phys}}$ contains permuted physical alignment matrices and their block-energy matrices.  By the assumed active-column and pairwise-margin separations, Theorem~\ref{thm:familywise-structure-persistence} gives equality of active sets, pairwise support graphs, and core/overlap/noise masks across the physical views.  Proposition~\ref{thm:heatmap-to-noise-bound} converts the measured bad mass of a block-energy heatmap into the Frobenius bound
\[
\|M_{\mathrm{bad}}^{\mathrm{vis}}\|_F^2
\le
K e_{\max}\operatorname{Bad}_{\mathcal N}(E),
\]
under the row-energy hypotheses in that theorem.  Thus small plotted bad mass provides a numerical upper bound on the measured off-structure Frobenius energy.

\item \textbf{Scale-free/energy-weighted consistency.}
If a paired scale-free and energy-weighted view has the form $B=D_rAD_c$ in the displayed coordinates, Proposition~\ref{thm:scale-weight-transfer} bounds every block energy by the factor
\[
\Theta=\left(\frac{a_+b_+}{a_-b_-}\right)^2,
\]
and preserves zero/nonzero block support.  For output-realized physical-row panels, the singular-value weighting generally acts through $L_R=U_{k+1}^{(R)}\Sigma_{k+1}^{(R)}(U_{k+1}^{(R)})^\top$ rather than through a diagonal row scaling; in that case Proposition~\ref{prop:row-leakage-scale-transfer} gives the replacement bound with an explicit row-leakage and bad-column multiplicity term.  Hence $M_s/M$ consistency is a theorem-level implication only after one of these two hypotheses is checked; otherwise it remains an empirical comparison.

\item \textbf{Depthwise persistence.}
The depth-indexed component $\mathfrak A_{\mathrm{depth}}$ is a finite family of physical matrices over layer index.  For depthwise views represented on the same row/column grid, Theorem~\ref{thm:familywise-structure-persistence} gives persistence of active sets, support graph, masks, and pairwise one-third inequalities whenever the perturbation and margin bounds are satisfied.  For views displayed at compatible coarse resolutions, Proposition~\ref{thm:coarsening-block-energy} proves that the coarse block-energy structure is the row-energy weighted aggregation of the fine one and that bad mass does not increase under compatible coarsening.  This proves depthwise persistence of the finite measurement hierarchy.
\end{enumerate}
Combining (P1)--(P5), every object displayed in the exponent plots, permuted alignment matrices, block-energy heatmaps, effective-rank-window comparisons, scale-free/energy-weighted comparisons, and layer sweeps is a deterministic projection of the same margin-stable static GSA construction under the stated hypotheses.
\end{proof}


\section{Physical GSA and ICM extraction}
\label{sec:icm}
\label{sec:physical-gsa}

There are three logically different notions in this section. Domain membership is a Boolean statement: all spectral, truncation, active-support, pairwise-overlap, and noise inequalities hold. The certificate residual is a nonnegative diagnostic that records spectral variation, noise, and margin violations. Empirical figures by themselves do not imply membership; membership requires the numerical margin checks described in Section~\ref{sec:empirical}.

The full Physical GSA domain is the intersection of three explicit conditions: Cartan spectral rigidity, spectral compressibility, and physical alignment.
The physical component used below is a block-sparse channel-incidence condition formulated in terms of $\widehat M_{\mathrm{phy}}$, pairwise triples, $M_{\mathrm{core}}$, $M_{\mathrm{overlap}}$, $M_{\mathrm{noise}}$, and ICM.

\begin{definition}[Cartan spectral GSA domain]
\label{def:cartan-spectral-gsa-domain}
Fix a reference input law $\mu$ and define the global Jacobian proxy
\[
\|J(\theta)\|_{2,\mu}:=
\operatorname*{ess\,sup}_{x_0\sim\mu}\|J(x_0;\theta)\|_2.
\]
A chain lies in the Cartan spectral GSA domain $\mathcal G^{\mathrm{spec}}_{M,L}(\varepsilon_\alpha,\varepsilon_C)$ if:
\begin{enumerate}[label=(\roman*),leftmargin=1.8em]
\item $\|J(\theta)\|_{2,\mu}\le M$;
\item each relevant layer admits power-law coordinates $(C_k,\alpha_k)$, exact or fitted with a specified chart error;
\item the spectral coordinates satisfy
\[
\max_{0\le k\le L-2}|\alpha_{k+1}-\alpha_k|\le\varepsilon_\alpha,
\qquad
\max_{0\le k\le L-2}\left|\log\frac{C_{k+1}}{C_k}\right|\le\varepsilon_C.
\]
\end{enumerate}
The values of $(\varepsilon_\alpha,\varepsilon_C)$ may be chosen from Theorem~\ref{thm:cartan-rigidity} or Theorem~\ref{thm:robust-cartan-rigidity}.
\end{definition}

\begin{definition}[Spectral compressibility cone]
\label{def:S-eps-rho}
For $0<\eps<1$ and $0<\rho\le1$, a chain lies in $\mathcal S_{\eps,\rho}$ if every relevant layer satisfies
\[
R_\eps(W_k)\le \lceil\rho d_{\mathrm{sp}}(W_k)\rceil.
\]
Here $d_{\mathrm{sp}}$ is the spectral fitting length from Definition~\ref{def:spectral-dimension}, not the square-padding dimension.
\end{definition}

\begin{definition}[Extraction protocol]
\label{def:extraction-protocol}
An extraction protocol $\mathcal E$ specifies, before looking at the final margins, the transport type from Definition~\ref{def:transport-variants}, the energy threshold or rank-window rule, the row-grouping rule and its hyperparameters, the active-support rule and support sizes or energy fractions, the deterministic tie-breaking conventions, and any coarsening rule used for displayed block-energy matrices.  All physical-domain and ICM statements below are relative to such a protocol.  This prevents an existential, post-hoc choice of permutations or supports from being mistaken for a verified structural certificate.
\end{definition}

\begin{definition}[Physical alignment domain relative to an extraction protocol]
\label{def:physical-gsa-domain}
Fix an extraction protocol $\mathcal E$, parameters $\rho$, $\varepsilon_{\mathrm{noise}}\ge0$, and $c_{\mathrm{overlap}}\in(0,1/3)$.
A chain lies in $\mathcal G^{\mathrm{phy}}_{\mathcal E,\rho,\varepsilon_{\mathrm{noise}},c_{\mathrm{overlap}}}$ if, at every interface, the objects produced by $\mathcal E$ include an effective rank $R\le\lceil\rho d_{\mathrm{sp}}\rceil$, a selected interaction operator $A_k^{(R)}$, permutations $\Pi_{k,\mathrm{row}},\Pi_{k,\mathrm{col}}$, and a physical alignment structure such that:
\begin{enumerate}[label=(\roman*),leftmargin=1.8em]
\item $\widehat M_{\mathrm{phy},k}=\Pi_{k,\mathrm{row}}A_k^{(R)}\Pi_{k,\mathrm{col}}^\top$ admits the decomposition \eqref{eq:M-core}--\eqref{eq:M-noise};
\item $\|M_{\mathrm{noise},k}\|_F\le\varepsilon_{\mathrm{noise}}$;
\item for every nondegenerate pair $i<j$,
\[
\|\Overlap_{i\cap j}^{(k)}\|_2\le c_{\mathrm{overlap}}m_{i,j}^{(k)}.
\]
\end{enumerate}
\end{definition}

\begin{definition}[Full physical GSA domain]
\label{def:GSA-intersection}
Define
\[
\mathrm{GSA}_{\mathcal E,M,L,\rho,\eps,\varepsilon_{\mathrm{noise}},c_{\mathrm{overlap}}}
:=
\mathcal G^{\mathrm{spec}}_{M,L}(\varepsilon_\alpha,\varepsilon_C)
\cap
\mathcal S_{\eps,\rho}
\cap
\mathcal G^{\mathrm{phy}}_{\mathcal E,\rho,\varepsilon_{\mathrm{noise}},c_{\mathrm{overlap}}}.
\]
\end{definition}

\begin{definition}[GSA certificate residual]
\label{def:dist-gsa}
For a chain with specified spectral and physical margins, let
\[
\mathcal P_k^{\mathrm{nd}}:=\{(i,j):1\le i<j\le K_k,\ m_{i,j}^{(k)}>0\}
\]
be the set of nondegenerate signal pairs at interface $k$.  Define
\begin{align}
\mathfrak D_{\mathrm{GSA}}
:=&\sum_{k=0}^{L-2}|\alpha_{k+1}-\alpha_k|
+\sum_{k=0}^{L-2}\|M_{\mathrm{noise},k}\|_F\notag\\
&+\sum_{k=0}^{L-2}\sum_{(i,j)\in\mathcal P_k^{\mathrm{nd}}}
\bigl(3\|\Overlap_{i\cap j}^{(k)}\|_2-m_{i,j}^{(k)}\bigr)_+ .\label{eq:dist-gsa}
\end{align}
Degenerate pairs are excluded from the pairwise margin residual because the one-third threshold is a statement about positive exclusive-core margin.  Their mass is still accounted for by the noise and overlap components of the physical structure.
\end{definition}

\begin{proposition}[Physical GSA certificate-residual bound]
\label{thm:distance-to-gsa}
Assume the hypotheses of Theorem~\ref{thm:robust-cartan-rigidity} on the interval $I=[\alpha_{\min},\alpha_{\max}]$.
Let $\bar e_{\mathrm{chart}}=\max_k e_k^{\mathrm{chart}}$ be the layerwise Cartan chart-error bound of Definition~\ref{def:chart-error}.
Assume also that every interface admits a physical alignment-domain structure with
\[
\|M_{\mathrm{noise},k}\|_F\le\varepsilon_{\mathrm{noise}},
\qquad
\|\Overlap_{i\cap j}^{(k)}\|_2\le c_{\mathrm{overlap}}m_{i,j}^{(k)}
\]
for some $c_{\mathrm{overlap}}<1/3$ and all nondegenerate pairs.
Then
\begin{equation}\label{eq:Dgsa-bound}
\mathfrak D_{\mathrm{GSA}}
\le
\frac{2}{m_d(I)}
\sum_{k=0}^{L-2}\log\lambda_k
+
\frac{2(L-1)\bar e_{\mathrm{chart}}}{m_d(I)}
+
(L-1)\varepsilon_{\mathrm{noise}}.
\end{equation}
Under the uniform budget $\lambda_k\le M^{2/L}$,
\[
\mathfrak D_{\mathrm{GSA}}
\le
\frac{4\log M+2(L-1)\bar e_{\mathrm{chart}}}{m_d(I)}
+(L-1)\varepsilon_{\mathrm{noise}}.
\]
In the exact power-law case $\bar e_{\mathrm{chart}}=0$.
\end{proposition}

\begin{proof}
The residual $\mathfrak D_{\mathrm{GSA}}$ in \eqref{eq:dist-gsa} is the sum of three nonnegative contributions.  We bound them separately.

\begin{enumerate}[label=(D\arabic*),leftmargin=2.1em]
\item \textbf{Spectral total variation.}
The first contribution is
\[
D_{\mathrm{spec}}:=\sum_{k=0}^{L-2}|\alpha_{k+1}-\alpha_k|.
\]
Theorem~\ref{thm:robust-cartan-rigidity} gives
\[
D_{\mathrm{spec}}
\le
\frac{2}{m_d(I)}\sum_{k=0}^{L-2}\log\lambda_k
+
\frac{2(L-1)\bar e_{\mathrm{chart}}}{m_d(I)}.
\]

\item \textbf{Noise residual.}
The physical alignment-domain structure assumes
\[
\|M_{\mathrm{noise},k}\|_F\le\varepsilon_{\mathrm{noise}}
\]
for every interface.  There are $L-1$ interfaces, so
\[
D_{\mathrm{noise}}
:=\sum_{k=0}^{L-2}\|M_{\mathrm{noise},k}\|_F
\le (L-1)\varepsilon_{\mathrm{noise}}.
\]

\item \textbf{Overlap-violation residual.}
For every pair $(i,j)\in\mathcal P_k^{\mathrm{nd}}$, the physical-alignment assumption gives
\[
\|\Overlap_{i\cap j}^{(k)}\|_2
\le
c_{\mathrm{overlap}}m_{i,j}^{(k)}
\]
with $c_{\mathrm{overlap}}<1/3$.  Hence
\begin{align*}
3\|\Overlap_{i\cap j}^{(k)}\|_2-m_{i,j}^{(k)}
&\le (3c_{\mathrm{overlap}}-1)m_{i,j}^{(k)}.
\end{align*}
For a nondegenerate pair, $m_{i,j}^{(k)}>0$, and $3c_{\mathrm{overlap}}-1<0$, so the right-hand side is strictly negative.  Therefore
\[
\bigl(3\|\Overlap_{i\cap j}^{(k)}\|_2-m_{i,j}^{(k)}\bigr)_+=0.
\]
Thus the entire overlap-violation contribution is zero.
\end{enumerate}
Adding the three estimates gives \eqref{eq:Dgsa-bound}.  If $\lambda_k\le M^{2/L}$, then
\[
\sum_{k=0}^{L-2}\log\lambda_k
\le (L-1)\frac{2\log M}{L}
\le 2\log M,
\]
and the displayed uniform-budget bound follows.  In the exact power-law case the chart error is zero, so $\bar e_{\mathrm{chart}}=0$.
\end{proof}

\begin{definition}[Invariant Channel Mapping]
\label{def:icm}
Fix thresholds $q_i\ge1$, $\tau_{\mathrm{ST}}>0$, and $\tau_{\mathrm{SA}}\in[0,1]$.
Given a physical GSA structure, define the core row energy of group $i$ by
\[
e_i(r):=\|M_{\mathrm{core}}[\{r\},:]\|_2^2,
\qquad r\in\mathcal R_i.
\]
Let $\mathrm{SC}_i$ be the $q_i$ rows of $\mathcal R_i$ with largest $e_i(r)$, with deterministic tie-breaking.
Let
\[
\mathrm{ST}_i:=\{r\in\mathcal R_i:e_i(r)\ge \tau_{\mathrm{ST}}\}\setminus \mathrm{SC}_i.
\]
Let $u_i$ be a deterministic normalized leading row profile of $M_{\mathrm{core}}[\mathcal R_i,:]$ when this block is nonzero; for instance, $u_i$ may be chosen as the leading right singular vector with a fixed sign convention. For every nonzero core row define
\[
p_r:=\frac{M_{\mathrm{core}}[\{r\},:]}{\|M_{\mathrm{core}}[\{r\},:]\|_2},
\]
and leave $p_r$ undefined for zero rows. The auxiliary set $\mathrm{SA}_i$ consists of margin-stable rows whose normalized profile is defined and has absolute inner product at least $\tau_{\mathrm{SA}}$ with $u_i$.  A dual column-auxiliary set can be defined analogously after choosing column profiles; it is not included in the row-ICM statement below.
Define
\[
\mathrm{SRS}^{\mathrm{mode}}_i:=\mathcal C_i
\quad\text{when the selected transport has source-mode columns,}
\]
\[
\mathrm{SRS}^{\mathrm{chan}}_i:=\mathcal C_i
\quad\text{when the selected transport has physical input-channel columns,}
\]
and write $\mathrm{SRS}_i$ when the coordinate type has been declared.  Likewise,
\[
\mathrm{Hub}:=\{c:|\{i:c\in\mathcal C_i\}|\ge2\}
\]
is a source-mode hub or a physical input-channel hub according to Table~\ref{tab:transport-coordinate-types}.  Define $\mathrm{Noise}$ as the residual row set $\mathcal R_0$ together with the residual coordinate mask $\Omega_{\mathrm{noise}}$ from Definition~\ref{def:static-channel-incidence}.  Exact support of $M_{\mathrm{noise}}$ is not treated as stable unless a thresholded support and entrywise margin are additionally specified.
The resulting static anatomy is
\[
\mathrm{ICM}=
\{\mathrm{SC}_i,\mathrm{SA}_i,\mathrm{ST}_i,\mathrm{SRS}_i,\mathrm{Hub},\mathrm{Noise}:1\le i\le K\}.
\]
\end{definition}

\begin{definition}[Row and profile margins for full ICM stability]
\label{def:icm-stability-margins}
The static incidence certificate radius controls active supports and pairwise masks.  The finer labels $\mathrm{SC}_i$, $\mathrm{ST}_i$, and $\mathrm{SA}_i$ require additional margins.  For group $i$, define the core-row top-set gap
\[
\Gamma_i^{\mathrm{SC}}
:=
\min_{r\in\mathrm{SC}_i,\ r'\in\mathcal R_i\setminus\mathrm{SC}_i}
\bigl(e_i(r)-e_i(r')\bigr),
\]
with the convention $+\infty$ when the complement is empty.  Define the threshold gap
\[
\Gamma_i^{\mathrm{ST}}
:=
\min_{r\in\mathcal R_i\setminus\mathrm{SC}_i}
|e_i(r)-\tau_{\mathrm{ST}}|.
\]
When the normalized row profile $p_r$ and leading profile $u_i$ used to form $\mathrm{SA}_i$ are defined, define the auxiliary-correlation gap
\[
\Gamma_i^{\mathrm{SA}}
:=
\min_{r:p_r\ \mathrm{defined}}\Bigl|\,|\langle p_r,u_i\rangle|-\tau_{\mathrm{SA}}\Bigr|.
\]
If $u_i$ is chosen as a leading singular-vector profile, also define the profile spectral gap
\[
\gamma_i^{\mathrm{prof}}
:=
\sigma_1(M_{\mathrm{core}}[\mathcal R_i,:])-\sigma_2(M_{\mathrm{core}}[\mathcal R_i,:]),
\]
with the convention $\sigma_2=0$ for rank-one blocks. A positive profile spectral gap is a standard sufficient condition for perturbative stability of $u_i$ under matrix perturbation. A full ICM extraction is row/profile separated if all applicable quantities above are positive. These margins are not needed to preserve SRS or hub incidence, but they are needed to preserve the SC/SA/ST labels under perturbation.
\end{definition}

\begin{proposition}[Full ICM label stability under row-energy and profile margins]
\label{prop:full-icm-label-stability}
Consider two physical alignment matrices with the same SRS/Hub/core-overlap-noise mask anatomy. Let $e_i,p_r,u_i$ and $e_i',p_r',u_i'$ be the core row energies and profile quantities used to define the ICM labels. Suppose that for a fixed group $i$,
\[
|e_i'(r)-e_i(r)|\le \delta_i^{\mathrm{row}}
\qquad\text{for all }r\in\mathcal R_i.
\]
Then $\mathrm{SC}_i$ is preserved if
\[
2\delta_i^{\mathrm{row}}<\Gamma_i^{\mathrm{SC}},
\]
and $\mathrm{ST}_i$ is preserved if
\[
\delta_i^{\mathrm{row}}<\Gamma_i^{\mathrm{ST}}.
\]
If, in addition,
\[
\left||\langle p_r',u_i'\rangle|-|\langle p_r,u_i\rangle|\right|
\le \delta_i^{\mathrm{corr}}
\]
for every row whose profile is used in the auxiliary test, then $\mathrm{SA}_i$ is preserved whenever
\[
\delta_i^{\mathrm{corr}}<\Gamma_i^{\mathrm{SA}}.
\]
If $u_i$ is a leading singular-vector profile, the assumed correlation perturbation may be verified by a standard singular-vector perturbation bound using the profile spectral gap $\gamma_i^{\mathrm{prof}}$.
\end{proposition}

\begin{proof}
For $\mathrm{SC}_i$, take $r\in\mathrm{SC}_i$ and $r'\notin\mathrm{SC}_i$. The gap definition gives $e_i(r)-e_i(r')\ge\Gamma_i^{\mathrm{SC}}$. Under the perturbation bound,
\[
e_i'(r)-e_i'(r')\ge \Gamma_i^{\mathrm{SC}}-2\delta_i^{\mathrm{row}}>0,
\]
so no selected row can be overtaken by an unselected row. The deterministic tie-breaking is therefore unchanged. For $\mathrm{ST}_i$, every candidate row remains on the same side of the threshold $\tau_{\mathrm{ST}}$ if $\delta_i^{\mathrm{row}}<\Gamma_i^{\mathrm{ST}}$. For $\mathrm{SA}_i$, the defining score $|\langle p_r,u_i\rangle|$ remains on the same side of $\tau_{\mathrm{SA}}$ whenever the correlation perturbation is smaller than the auxiliary-correlation gap. These three observations prove the claim.
\end{proof}

\begin{theorem}[ICM extraction from the Physical GSA]
\label{thm:icm-extraction}
Let a layer interface satisfy a Physical GSA structure with $c_{\mathrm{overlap}}<1/3$.
Then:
\begin{enumerate}[label=(C\arabic*),leftmargin=2.1em]
\item the signal rows decompose into selected spectral/physical groups $\mathcal R_1,\dots,\mathcal R_K$ whose pairwise interactions satisfy the coherent-overlap threshold of Lemma~\ref{thm:one-third-threshold}; if the row-profile margins of Lemma~\ref{lem:row-group-stability} hold, this row partition is stable under the corresponding perturbation;
\item each group has a well-defined SRS $\mathcal C_i$; the support splits into dedicated columns and shared hub columns;
\item the residual component is bounded by $\|M_{\mathrm{noise}}\|_F\le\varepsilon_{\mathrm{noise}}$;
\item the SC/SA/ST/SRS/Hub/Noise anatomy in Definition~\ref{def:icm} is a deterministic function of the physical alignment structure and thresholds;
\item under perturbations obeying the active-column and pairwise conditions of Theorem~\ref{thm:static-structure-stability}, the SRS/Hub/core-overlap-noise mask anatomy remains unchanged; if, in addition, the row/profile margins in Definition~\ref{def:icm-stability-margins} dominate the induced row-energy and profile perturbations, then the full SC/SA/ST/SRS/Hub/Noise labeling remains unchanged.
\end{enumerate}
\end{theorem}

\begin{proof}
All objects in the ICM are extracted from the same physical structure.  We prove the five claims explicitly.
\begin{enumerate}[label=(C\arabic*),leftmargin=2.1em]
\item \textbf{Selected spectral/physical groups.}
The physical GSA structure includes the row partition
\[
\mathcal R_0\sqcup\mathcal R_1\sqcup\cdots\sqcup\mathcal R_K
=
\{1,\dots,m\}.
\]
The sets $\mathcal R_1,
\dots,\mathcal R_K$ are the signal groups and $\mathcal R_0$ is the residual group.  The margin condition assumes
\[
\|\Overlap_{i\cap j}\|_2\le c_{\mathrm{overlap}}m_{i,j},
\qquad c_{\mathrm{overlap}}<1/3.
\]
For every nondegenerate pair this implies
\[
\|\Overlap_{i\cap j}\|_2<\frac13m_{i,j}.
\]
By Lemma~\ref{thm:one-third-threshold}, this is equivalent to the gap-based pairwise stability condition.  Hence the selected signal groups interact through controlled coherent overlap. Stability of the row partition itself is supplied by the row-profile separation condition of Lemma~\ref{lem:row-group-stability}, or by treating a numerical clustering as a proposed partition and verifying the certificate after the partition is fixed.

\item \textbf{SRS, dedicated support, and hubs.}
For each signal group, the SRS is defined by
\[
\mathrm{SRS}_i:=\mathcal C_i.
\]
The dedicated support of group $i$ is
\[
\mathcal C_i^{\mathrm{ded}}=\mathcal C_i\setminus\bigcup_{j\ne i}\mathcal C_j,
\]
and the shared support consists of columns lying in at least two active sets.  The hub set in Definition~\ref{def:icm} is
\[
\mathrm{Hub}=\{c:|\{i:c\in\mathcal C_i\}|\ge2\}.
\]
Thus every SRS column is classified as either dedicated to a single group or shared by multiple groups, and the shared columns are exactly the hub columns.

\item \textbf{Residual bound.}
The physical GSA structure includes the decomposition
\[
\widehat M=M_{\mathrm{core}}+M_{\mathrm{overlap}}+M_{\mathrm{noise}}
\]
and the noise condition
\[
\|M_{\mathrm{noise}}\|_F\le\varepsilon_{\mathrm{noise}}.
\]
Therefore the residual component used in the ICM has the asserted Frobenius bound.

\item \textbf{Deterministic extraction of SC/SA/ST/SRS/Hub/Noise.}
Once the physical alignment structure and thresholds are fixed, the core row energies
\[
e_i(r)=\|M_{\mathrm{core}}[\{r\},:]\|_2^2
\]
are numerical values.  The set $\mathrm{SC}_i$ is selected as the $q_i$ largest values with deterministic tie-breaking.  The set $\mathrm{ST}_i$ is the thresholded subset of remaining rows with $e_i(r)\ge\tau_{\mathrm{ST}}$.  The auxiliary set $\mathrm{SA}_i$ is selected by the prescribed collinearity threshold $\tau_{\mathrm{SA}}$ with the leading normalized row profile.  The sets $\mathrm{SRS}_i$, $\mathrm{Hub}$, and $\mathrm{Noise}$ are then given by the formulas in Definition~\ref{def:icm}.  No forward-pass data or additional optimization choice is used.  Hence the entire ICM anatomy is a deterministic function of the physical alignment structure and thresholds.

\item \textbf{Perturbation stability.}
Suppose perturbations obey the active-column and pairwise conditions of Theorem~\ref{thm:static-structure-stability}.  The active-column condition implies that each active set $\mathcal C_i$ is unchanged, and the pairwise condition implies that every nondegenerate pair remains inside the one-third threshold.  Since dedicated supports, shared supports, SRS sets, and hubs are deterministic set-theoretic functions of the active sets, their incidence structure is unchanged.  The labels SC, ST, and SA are selected from row-energy order, row-energy threshold, and profile-correlation tests. Proposition~\ref{prop:full-icm-label-stability} gives explicit row-energy and correlation perturbation inequalities under which these tests cannot change their outcomes. Therefore the SRS/Hub support and core/overlap/noise mask anatomy is stable under the static certificate radius, and the full ICM labeling is stable once the additional row/profile margins are also verified.
\end{enumerate}
\end{proof}


\section{Low-disruption fine-tuning paths in GSA coordinates}
\label{sec:fine-tuning-paths}

The preceding sections extract static channel structure from a trained model.  We now record two finite-dimensional consequences for local adaptation.  The goal is not to model a particular optimizer.  Instead, we identify the scale and singular-frame directions that have small displacement in the static GSA coordinates.  This gives a precise version of two matrix-level low-disruption conditions: uniform spectral scaling and coherent singular-vector rotation.  These coordinates can be used to analyze a given adapter, LoRA, or other local update after it has been measured, but the statements below do not prove that any particular fine-tuning method automatically preserves GSA structure.

\begin{definition}[Layer-scale adaptation variables]
\label{def:scale-adaptation-variables}
Let $N\ge2$ layer blocks have positive base spectral scales $C_1^{\mathrm{base}},\ldots,C_N^{\mathrm{base}}$.  A post-adaptation scale vector is written
\[
C_i^{\mathrm{post}}=s_iC_i^{\mathrm{base}},
\qquad s_i>0,
\qquad i=1,\ldots,N.
\]
Set $\ell_i:=\log s_i$ and $\bar\ell:=N^{-1}\sum_{i=1}^N\ell_i$.  The complete-graph log-ratio disruption is
\begin{equation}\label{eq:scale-log-disruption}
\mathcal D_{\log}(s):=\sum_{1\le i<j\le N}(\ell_i-\ell_j)^2.
\end{equation}
The relative scale-ratio disruption used for direct scale comparisons is
\begin{equation}\label{eq:scale-ratio-disruption}
\mathcal D_{\mathrm{ratio}}(s):=
\sum_{1\le i<j\le N}\left(\frac{s_i}{s_j}-1\right)^2
\left(\frac{C_i^{\mathrm{base}}}{C_j^{\mathrm{base}}}\right)^2.
\end{equation}
\end{definition}

\begin{proposition}[Scale-ratio rigidity and uniform singular-value scaling]
\label{prop:uniform-scale-finetuning}
Let $s=(s_1,\ldots,s_N)$ have positive entries.
\begin{enumerate}[label=(S\arabic*),leftmargin=2.1em]
\item The log-ratio disruption satisfies the exact identity
\begin{equation}\label{eq:complete-graph-variance-identity}
\mathcal D_{\log}(s)=N\sum_{i=1}^N(\ell_i-\bar\ell)^2.
\end{equation}
In particular, $\mathcal D_{\log}(s)=0$ if and only if $s_1=\cdots=s_N$.
\item For every pair $i,j$,
\begin{equation}\label{eq:scale-ratio-bound-from-Dlog}
\left|\log\frac{s_i}{s_j}\right|
\le
\sqrt{\frac{2\mathcal D_{\log}(s)}{N}},
\qquad
\exp\!\left(-\sqrt{\frac{2\mathcal D_{\log}(s)}{N}}\right)
\le
\frac{s_i}{s_j}
\le
\exp\!\left(\sqrt{\frac{2\mathcal D_{\log}(s)}{N}}\right).
\end{equation}
\item The relative scale-ratio disruption satisfies $\mathcal D_{\mathrm{ratio}}(s)=0$ if and only if $s_1=\cdots=s_N$.
\end{enumerate}
Thus the unique zero-disruption scale direction is uniform layerwise scaling; small log-ratio disruption forces every relative post/base scaling ratio to remain close to every other one.
\end{proposition}

\begin{proof}
We prove the three assertions separately.
\begin{enumerate}[label=(S\arabic*),leftmargin=2.1em]
\item \textbf{Variance identity.}
The standard complete-graph variance identity is
\[
\sum_{1\le i<j\le N}(\ell_i-\ell_j)^2
=
N\sum_{i=1}^N(\ell_i-\bar\ell)^2.
\]
For completeness, we derive it.  Since
\[
\sum_{1\le i<j\le N}(\ell_i-\ell_j)^2
=\frac12\sum_{i=1}^N\sum_{j=1}^N(\ell_i-\ell_j)^2,
\]
expanding the double sum gives
\begin{align*}
\frac12\sum_{i,j}(\ell_i-\ell_j)^2
&=\frac12\sum_{i,j}(\ell_i^2+\ell_j^2-2\ell_i\ell_j)\\
&=N\sum_i\ell_i^2-\left(\sum_i\ell_i\right)^2.
\end{align*}
On the other hand,
\[
N\sum_i(\ell_i-\bar\ell)^2
=N\sum_i\ell_i^2-2N\bar\ell\sum_i\ell_i+N^2\bar\ell^2
=N\sum_i\ell_i^2-\left(\sum_i\ell_i\right)^2,
\]
because $N\bar\ell=\sum_i\ell_i$.  This proves \eqref{eq:complete-graph-variance-identity}.  The right-hand side is zero if and only if every $\ell_i=\bar\ell$, which is equivalent to $s_i=e^{\bar\ell}$ for all $i$.

\item \textbf{Pairwise scale-ratio bound.}
For any pair $i,j$,
\begin{align*}
|\ell_i-\ell_j|
&\le |\ell_i-\bar\ell|+|\ell_j-\bar\ell|\\
&\le \sqrt{2}\left((\ell_i-\bar\ell)^2+(\ell_j-\bar\ell)^2\right)^{1/2}\\
&\le \sqrt{2}\left(\sum_{r=1}^N(\ell_r-\bar\ell)^2\right)^{1/2}.
\end{align*}
Using \eqref{eq:complete-graph-variance-identity} gives
\[
|\ell_i-\ell_j|\le \sqrt{\frac{2\mathcal D_{\log}(s)}{N}}.
\]
Since $\ell_i-\ell_j=\log(s_i/s_j)$, exponentiating the two-sided inequality proves \eqref{eq:scale-ratio-bound-from-Dlog}.

\item \textbf{Relative scale-ratio disruption.}
Every summand in \eqref{eq:scale-ratio-disruption} is nonnegative, and the base scale weights $(C_i^{\mathrm{base}}/C_j^{\mathrm{base}})^2$ are strictly positive.  Hence $\mathcal D_{\mathrm{ratio}}(s)=0$ if and only if
\[
\frac{s_i}{s_j}-1=0
\]
for every $i<j$.  This is equivalent to $s_i=s_j$ for every pair, hence to $s_1=\cdots=s_N$.
\end{enumerate}
\end{proof}

\begin{definition}[SVD-frame rotations for a local layer update]
\label{def:svd-frame-rotations}
Let a base layer have SVD
\[
W^{\mathrm{base}}=U\Sigma V^\top,
\qquad
\Sigma=\diag(\sigma_1,\ldots,\sigma_d),
\qquad \sigma_i\ge0.
\]
A post-adaptation layer with the same dimensions is represented in the base singular frames as
\begin{equation}\label{eq:post-svd-frame-rotation}
W^{\mathrm{post}}=UQ_U\Sigma^{\mathrm{post}}Q_V^\top V^\top,
\qquad
Q_U,Q_V\in\Orth(d),
\end{equation}
where $Q_U$ and $Q_V$ are the output-frame and input-frame rotations relative to the base SVD gauge.  The layer displacement is
\[
\Delta W:=W^{\mathrm{post}}-W^{\mathrm{base}}.
\]
\end{definition}

\begin{proposition}[Frobenius cost of coherent and incoherent SVD-frame rotations]
\label{prop:coherent-svd-rotation}
In the setting of Definition~\ref{def:svd-frame-rotations}, the following statements hold.
\begin{enumerate}[label=(R\arabic*),leftmargin=2.1em]
\item \textbf{Exact frame-reduced perturbation formula.}
\begin{equation}\label{eq:frame-reduced-perturbation}
\|\Delta W\|_F
=
\|Q_U\Sigma^{\mathrm{post}}Q_V^\top-\Sigma\|_F.
\end{equation}
\item \textbf{Uniform scale plus coherent rotation.}
If $\Sigma^{\mathrm{post}}=s\Sigma$ with $s>0$ and $Q_U=Q_V=Q$, then
\begin{equation}\label{eq:coherent-rotation-cost-general-s}
\|\Delta W\|_F^2
=
\sum_{a=1}^d\sum_{b=1}^d q_{ab}^2\,(s\sigma_b-\sigma_a)^2.
\end{equation}
In particular, when $s=1$,
\begin{equation}\label{eq:coherent-rotation-cost}
\|\Delta W\|_F^2
=
\sum_{a,b}q_{ab}^2(\sigma_b-\sigma_a)^2.
\end{equation}
Thus a common rotation has small cost when it mainly mixes singular directions with nearly equal singular values.
\item \textbf{Relative rotation control.}
Assume $\Sigma^{\mathrm{post}}=s\Sigma$ with $s>0$ and define the relative rotation
\[
R:=Q_U^\top Q_V.
\]
Then
\begin{equation}\label{eq:relative-rotation-bound}
\|s\Sigma(R^\top-\Id)\|_F
\le
\|\Delta W\|_F+
\|Q_U(s\Sigma)Q_U^\top-\Sigma\|_F.
\end{equation}
If $\sigma_d>0$, then
\begin{equation}\label{eq:relative-rotation-bound-sigmin}
\|R-\Id\|_F
\le
\frac{\|\Delta W\|_F+
\|Q_U(s\Sigma)Q_U^\top-\Sigma\|_F}{s\sigma_d}.
\end{equation}
\end{enumerate}
Consequently, in a low-displacement update for which the common-frame rotation is also low cost, the left and right singular-vector rotations must remain close in the relative-rotation metric.
\end{proposition}

\begin{proof}
We prove the three assertions.
\begin{enumerate}[label=(R\arabic*),leftmargin=2.1em]
\item \textbf{Exact frame-reduced perturbation formula.}
Using \eqref{eq:post-svd-frame-rotation},
\[
\Delta W
=UQ_U\Sigma^{\mathrm{post}}Q_V^\top V^\top-U\Sigma V^\top
=U\bigl(Q_U\Sigma^{\mathrm{post}}Q_V^\top-\Sigma\bigr)V^\top.
\]
The Frobenius norm is invariant under multiplication by orthogonal matrices on the left and right.  Therefore
\[
\|\Delta W\|_F=
\|Q_U\Sigma^{\mathrm{post}}Q_V^\top-\Sigma\|_F.
\]

\item \textbf{Uniform scale plus coherent rotation.}
Set $D=s\Sigma$ and $E=\Sigma$.  If $Q_U=Q_V=Q$, then by (R1)
\[
\|\Delta W\|_F^2=\|QDQ^\top-E\|_F^2.
\]
Expanding the square gives
\[
\|QDQ^\top-E\|_F^2
=\tr(D^2)+\tr(E^2)-2\tr(QDQ^\top E).
\]
Because $D$ and $E$ are diagonal,
\[
\tr(QDQ^\top E)=\sum_{a=1}^d\sum_{b=1}^d q_{ab}^2 d_b e_a.
\]
Also $\sum_a q_{ab}^2=1$ for each $b$ and $\sum_b q_{ab}^2=1$ for each $a$.  Hence
\begin{align*}
\tr(D^2)+\tr(E^2)-2\tr(QDQ^\top E)
&=\sum_{a,b}q_{ab}^2d_b^2+
  \sum_{a,b}q_{ab}^2e_a^2-
  2\sum_{a,b}q_{ab}^2d_be_a\\
&=\sum_{a,b}q_{ab}^2(d_b-e_a)^2.
\end{align*}
Substituting $d_b=s\sigma_b$ and $e_a=\sigma_a$ proves \eqref{eq:coherent-rotation-cost-general-s}.  Setting $s=1$ gives \eqref{eq:coherent-rotation-cost}.

\item \textbf{Relative rotation control.}
Let $R=Q_U^\top Q_V$, so $Q_V=Q_UR$ and $Q_V^\top=R^\top Q_U^\top$.  Then
\[
Q_U(s\Sigma)Q_V^\top
=Q_U(s\Sigma)R^\top Q_U^\top.
\]
Subtract the coherent-rotation matrix $Q_U(s\Sigma)Q_U^\top$:
\begin{align*}
Q_U(s\Sigma)Q_V^\top-Q_U(s\Sigma)Q_U^\top
&=Q_U\bigl(s\Sigma(R^\top-\Id)\bigr)Q_U^\top.
\end{align*}
Orthogonal invariance gives
\[
\|Q_U(s\Sigma)Q_V^\top-Q_U(s\Sigma)Q_U^\top\|_F
=
\|s\Sigma(R^\top-\Id)\|_F.
\]
By the triangle inequality,
\begin{align*}
\|s\Sigma(R^\top-\Id)\|_F
&\le
\|Q_U(s\Sigma)Q_V^\top-\Sigma\|_F
+
\|Q_U(s\Sigma)Q_U^\top-\Sigma\|_F.
\end{align*}
The first term is $\|\Delta W\|_F$ by (R1), proving \eqref{eq:relative-rotation-bound}.  If $\sigma_d>0$, then the smallest singular value of $s\Sigma$ is $s\sigma_d$, and
\[
\|s\Sigma(R^\top-\Id)\|_F\ge s\sigma_d\|R^\top-\Id\|_F=s\sigma_d\|R-\Id\|_F.
\]
Combining this lower bound with \eqref{eq:relative-rotation-bound} gives \eqref{eq:relative-rotation-bound-sigmin}.
\end{enumerate}
\end{proof}

\begin{remark}[Deep-learning interpretation]
Proposition~\ref{prop:uniform-scale-finetuning} identifies uniform layerwise scaling as the zero-disruption scale direction.  Proposition~\ref{prop:coherent-svd-rotation} identifies the corresponding directional condition: if a fine-tuned layer remains close to the base layer in Frobenius norm, and if its common frame rotation does not mix widely separated singular directions, then the output and input singular-vector rotations must be close to one another.  In channel terms, low-disruption adaptation preserves the layer's input-output frame up to a coherent rotation and avoids arbitrary re-wiring of the dominant singular directions.
\end{remark}

\section{Empirical measurements of the finite-dimensional predictions}
\label{sec:empirical}

The experiments are organized by the finite-dimensional variables appearing in the theorems.  The spectral measurements track the fitted Cartan coordinate.  The physical-alignment measurements build the transport matrices and their block-energy summaries.  The effective-rank measurements compare whether the same block structure persists under different spectral truncation windows.  This organization separates three levels of empirical measurements: spectral evolution, channel organization, and stability of the observed organization under changes of the retained rank window.

\paragraph{Checkpoint-level measurement procedure.}
For a trained checkpoint and a selected sequence of layerwise matrices, the measurements are computed in a deterministic order.
\begin{enumerate}[label=(\arabic*),leftmargin=2.0em]
\item Compute singular values and singular vectors of each selected operator $W_k=U_k\Sigma_kV_k^\top$; for large matrices this step may use standard truncated or randomized SVD methods \cite{halko2011finding}.
\item Fit the Cartan coordinate $\widehat\alpha_k$ from the singular-value profile and record the fit residual or regression score.
\item Choose an energy threshold $1-\eps$ and compute the effective-rank window $R_\eps(W_k)$ from Definition~\ref{def:effective-rank}.
\item Construct the interface matrices used in the figures. The output-realized scale-free angular transport is denoted by
\[
M_s:=U_{k+1}^{(R)}(V_{k+1}^{(R)})^{\top}U_k^{(R)},
\]
while the target-truncated physical, energy-realized transport is
\[
M^{(R)}:=W_{k+1}^{[R]}U_k^{(R)}
=U_{k+1}^{(R)}\Sigma_{k+1}^{(R)}(V_{k+1}^{(R)})^{\top}U_k^{(R)}.
\]
Some displayed experiments use the full-row version $M:=W_{k+1}U_k^{(R)}$; in that case
\[
M=M^{(R)}+(W_{k+1}-W_{k+1}^{[R]})U_k^{(R)},
\]
and the target-tail residual satisfies
\[
\|(W_{k+1}-W_{k+1}^{[R]})U_k^{(R)}\|_F\le E_{>R}(W_{k+1})^{1/2}.
\]
These are the experimental representatives of the transport variants in Definition~\ref{def:transport-variants}.
\item Apply the chosen physical ordering or clustering permutation to obtain the measured physical alignment matrix $\widehat M_{\mathrm{phy}}$ as in Definition~\ref{def:physical-alignment-matrix}.
\item Compute the block-energy matrices
\[
E_r(M),\qquad E_r(M_s),
\]
which are instances of Definition~\ref{def:block-energy-matrix}.  These matrices summarize how much row-group energy flows into each active column group after permutation.
\end{enumerate}

\paragraph{Operator and protocol reporting.}
A reproducible certificate report must state the extraction protocol $\mathcal E$ of Definition~\ref{def:extraction-protocol}.  In particular, it must identify the measured operator for each architecture: for transformers, which attention or MLP matrices are used, or whether a composed interface block is measured; for convolutional networks, how convolution kernels are flattened and whether normalization layers are folded into the matrix; and for diffusion or vision backbones, which linearized block is treated as $W_k$.  The report must also state whether the displayed columns are source singular modes or physical input channels, following Table~\ref{tab:transport-coordinate-types}.

\paragraph{Null and randomization baselines.}
Because permutations and clustering can create visually organized displays, the same extraction pipeline should also be run on null controls before making an empirical membership claim.  Recommended controls include random Gaussian matrices with matched dimensions, matrices with the same singular values but random singular vectors, trained weights with randomly permuted channels, untrained initialization, random orthogonal rotations of singular frames, and the same clustering pipeline applied to null data.  A Physical GSA certificate should report whether the measured margins distinguish the trained interface from these controls.

\paragraph{How empirical support is interpreted.}
The deterministic theorems have the form
\[
\text{explicit numerical hypotheses}\quad\Longrightarrow\quad\text{stable structural conclusion}.
\]
The figures measure the matrices and block-energy quantities that occur in those hypotheses and conclusions.  A smooth exponent plot supports the spectral part by showing that the measured fitted coordinates have small total variation.  A block-dominant $E_r$ heatmap supports the physical-alignment part by showing that core and accepted-overlap blocks carry most of the measured row-normalized energy.  An effective-rank-window comparison supports the truncation part by showing that the same block pattern is visible under nearby retained-energy windows.  A complete finite-dimensional margin test additionally requires the corresponding numerical margins: local transport values, power-law fit errors, rank-tail gaps, active-column gaps, bad-block mass, and pairwise one-third margin screens.

\begin{table}[H]
\centering
\scriptsize
\renewcommand{\arraystretch}{1.32}
\setlength{\tabcolsep}{2.8pt}
\begin{tabularx}{\textwidth}{L{0.21\textwidth}L{0.33\textwidth}Y}
\toprule
\textbf{Measurement family} & \textbf{Theoretical prediction} & \textbf{Measured quantity and interpretation}\\
\midrule
Exponent-profile plots & Theorem~\ref{thm:cartan-rigidity} and Corollary~\ref{cor:cartan-to-rank-window}. & Fitted $\widehat\alpha_k$ values are computed across depth.  Short trajectories correspond to small layerwise coordinate changes rather than noisy, unrelated spectra.\\
Single-interface alignment example & Definitions~\ref{def:transport-variants}--\ref{def:physical-alignment-matrix} and Propositions~\ref{thm:block-energy-margin}--\ref{thm:heatmap-to-noise-bound}. & A representative transport matrix is filtered, clustered, permuted, and summarized by $E_r$.  The resulting matrix and heatmap display the measured core/overlap/noise pattern.\\
Fixed-cluster panels & Lemma~\ref{thm:one-third-threshold}, \ref{thm:physical-incidence-margin}, and~\ref{thm:block-energy-margin}. & A prescribed number of groups is used to test whether a proposed physical grouping exposes a modular transport structure.\\
$25\mathrm{ER}$ and $50\mathrm{ER}$ panels & Corollary~\ref{cor:Er-window-robustness} and Theorem~\ref{thm:familywise-structure-persistence}. & The same interface is remeasured under different retained-energy windows.  Persistence of the block pattern indicates robustness to the spectral truncation choice.\\
$M_s$ versus $M$ panels & Proposition~\ref{thm:scale-weight-transfer}, Proposition~\ref{prop:row-leakage-scale-transfer}, and Proposition~\ref{thm:heatmap-margin-screen}. & Agreement between scale-free angular transport and energy-weighted physical transport is a measured consistency; theorem-level transfer in physical rows additionally requires diagonal weighting in the displayed coordinates or a verified row-leakage bound.\\
Layer sweeps & Theorem~\ref{thm:dynamic-to-static-bridge}, Proposition~\ref{thm:static-structure-closure}, Theorem~\ref{thm:familywise-structure-persistence}, and Corollary~\ref{thm:multi-view-measurement-aggregation}. & Repeated block-energy structure across depth indicates coherent static channel organization rather than a single-layer artifact.\\
\bottomrule
\end{tabularx}
\caption{Placement of experimental figures in the theorem chain.  Each family of figures is attached to a mathematical prediction and to the finite quantity used to evaluate it.}
\label{tab:experiment-theorem-map}
\end{table}

\begin{table}[H]
\centering
\scriptsize
\renewcommand{\arraystretch}{1.32}
\setlength{\tabcolsep}{2.8pt}
\begin{tabularx}{\textwidth}{L{0.22\textwidth}L{0.34\textwidth}Y}
\toprule
\textbf{Measured quantity} & \textbf{How it is computed} & \textbf{Role in a finite-dimensional margin test}\\
\midrule
Cartan total variation & $\sum_k |\widehat\alpha_{k+1}-\widehat\alpha_k|$ after power-law fitting. & Empirical left-hand side of the coordinate-rigidity bound.\\
Rank-window stability & Compare $25\mathrm{ER}$, $50\mathrm{ER}$, and the corresponding truncated heatmaps. & Tests whether the block structure is stable under the truncation error controlled by Corollary~\ref{cor:Er-window-robustness}.\\
Diagonal block mass & Row-normalized mass of $E_r(M)$ or $E_r(M_s)$ on diagonal or accepted blocks. & Measured version of the core/accepted-overlap dominance controlled by Proposition~\ref{thm:block-energy-margin}.\\
Visible bad mass & Mass of $E_r$ outside accepted blocks. & Input to the Frobenius visible-noise bound in Proposition~\ref{thm:heatmap-to-noise-bound}.\\
One-third margin screen & Check inequalities such as $e_iE_{ij}+e_jE_{ji}<m_{i,j}^2/9$. & Sufficient measured condition for the one-third coherent-overlap threshold through Proposition~\ref{thm:heatmap-margin-screen}.\\
Hub degree & Number of active row groups incident to an active support column. & Measured version of the shared-support structure in Definition~\ref{def:static-channel-incidence} and the scaling law in Proposition~\ref{thm:hub-scaling}.\\
\bottomrule
\end{tabularx}
\caption{Numerical quantities associated with the displayed measurements.  The figures show the matrices or heatmaps from which these values are computed; complete verification requires the margin values in this table to be reported and checked.}
\label{tab:measurement-statistics}
\end{table}

\begin{remark}[Boundary between empirical support and finite-dimensional margin tests]
In the formal statements, a displayed heatmap corresponds to a concrete matrix, a row/column partition, and a finite list of numerical inequalities.  For example, an $E_r$ heatmap is the row-normalized block-energy matrix of Definition~\ref{def:block-energy-matrix}.  When the measured bad mass and the one-third margin screen satisfy Propositions~\ref{thm:heatmap-to-noise-bound} and~\ref{thm:heatmap-margin-screen}, the heatmap data give a finite-dimensional physical-alignment certificate.  The figures show that the matrices and block-energy quantities predicted by the theory are visible in trained models; the associated margin verification is reported through the numerical inequalities.
\end{remark}

\paragraph{Reading the alignment figure set as empirical measurements.}
The systematic alignment figure set has three complementary roles.  Fixed-cluster panels test block-sparse structure under a prescribed grouping.  The $25\mathrm{ER}$ and $50\mathrm{ER}$ panels test robustness under spectral truncation.  Layer sweeps test persistence across depth.  The block-energy matrices $E_r(M)$ and $E_r(M_s)$ capture energy-weighted and scale-free transport, respectively.  Agreement between these panels, together with concentration of diagonal mass and controlled off-diagonal energy, is the measured signature corresponding to Proposition~\ref{thm:scale-weight-transfer} and Propositions~\ref{thm:heatmap-margin-screen}--\ref{thm:heatmap-numerical-certificate}.

\subsection{Empirical measurement II: block-sparse physical alignment structure}
\label{sec:prediction-physical}
This measurement instantiates the finite physical objects defined in Sections~\ref{sec:physical-alignment}--\ref{sec:block-energy-tests}.  A block-diagonal or block-dominant pattern in the permuted transport matrix means that most channel groups use their own active supports.  Sparse off-diagonal blocks represent controlled sharing, while diffuse background mass corresponds to residual noise in the finite decomposition.
\label{sec:exp2-alignment}

Lemma~\ref{thm:one-third-threshold} and Definition~\ref{def:physical-gsa-domain} require that, after selecting a dominant effective-rank window and a physical ordering, the physical alignment matrix be well approximated by
\[
\widehat M_{\mathrm{phy}}
=
M_{\mathrm{core}}+M_{\mathrm{overlap}}+M_{\mathrm{noise}},
\]
with small noise and structured rather than dense overlap.
The alignment figure set below measures the finite matrices associated with this structural prediction: after energy filtering and cosine/spectral clustering, the matrices $M$, $M_s$, $E_r(M)$, and $E_r(M_s)$ show diagonal or block-dominant structure across ResNet50, Qwen3-8B, LLaMA3-8B, and DiT-XL-2-512.

\paragraph{Deep-learning meaning.}
The block-diagonal and sparse off-diagonal patterns are interpreted as static channel anatomy.  A diagonal block records a group of output channels drawing energy from its own support; a sparse off-diagonal block records controlled feature sharing; diffuse background records unstructured interference.  Through Propositions~\ref{thm:block-energy-margin}, \ref{thm:heatmap-to-noise-bound}, and~\ref{thm:heatmap-margin-screen}, the heatmaps identify the numerical margins needed to certify the corresponding core/overlap/noise decomposition.

\begin{figure}[H]
\centering
\begin{subfigure}[t]{0.49\textwidth}
\centering
\includegraphics[width=\linewidth]{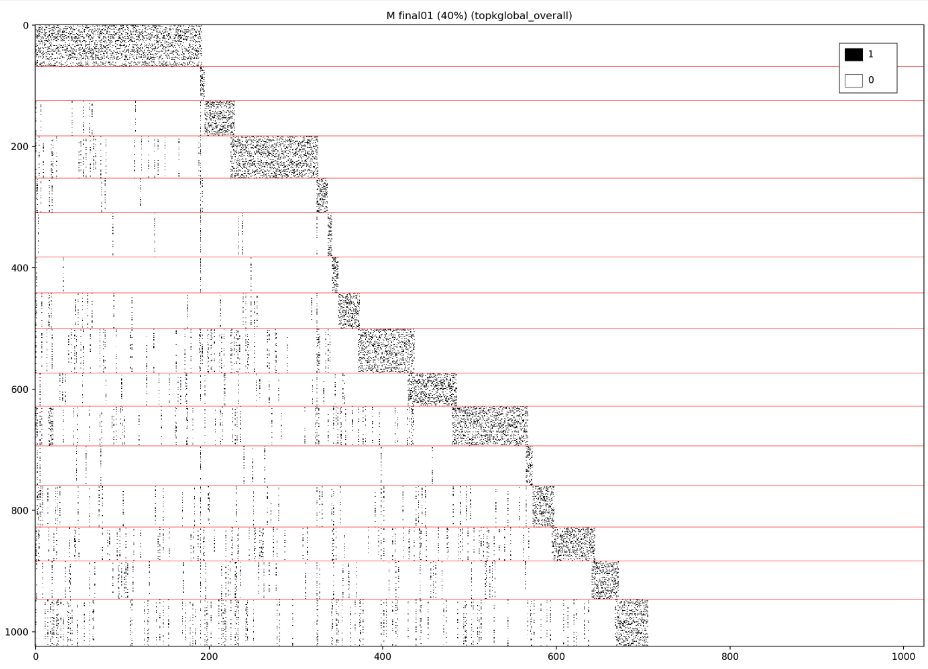}
\caption{Permuted $M$}
\end{subfigure}
\hfill
\begin{subfigure}[t]{0.49\textwidth}
\centering
\includegraphics[width=\linewidth]{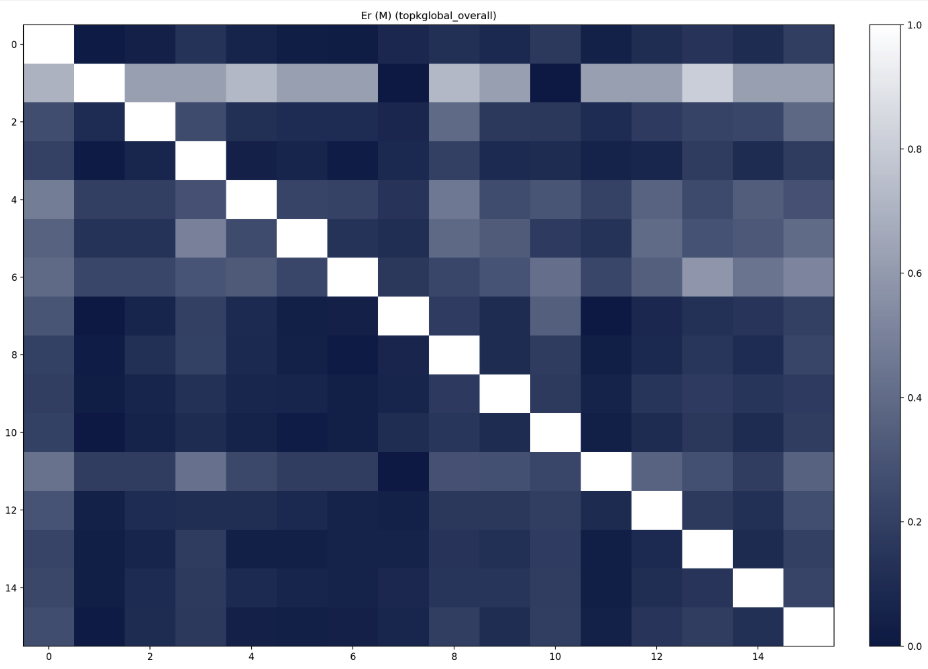}
\caption{$E_r(M)$}
\end{subfigure}

\medskip

\begin{subfigure}[H]{0.49\textwidth}
\centering
\includegraphics[width=\linewidth]{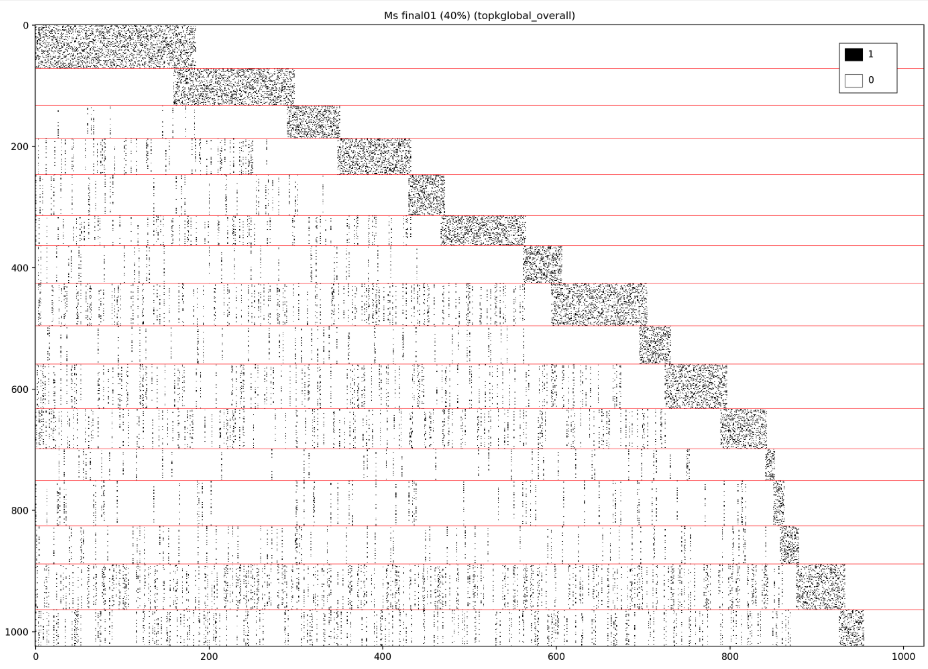}
\caption{Permuted $M_s$}
\end{subfigure}
\hfill
\begin{subfigure}[H]{0.49\textwidth}
\centering
\includegraphics[width=\linewidth]{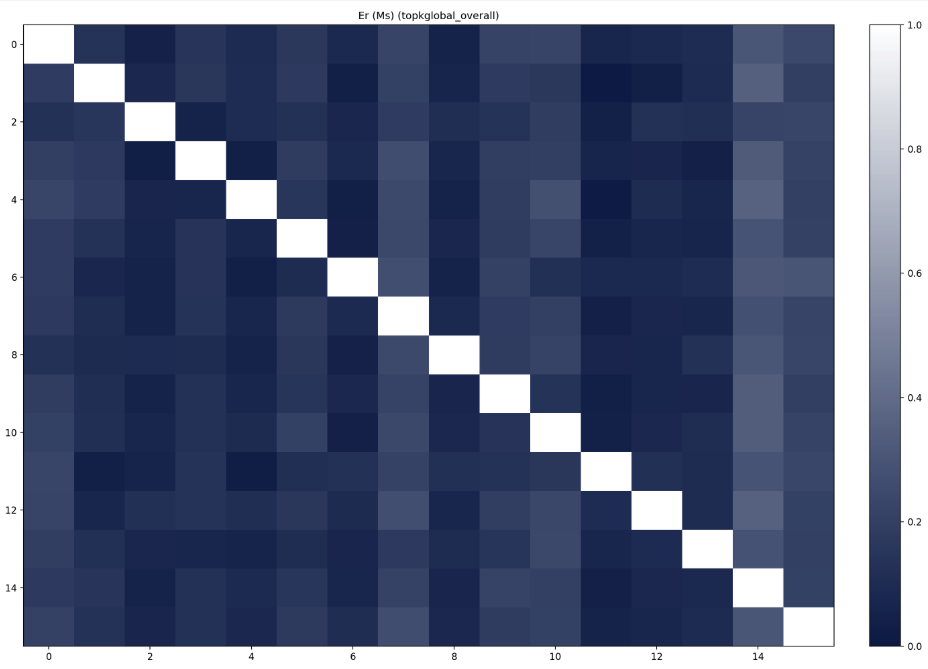}
\caption{$E_r(M_s)$}
\end{subfigure}
\caption{\textbf{Physical-alignment measurement:} representative four-panel interface measurement.  The panels show the permuted output-realized transport matrix $M=W_{k+1}U_k^{(R)}$ with physical output rows and source-mode columns, its block-energy summary $E_r(M)$, the scale-free angular transport matrix $M_s=U_{k+1}^{(R)}(V_{k+1}^{(R)})^{\top}U_k^{(R)}$, and its block-energy summary $E_r(M_s)$.  The block-dominant pattern is consistent with the decomposition $\widehat M_{\mathrm{phy}}=M_{\mathrm{core}}+M_{\mathrm{overlap}}+M_{\mathrm{noise}}$.  The heatmaps measure diagonal/core mass, structured off-diagonal sharing, and visible residual mass; a rigorous finite-dimensional margin test additionally requires the bad-mass and pairwise-margin inequalities in Table~\ref{tab:measurement-statistics}.}
\label{fig:align-16}
\end{figure}

\begin{figure}[H]
\centering
\begin{subfigure}[H]{0.49\textwidth}
\centering
\includegraphics[width=\linewidth]{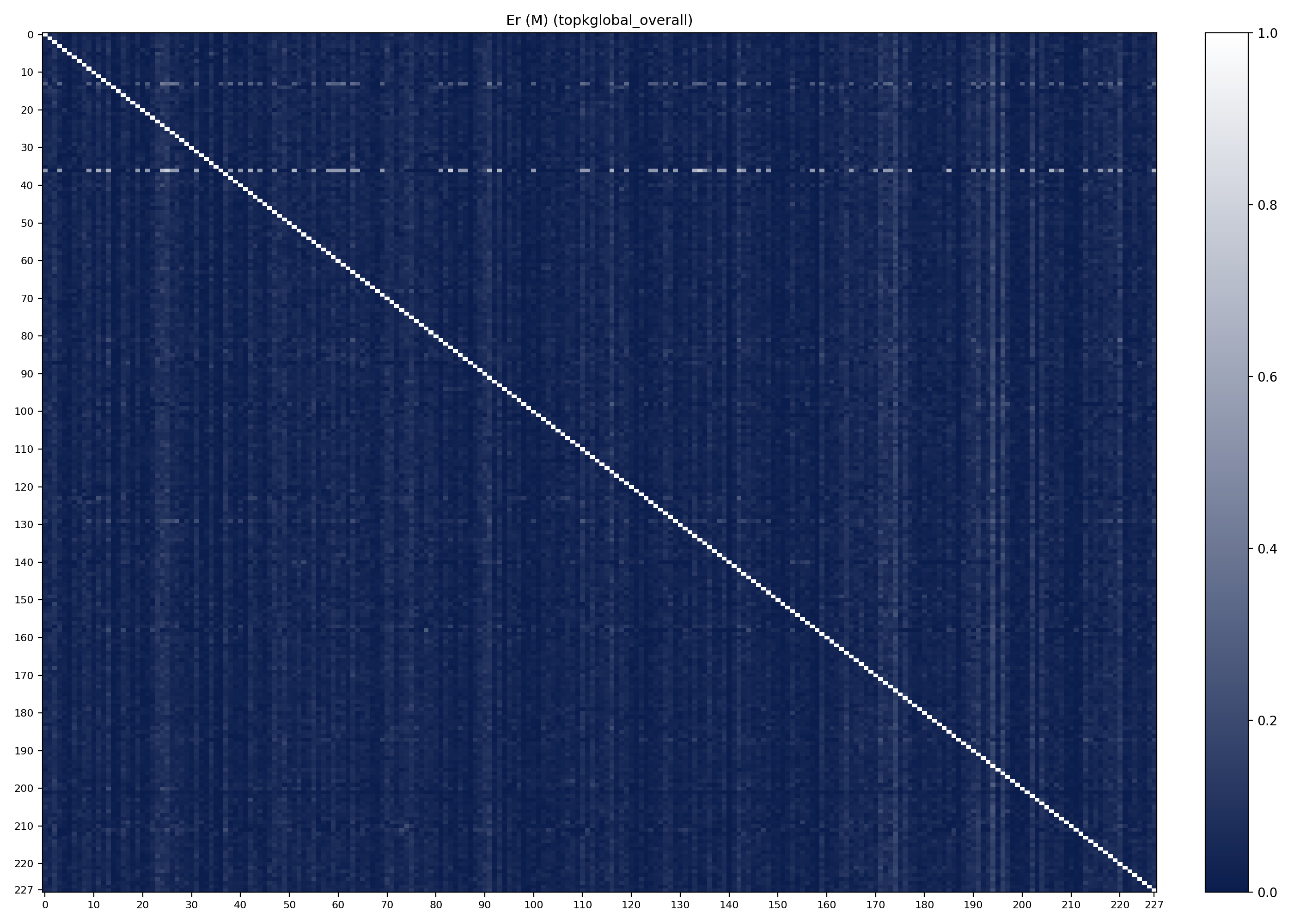}
\caption{Qwen3-8B, $50\mathrm{ER}$: $E_r(M)$}
\end{subfigure}
\hfill
\begin{subfigure}[H]{0.49\textwidth}
\centering
\includegraphics[width=\linewidth]{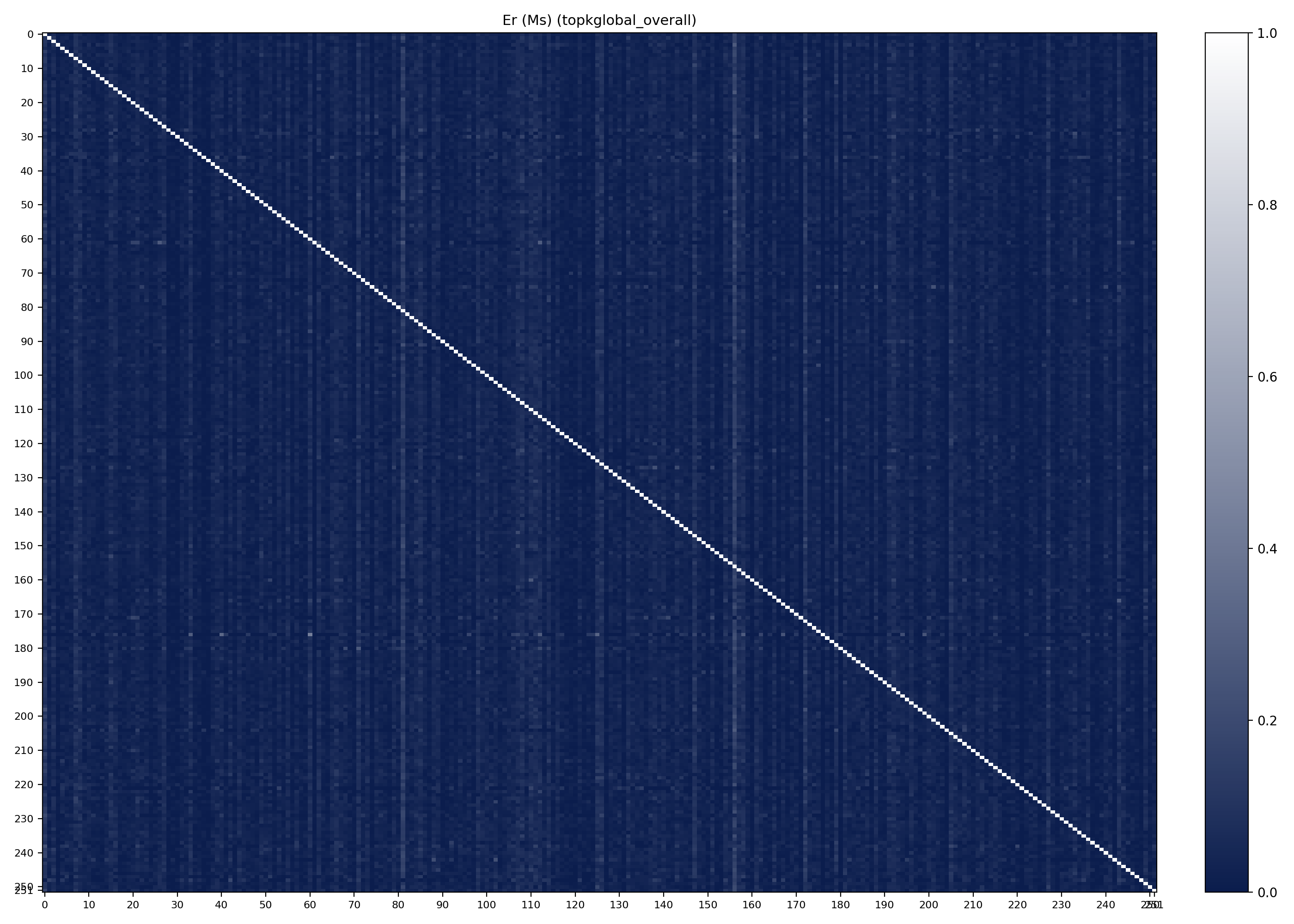}
\caption{Qwen3-8B, $50\mathrm{ER}$: $E_r(M_s)$}
\end{subfigure}

\medskip

\begin{subfigure}[H]{0.49\textwidth}
\centering
\includegraphics[width=\linewidth]{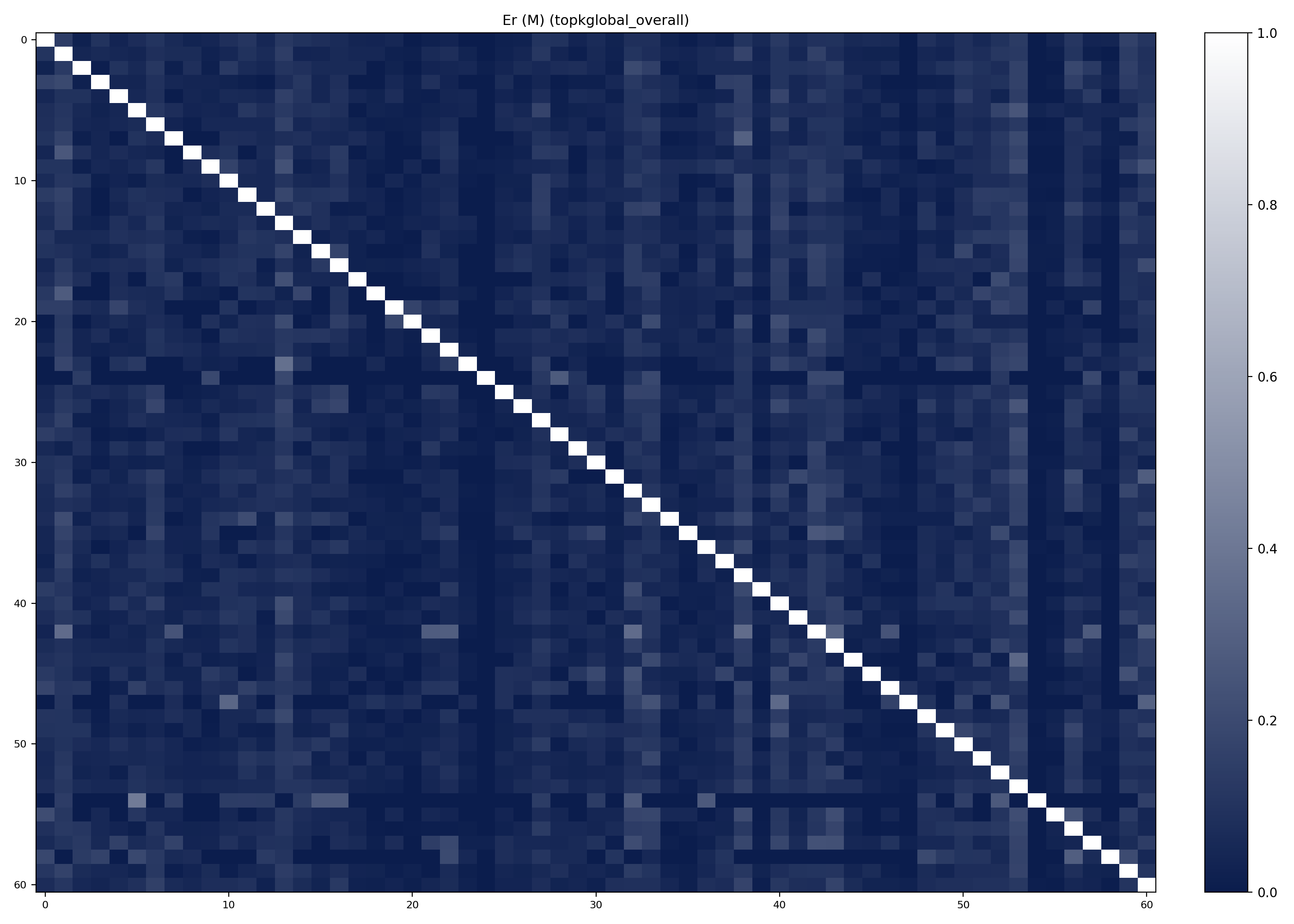}
\caption{DiT-XL-2-512, $50\mathrm{ER}$: $E_r(M)$}
\end{subfigure}
\hfill
\begin{subfigure}[H]{0.49\textwidth}
\centering
\includegraphics[width=\linewidth]{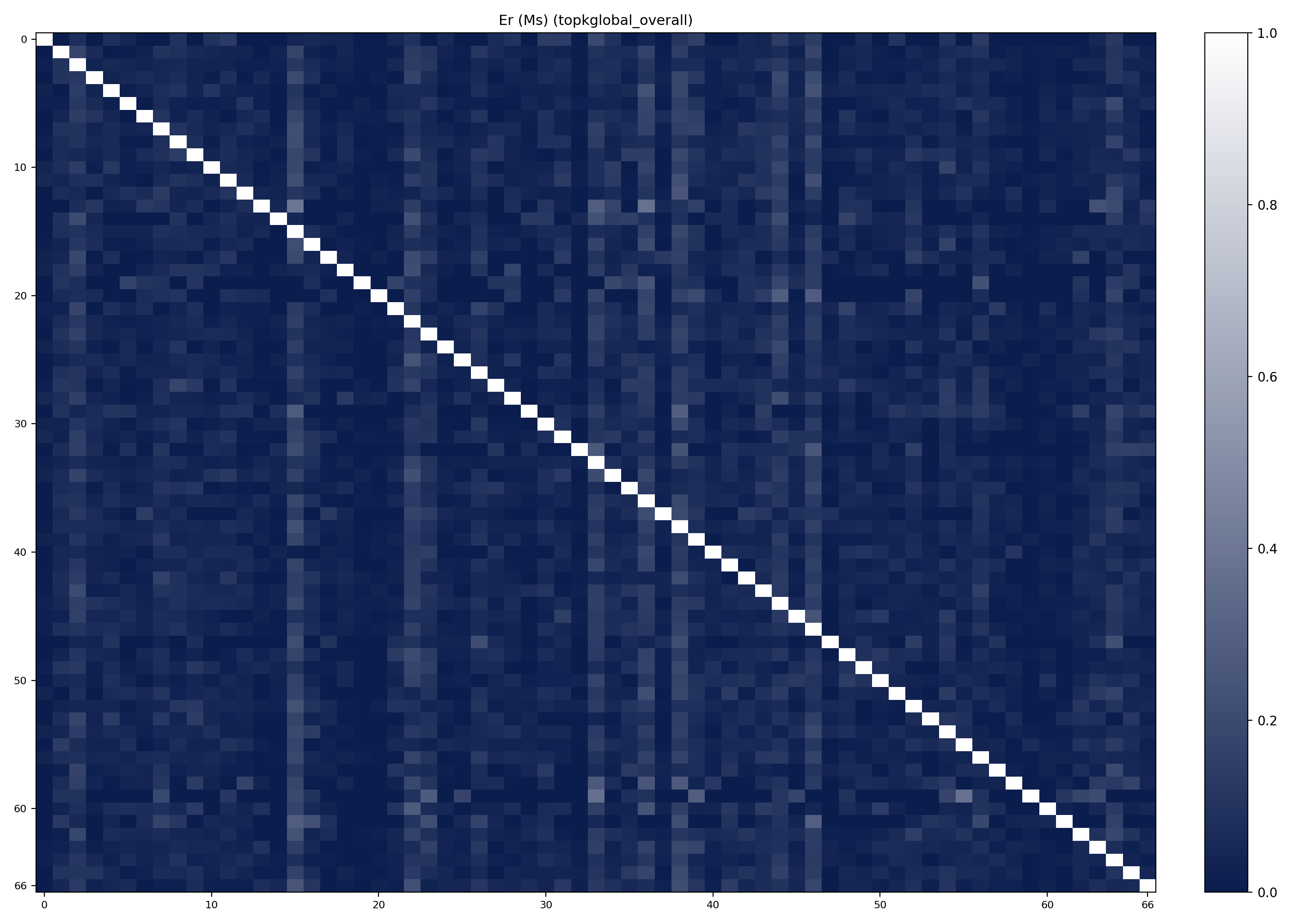}
\caption{DiT-XL-2-512, $50\mathrm{ER}$: $E_r(M_s)$}
\end{subfigure}
\caption{\textbf{Physical-alignment measured quantity:} representative block-energy matrices from the systematic \texttt{alignmentM} figure set.  The panels are direct instances of Definition~\ref{def:block-energy-matrix}.  Diagonal or block-dominant mass is the finite quantity controlled by Proposition~\ref{thm:block-energy-margin}; differences between $M_s$ and $M$ report how singular-value weighting and physical output realization change the scale-free angular organization; theorem-level transfer requires the diagonal-weighting or row-leakage conditions stated in Propositions~\ref{thm:scale-weight-transfer} and~\ref{prop:row-leakage-scale-transfer}.}
\label{fig:alignmentM-main-Er}
\end{figure}

\begin{figure}[H]
\centering
\begin{subfigure}[H]{0.49\textwidth}
\centering
\includegraphics[width=\linewidth]{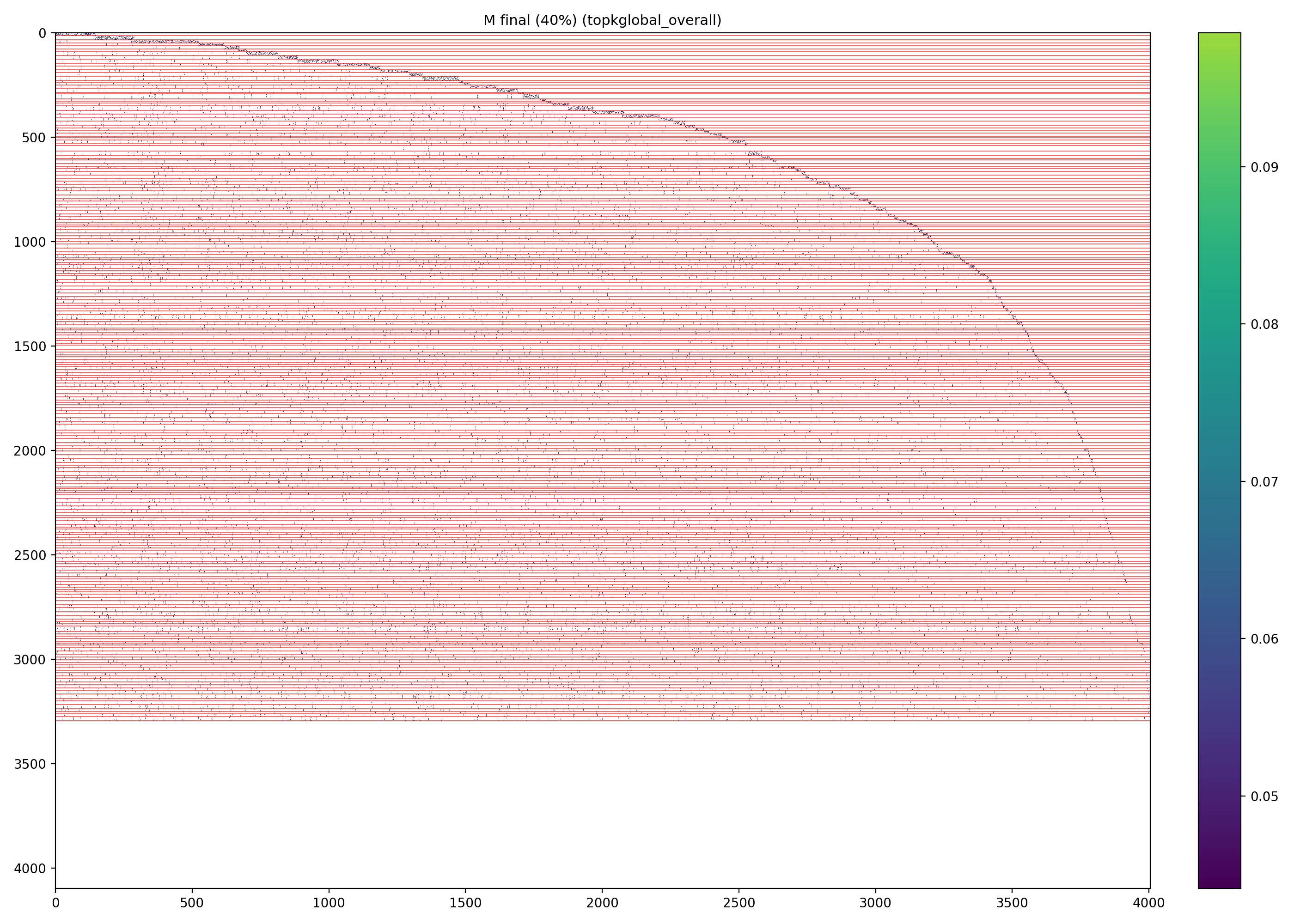}
\caption{$25\mathrm{ER}$: permuted $M$}
\end{subfigure}
\hfill
\begin{subfigure}[H]{0.49\textwidth}
\centering
\includegraphics[width=\linewidth]{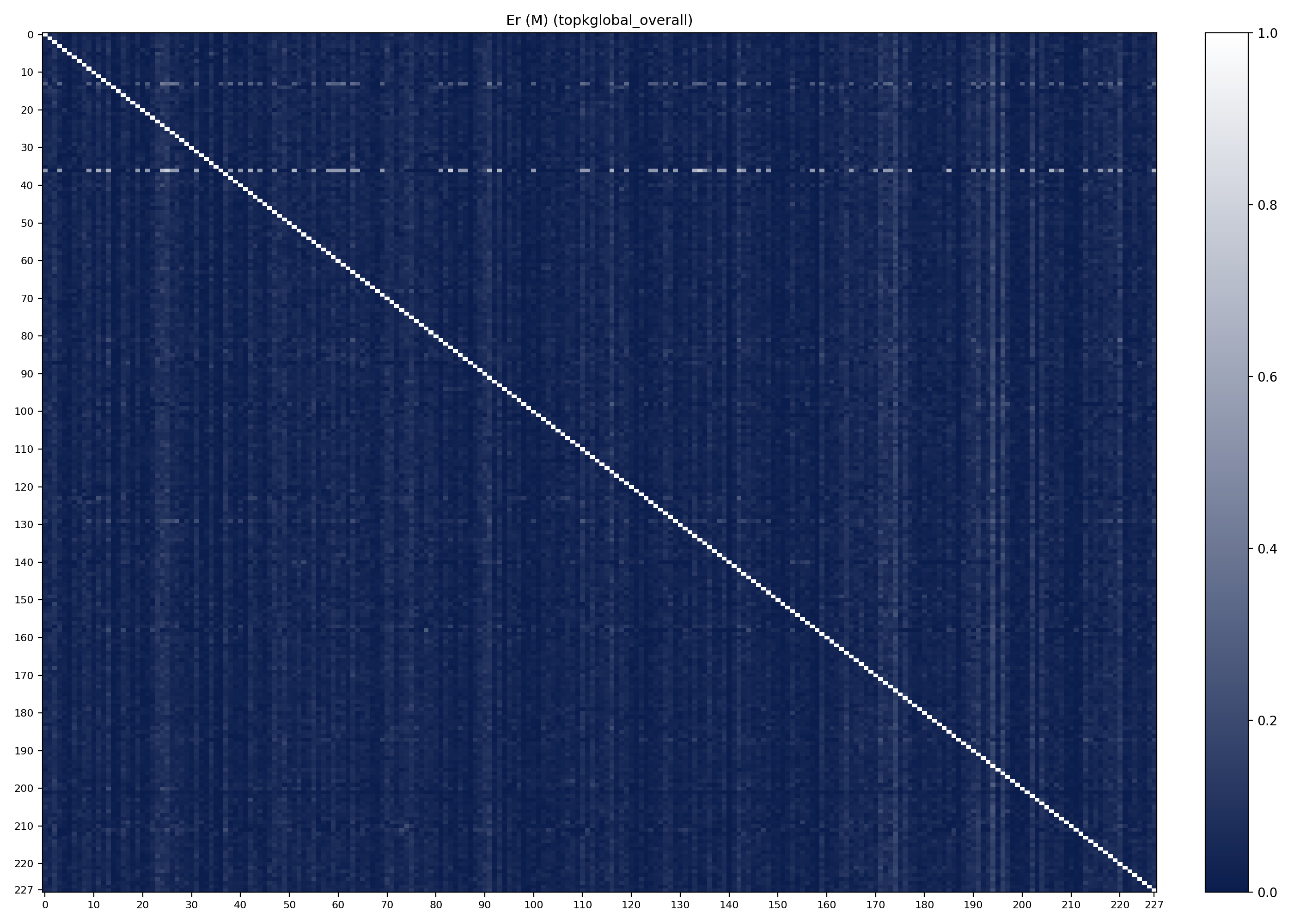}
\caption{$25\mathrm{ER}$: $E_r(M)$}
\end{subfigure}

\medskip

\begin{subfigure}[H]{0.49\textwidth}
\centering
\includegraphics[width=\linewidth]{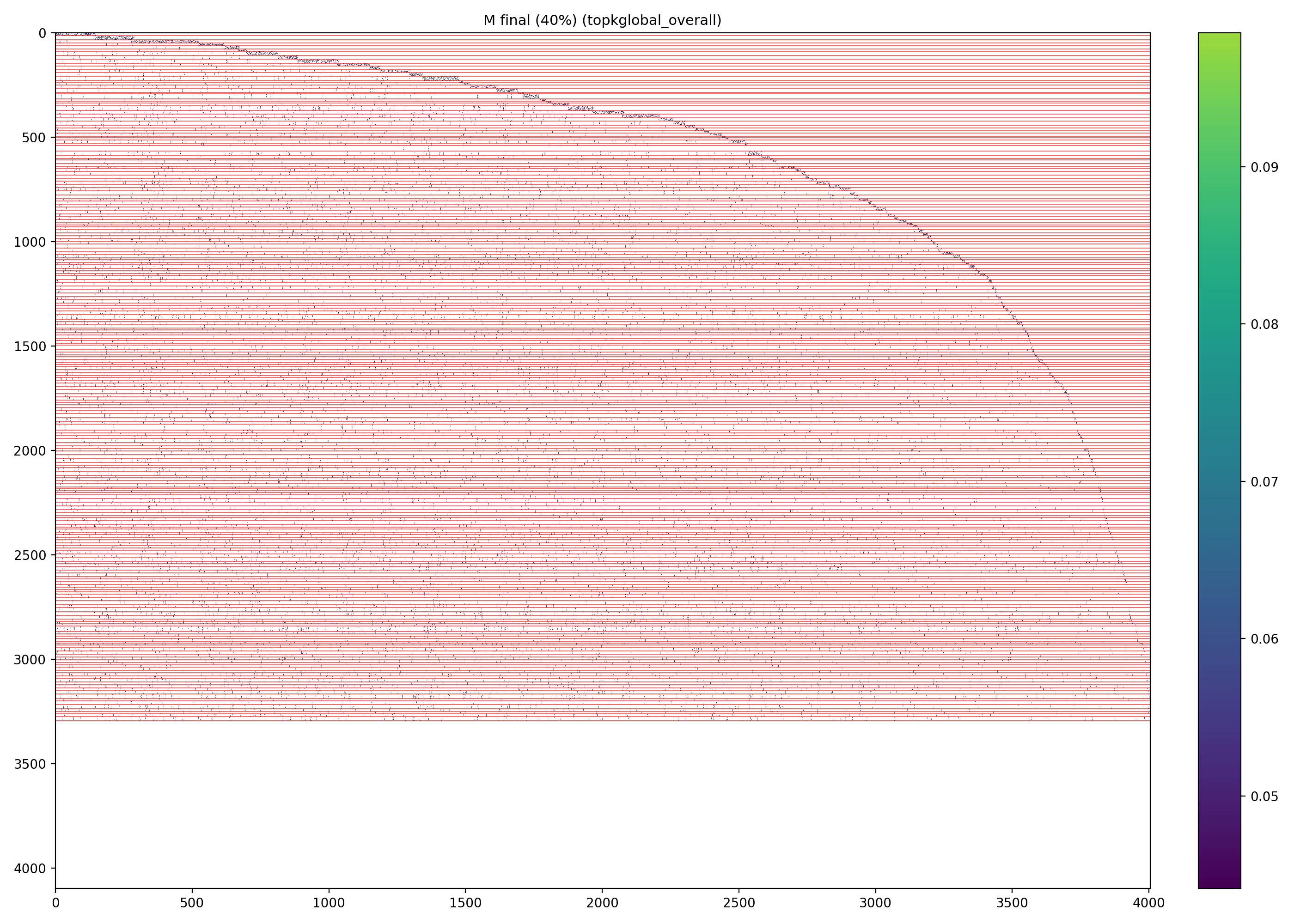}
\caption{$50\mathrm{ER}$: permuted $M$}
\end{subfigure}
\hfill
\begin{subfigure}[H]{0.49\textwidth}
\centering
\includegraphics[width=\linewidth]{alignmentM/Qwen3-8B/50ER/layer_18/M/Er_M.png}
\caption{$50\mathrm{ER}$: $E_r(M)$}
\end{subfigure}
\caption{\textbf{Physical/rank-window bridge:} the same Qwen3-8B interface is displayed under two energy-rank windows.  The repeated block-energy structure is consistent with the spectral-tail window stability and static GSA stability theorems.}
\label{fig:qwen-er-window-main}
\end{figure}

\begin{figure}[H]
\centering
\begin{subfigure}[H]{0.24\textwidth}
\centering
\includegraphics[width=\linewidth]{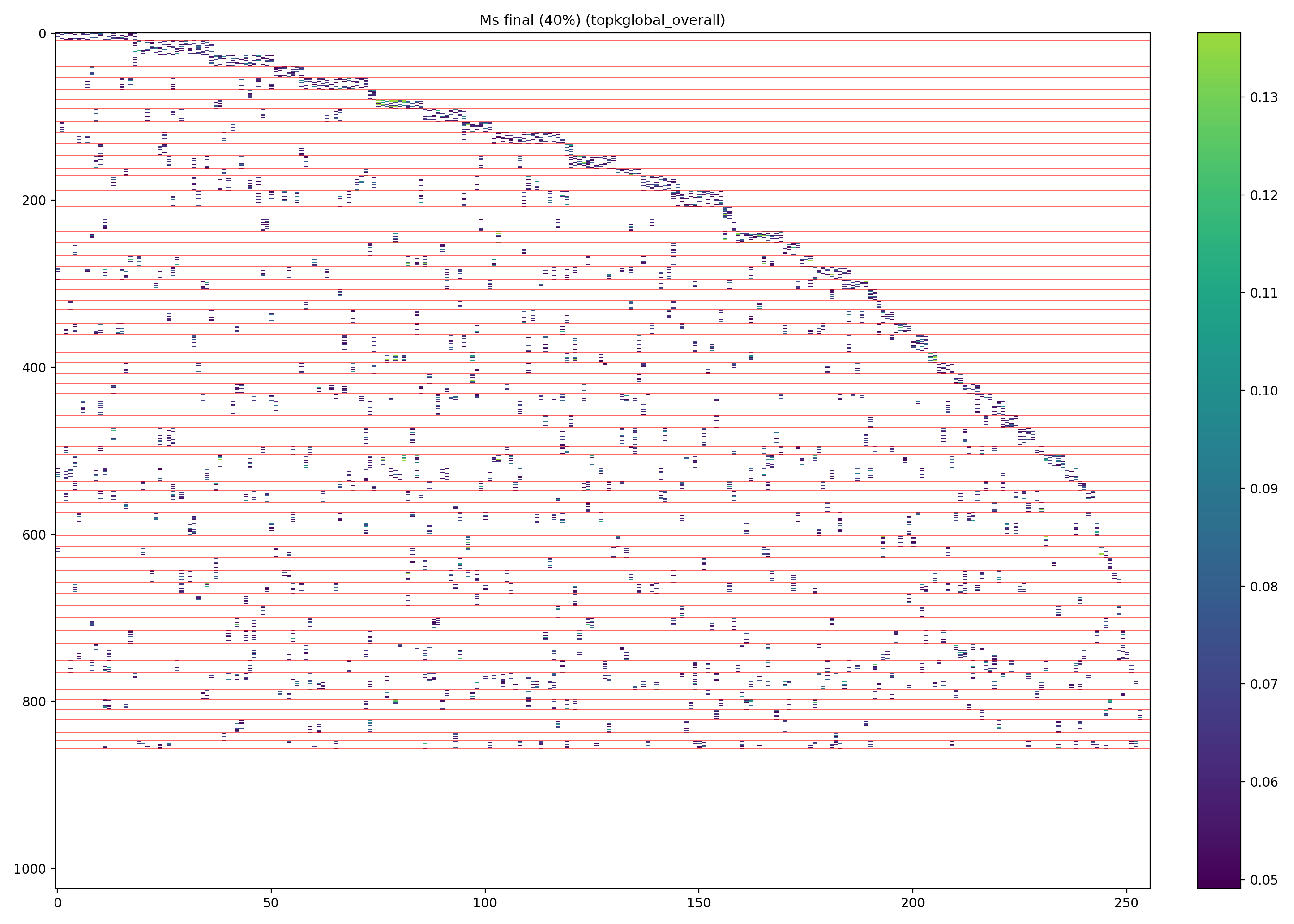}
\caption{permuted $M_s$}
\end{subfigure}\hfill
\begin{subfigure}[H]{0.24\textwidth}
\centering
\includegraphics[width=\linewidth]{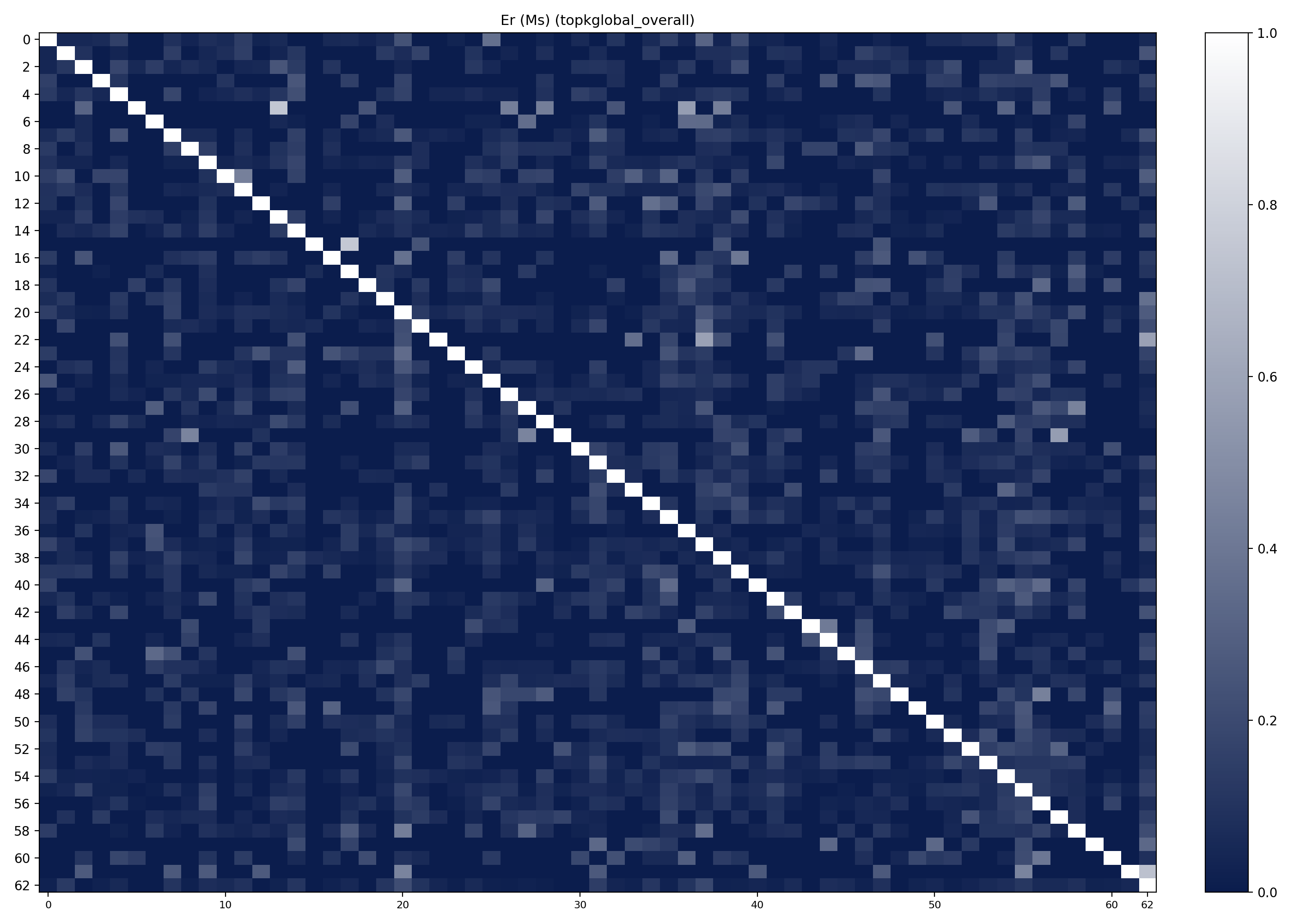}
\caption{$E_r(M_s)$}
\end{subfigure}\hfill
\begin{subfigure}[H]{0.24\textwidth}
\centering
\includegraphics[width=\linewidth]{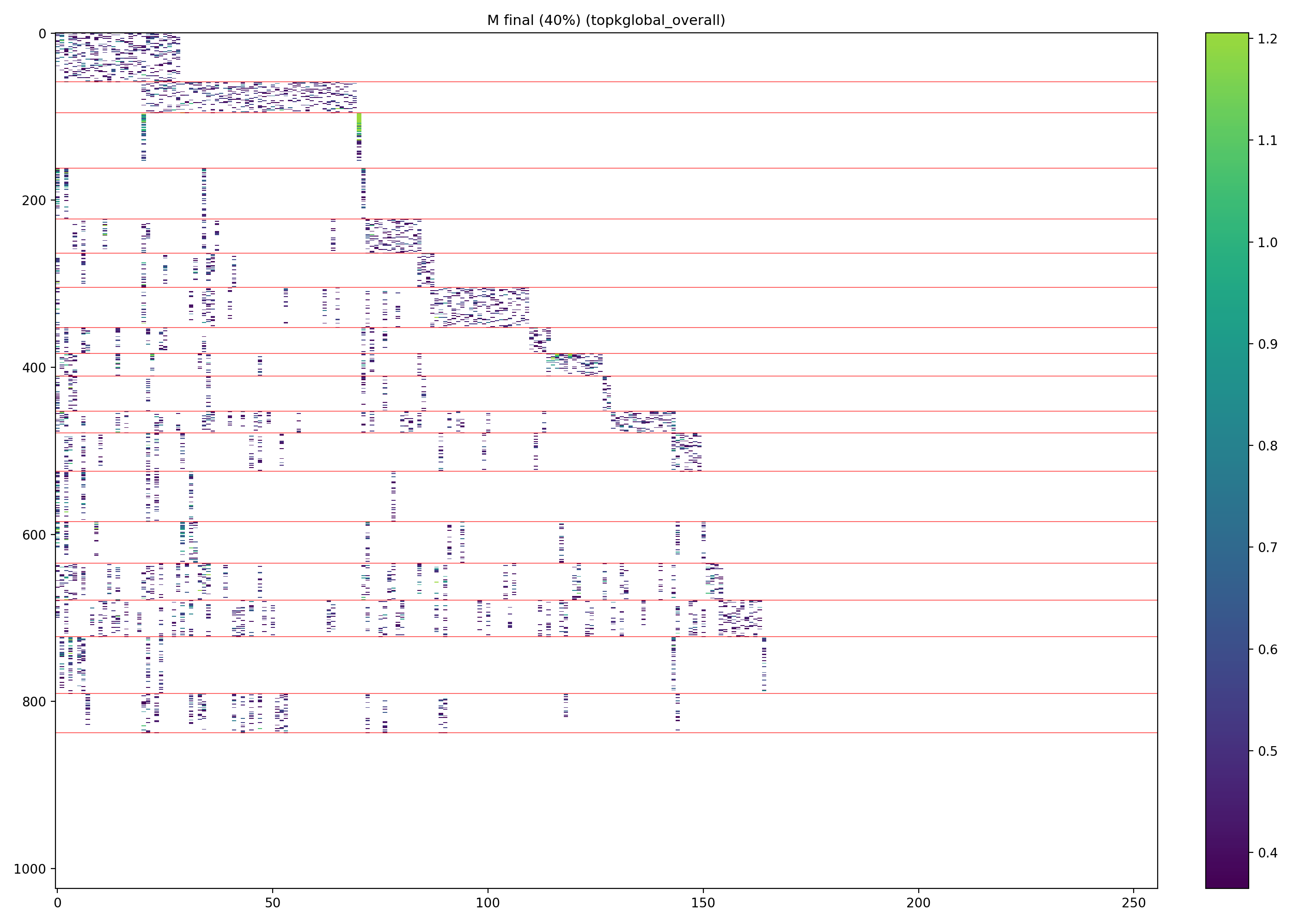}
\caption{permuted $M$}
\end{subfigure}\hfill
\begin{subfigure}[H]{0.24\textwidth}
\centering
\includegraphics[width=\linewidth]{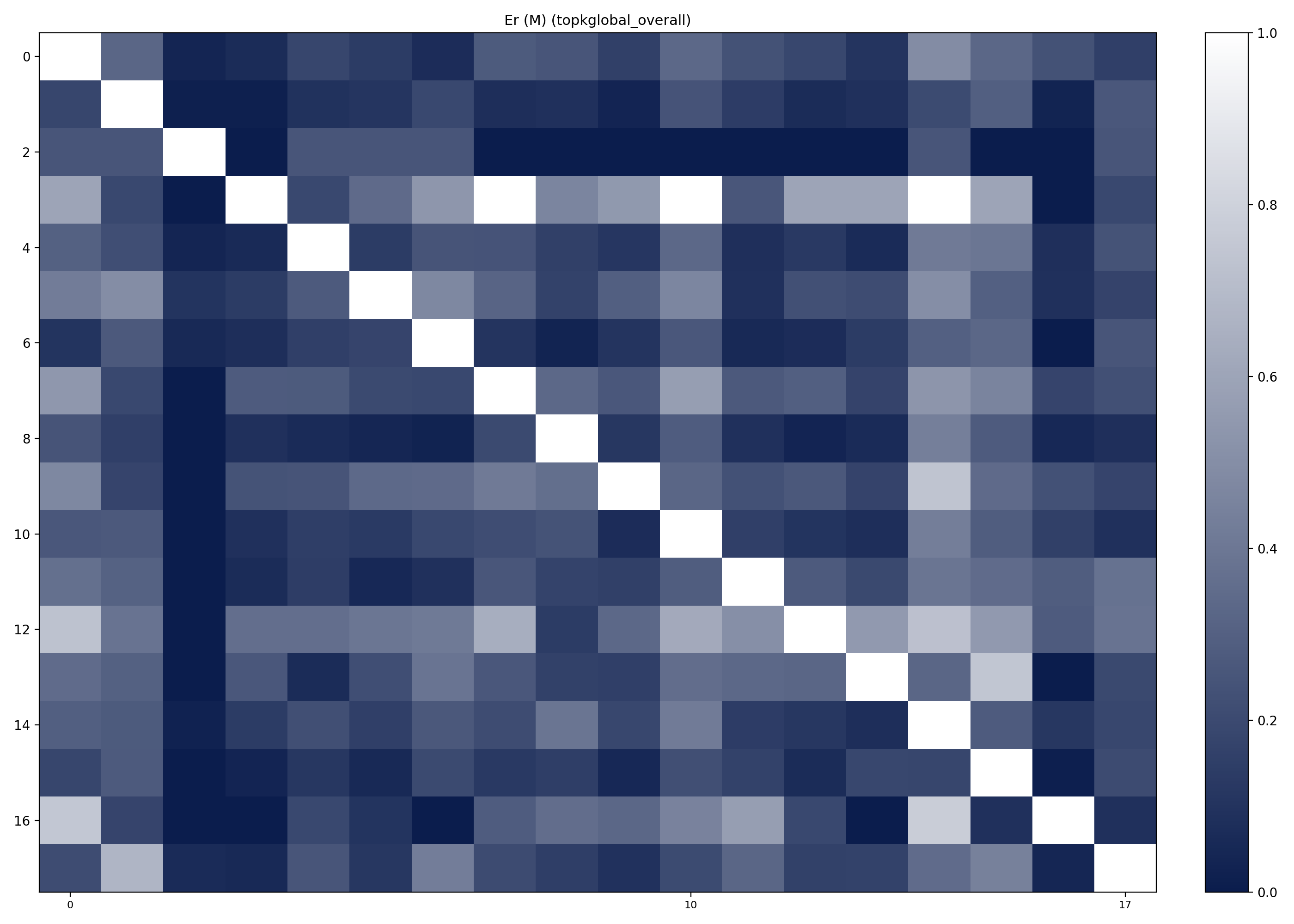}
\caption{$E_r(M)$}
\end{subfigure}
\caption{\textbf{Physical-alignment measurement beyond Transformers:} ResNet50 alignment measurement at layer 8 with the $25\mathrm{ER}$ rank window.  This figure tests Propositions~\ref{thm:block-energy-margin}, \ref{thm:heatmap-to-noise-bound}, and~\ref{thm:scale-weight-transfer}: the permuted matrices exhibit structured physical transport, while the $E_r$ heatmaps summarize the measured core/overlap/noise structure.}
\label{fig:resnet-physical-alignment-main}
\end{figure}

\subsection{Empirical measurement III: stable effective-rank windows}
\label{sec:prediction-compressibility}

The spectral-tail theorem states that dominant energy can be captured in a finite truncation window governed by the tail measure $\nu_\alpha$.
The alignment experiments explicitly use energy-thresholded windows such as top-$50\%$ energy and effective-rank variants.
Figure~\ref{fig:qwen-er-window-main} is placed here because it tests exactly the middle bridge: the Cartan/tail theory selects a stable dominant window, and Theorem~\ref{thm:static-structure-stability} explains when the same physical static structures survive the truncation error.
The persistence of block structure across those windows is consistent with the spectral-tail truncation serving as the interface between Cartan spectral geometry and physical alignment.

\paragraph{Deep-learning meaning.}
Changing the energy-rank window changes the number of retained singular directions.  If the same block structure persists under $25\mathrm{ER}$ and $50\mathrm{ER}$ windows, then the observed modular structure is not an artifact of one hand-picked truncation level.  It indicates that the physically meaningful channel organization is already present in the dominant energy range and is stable under moderate changes of the retained spectral window.

\begin{figure}[t]
\centering
\begin{subfigure}[t]{0.49\textwidth}
\centering
\includegraphics[width=\linewidth]{alignmentM/Qwen3-8B/25ER/layer_18/M/Er_M.png}
\caption{Qwen3-8B, $25\mathrm{ER}$: $E_r(M)$}
\end{subfigure}
\hfill
\begin{subfigure}[t]{0.49\textwidth}
\centering
\includegraphics[width=\linewidth]{alignmentM/Qwen3-8B/50ER/layer_18/M/Er_M.png}
\caption{Qwen3-8B, $50\mathrm{ER}$: $E_r(M)$}
\end{subfigure}

\medskip

\begin{subfigure}[t]{0.49\textwidth}
\centering
\includegraphics[width=\linewidth]{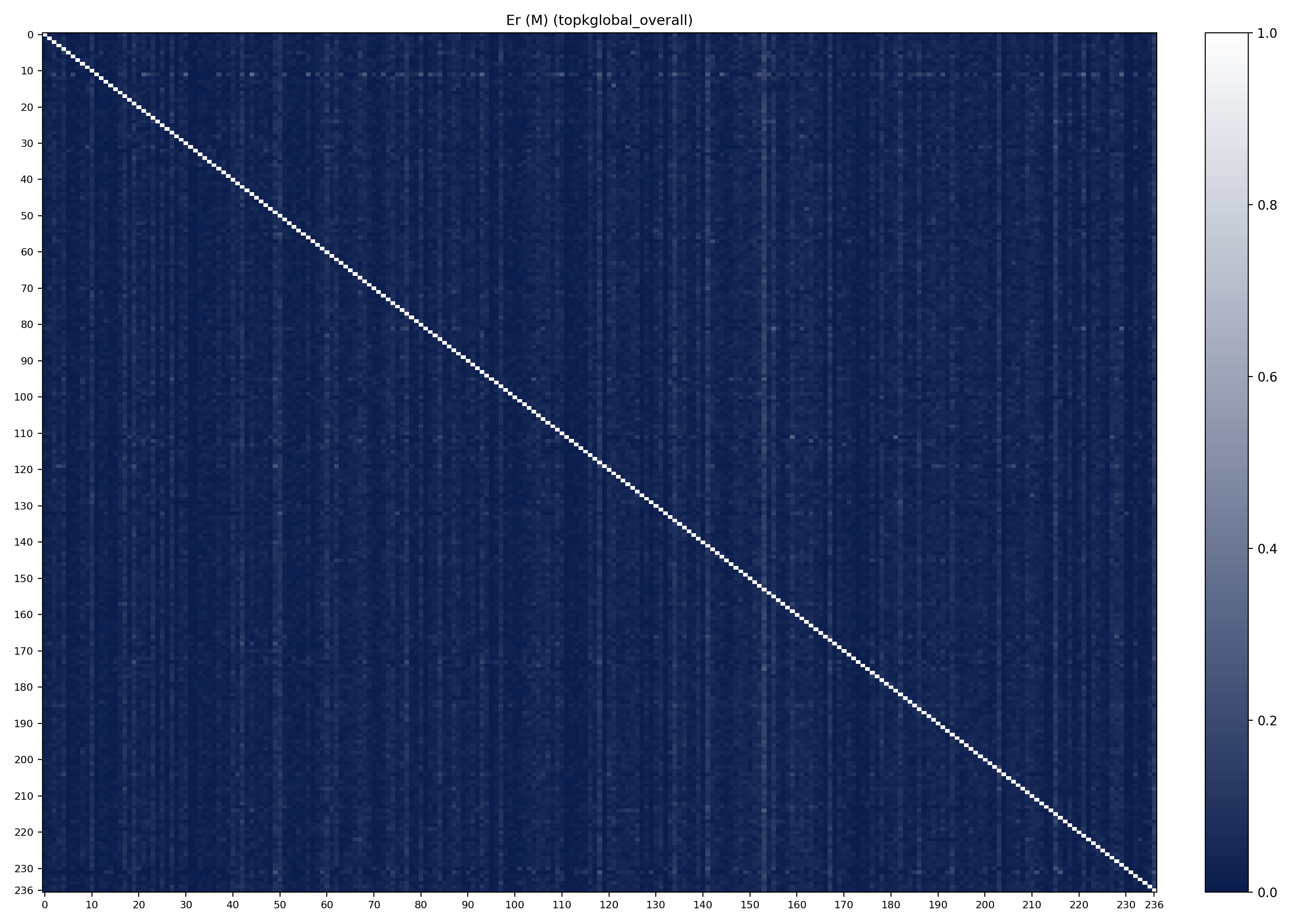}
\caption{LLaMA3-8B, $25\mathrm{ER}$: $E_r(M)$}
\end{subfigure}
\hfill
\begin{subfigure}[t]{0.49\textwidth}
\centering
\includegraphics[width=\linewidth]{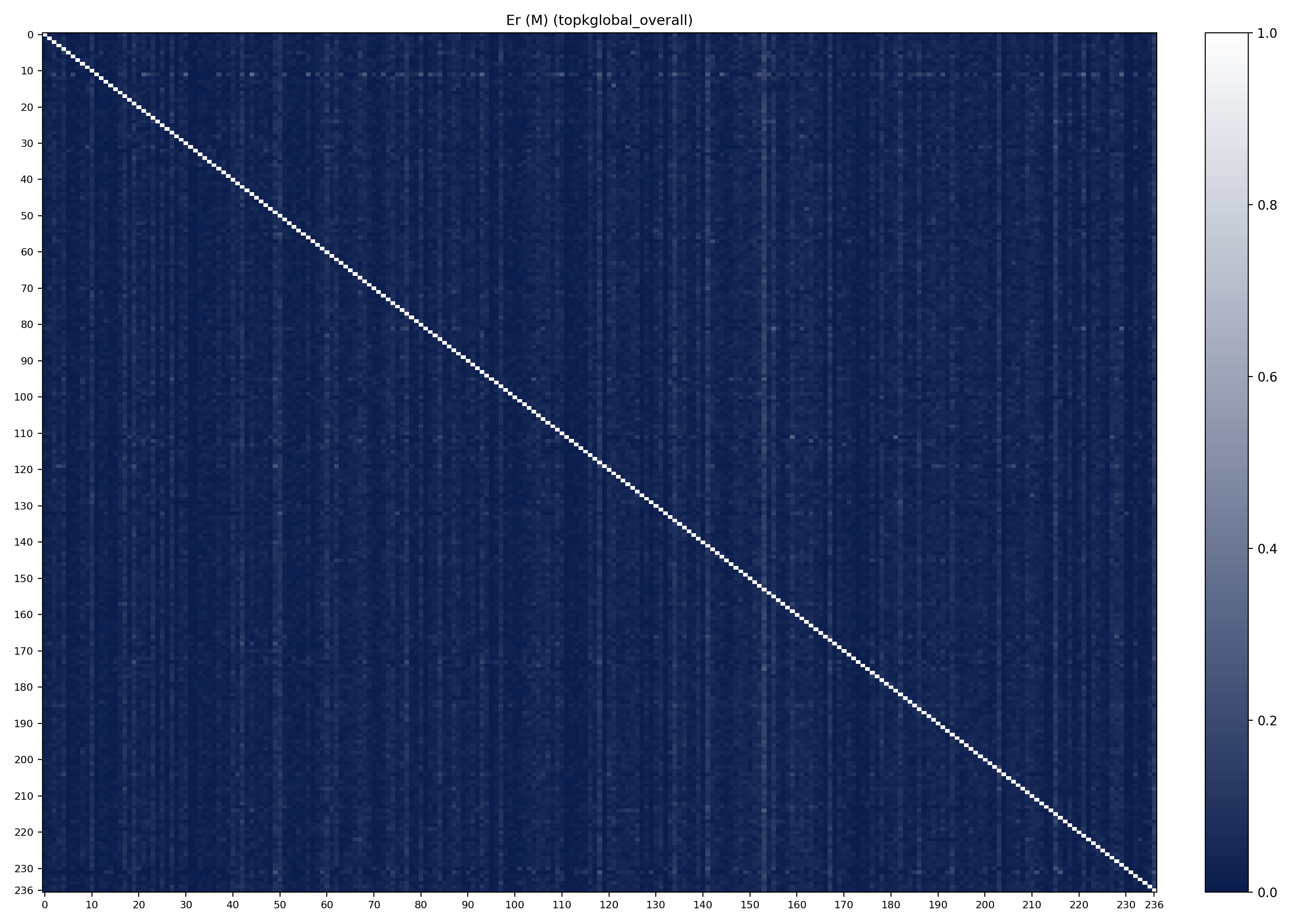}
\caption{LLaMA3-8B, $50\mathrm{ER}$: $E_r(M)$}
\end{subfigure}
\caption{\textbf{Effective-rank-window measurement.}  These panels compare $25\mathrm{ER}$ and $50\mathrm{ER}$ versions of the block-energy matrix for representative LLM interfaces.  The persistence of a similar block-energy pattern under a larger retained spectral window indicates that the measured channel organization is not tied to a single truncation level.  Corollary~\ref{cor:Er-window-robustness} states that exact pixels need not match; the invariant quantity is the block-energy structure when truncation error is controlled.}
\label{fig:rank-window-robustness}
\end{figure}

\subsection{Numerical margin measurement protocol}
\label{sec:margin-measurement-protocol}

A checkpoint-level finite-dimensional test of the full theorem package requires measuring:
\begin{enumerate}[leftmargin=2.0em]
\item fitted coordinates $\hat\alpha_k$, chart errors $e_k^{\mathrm{chart}}$, and fitted-tail errors $\Delta_{\mathrm{tail}}(W_k,\hat\alpha_k)$;
\item local transport budgets $\widehat\Lambda_k$;
\item effective ranks $R_\eps(W_k)$ and rank-tail margins $\mathfrak m_\eps(\hat\alpha_k)$;
\item the coordinate interpretation of the measured transport: source-mode columns for $M_{\mathrm{out}}$ or physical input-channel columns for $M_{\mathrm{phys}}$;
\item physical alignment structures $(\Pi_{k,\mathrm{row}},\Pi_{k,\mathrm{col}},\mathcal R_i,\mathcal C_i)$;
\item block-energy matrices $E_{\mathcal R,\mathcal C}(\widehat M_{\mathrm{phy}})$ and their off-diagonal mass $\operatorname{Off}(E)$;
\item the three GSA residuals: spectral total variation, $\|M_{\mathrm{noise}}\|_F$, and pairwise violations $(3\|\Overlap_{i\cap j}\|_2-m_{i,j})_+$, together with the heatmap margin screen $e_iE_{ij}+e_jE_{ji}-m_{i,j}^2/9$;
\item the row/profile margins $\Gamma_i^{\mathrm{SC}}$, $\Gamma_i^{\mathrm{ST}}$, $\Gamma_i^{\mathrm{SA}}$ when full SC/SA/ST ICM labels are claimed stable.
\end{enumerate}
The figures included here provide the first three groups of measurements in this protocol: exponent trajectories, physical alignment matrices, and block-energy matrices.  These experiments measure the finite-dimensional quantities appearing in the certificate statements and compare their observed behavior with the structural predictions of the theorems.  A complete finite-dimensional margin test additionally reports the corresponding numerical margins and compares them with the theorem thresholds.  The pairwise-gap quantities $m_{i,j}$ and $3\|\Overlap_{i\cap j}\|_2-m_{i,j}$ convert displayed block structure into a numerical physical GSA membership test.  The active-column gaps $\Gamma_i(\widehat M)$ from Definition~\ref{def:active-column-order-gap} are also required for a complete margin test, since Theorem~\ref{thm:static-structure-stability} uses them to decide whether the measured static structures are stable under truncation and full-transport perturbations.

Table~\ref{tab:numerical-certificate-template} records the numerical certificate entries that should accompany a representative checkpoint-level report. The present heatmaps visualize the matrices from which these values are computed; certificate-domain membership is asserted only for interfaces whose numerical entries satisfy the stated inequalities.

\clearpage
\begin{table}[H]
\centering
\scriptsize
\setlength{\tabcolsep}{2pt}
\begin{tabularx}{\textwidth}{L{0.15\textwidth}L{0.18\textwidth}L{0.22\textwidth}Y}
\toprule
Certificate entry & Computed from & Required comparison & Decision role \\
\midrule
$R_\eps(W_k)$ and $\widehat\alpha_k$ & singular values of the declared operator & rank-window rule and fitted-tail model & identifies the tested dominant window \\
$\Delta_{\mathrm{tail}}(W_k,\widehat\alpha_k)$ & empirical and fitted spectral tails & condition~\eqref{eq:cartan-empirical-rank-stability} & validates fitted-tail rank transfer \\
$\mathcal E_{\mathrm{tr},k}$ and $r_{\mathrm{cert}}$ & tail energies and extracted margins & $\mathcal E_{\mathrm{tr},k}<r_{\mathrm{cert}}$ & certifies full-to-truncated incidence stability \\
$\min_i\Gamma_i$ & active-column energy scores & Theorem~\ref{thm:static-structure-stability}(C1) & certifies active-support stability \\
$\max_{i<j}3o_{ij}/m_{ij}$ or $H_{\max}$ & pairwise blocks or heatmap screen & value $<1$ & certifies one-third pairwise overlap \\
$\|M_{\mathrm{noise}}\|_F/\|\widehat M\|_F$ & core/overlap/noise decomposition & chosen noise tolerance & reports residual unstructured mass \\
Null-baseline margins & same pipeline on controls & trained margin exceeds null margins & checks against clustering/permutation artifacts \\
\bottomrule
\end{tabularx}
\caption{Numerical certificate entries required for a complete finite-dimensional Physical GSA margin test. Values are not inferred from visual intensity alone; they are computed from the same matrices, partitions, coordinate interpretation, and extraction protocol used to generate the figures. If this table is not filled with actual values for an interface, the figures should be read as measurements rather than certificate-domain membership claims.}
\label{tab:numerical-certificate-template}
\end{table}

\begin{corollary}[Complete numerical margin criterion for empirical measurements]
\label{cor:margin-complete-test}
Fix a trained checkpoint and a family of measured interfaces.  Suppose the following quantities have been computed on the same layer indices, rank windows, row partitions, coordinate interpretation, and active-column rules used to generate the figures:
\[
\widehat D_{\mathrm{spec}}
:=\sum_{k=0}^{L-2}|\widehat\alpha_{k+1}-\widehat\alpha_k|,
\qquad
\widehat D_{\mathrm{noise}}
:=\sum_{k=0}^{L-2}\|M_{\mathrm{noise},k}\|_F,
\]
and
\[
\widehat D_{\mathrm{pair}}
:=
\sum_{k=0}^{L-2}\sum_{(i,j)\in\mathcal P_k^{\mathrm{nd}}}
\bigl(3\|\Overlap_{i\cap j}^{(k)}\|_2-m_{i,j}^{(k)}\bigr)_+ .
\]
Assume further that the fitted-tail errors satisfy the empirical rank-window condition \eqref{eq:cartan-empirical-rank-stability} whenever a fitted Cartan tail is used to justify an empirical rank window, and that the measured active-column gaps satisfy
\[
\Gamma_i(\widehat M_{\mathrm{phy},k})>2\omega(\widehat M_{\mathrm{phy},k},\eta_k)
\]
whenever a truncation or full-transport perturbation of size at most $\eta_k$ is used, and that the heatmap margin screen
\[
e_iE_{ij}+e_jE_{ji}<\frac{m_{i,j}^2}{9}
\]
holds for every pair whose one-third threshold is inferred from block-energy data rather than measured directly.  Then the displayed measurements constitute a complete finite-dimensional margin test in the following precise sense:
\begin{enumerate}[label=(V\arabic*),leftmargin=2.1em]
\item $\widehat D_{\mathrm{spec}}$ is the empirical value of the Cartan-coordinate total variation controlled by Theorem~\ref{thm:cartan-rigidity}.
\item $\widehat D_{\mathrm{noise}}$ and $\widehat D_{\mathrm{pair}}$ are the empirical physical-alignment residuals appearing in Definition~\ref{def:dist-gsa}.
\item If the measured transport budgets and chart errors satisfy the right-hand side of Proposition~\ref{thm:distance-to-gsa} with tolerance $\tau$, then
\[
\widehat D_{\mathrm{spec}}+\widehat D_{\mathrm{noise}}+\widehat D_{\mathrm{pair}}\le \tau
\]
is a margin-verified physical-GSA certificate-residual bound for that checkpoint.
\item If, in addition, the active-column and static-structure stability inequalities of Theorem~\ref{thm:static-structure-stability} hold, then the extracted SRS/Hub/core-overlap-noise anatomy is stable under the measured truncation and full-transport error; the full SC/SA/ST ICM labels are stable when the row/profile margins of Definition~\ref{def:icm-stability-margins} are also checked.
\end{enumerate}
\end{corollary}

\begin{proof}
We prove the four assertions.

For (V1), Theorem~\ref{thm:robust-cartan-rigidity} states that the total variation of the Cartan coordinate is bounded by the measured local transport budgets and chart errors.  The quantity $\widehat D_{\mathrm{spec}}$ is exactly the left-hand side of that total-variation estimate with the fitted coordinates replacing the theoretical coordinates.  Thus, once the fit errors are included in the chart-error term, it is the empirical instance of the theorem's spectral residual.

For (V2), Definition~\ref{def:dist-gsa} defines the physical GSA residual as the sum of the spectral total variation, the Frobenius norms of the noise components, and the positive parts of the one-third-threshold violations over nondegenerate pairs.  The quantities $\widehat D_{\mathrm{noise}}$ and $\widehat D_{\mathrm{pair}}$ are precisely the second and third terms of that definition computed from the measured physical structures.

For (V3), Proposition~\ref{thm:distance-to-gsa} gives an upper bound for the sum of the three residuals in Definition~\ref{def:dist-gsa}.  Substituting the measured budgets, chart errors, noise norms, and pairwise margins into that proposition yields the asserted finite-dimensional margin test.  The conclusion is conditional on the same hypotheses as the proposition: the rank windows, physical structures, nondegenerate-pair convention, and overlap margin assumptions must be evaluated on the same objects that produced the figures.

For (V4), Theorem~\ref{thm:static-structure-stability} proves that active-column sets, pairwise incidence structure, and the derived core/overlap/noise and SRS/Hub incidence structures are unchanged under perturbations whose size is below the specified active-column and pairwise margin thresholds.  The active-column inequality in the statement is exactly the stability condition used to preserve selected support sets.  When the one-third threshold is not measured directly but inferred from a block-energy heatmap, Proposition~\ref{thm:heatmap-margin-screen} converts the numerical screen $e_iE_{ij}+e_jE_{ji}<m_{i,j}^2/9$ into $3\|\Overlap_{i\cap j}\|_2<m_{i,j}$.  Therefore, if these inequalities hold for the measured perturbation sizes $\eta_k$, the extracted SRS/Hub/core-overlap-noise anatomy is stable under the corresponding truncation and full-transport error, and the full SC/SA/ST ICM labels are stable once the additional row/profile margins are verified.
\end{proof}

\section{Conclusion}
\label{sec:conclusion}
This article develops the angular and static-channel component of GSA.  The spectral results imported from the companion article identify a stable dominant energy window.  Inside that window, angular transport matrices and their physical realizations produce finite static objects: the Physical Alignment Matrix, block-energy matrices, pairwise relational triples, the core/overlap/noise decomposition, and ICM/SRS/Hub variables.  The deterministic theorems specify when these objects are stable under truncation error, diagonal reweighting, full-transport perturbations, and pairwise overlap perturbations.  The same coordinates yield low-disruption fine-tuning consequences: small scale-ratio cost forces uniform layerwise scaling, while small SVD-frame displacement and low common-frame cost force coherent left/right singular-vector rotations.

The alignment experiments are placed next to the corresponding finite quantities.  They measure whether trained models display the predicted block-dominant physical transport and whether the same structure persists across effective-rank windows, cluster choices, and layer sweeps.  The visual matrices give measurements consistent with the qualitative structural pattern; the finite-dimensional margin checking is obtained by reporting and checking the associated numerical margins.  Together with the spectral article, this gives a two-part GSA certificate theory: the spectral article controls the location and motion of the dominant spectral window, and this article gives finite-dimensional margin conditions under which the channel-incidence structure transported inside that window is stable and represents a controlled approximation to the full interface transport.
\printbibliography
\end{document}